\newcommand{\CLASSINPUTtoptextmargin}{0.75in}
\newcommand{\CLASSINPUTbottomtextmargin}{0.75in}
\newcommand{\CLASSINPUTinnersidemargin}{0.75in}
\newcommand{\CLASSINPUToutersidemargin}{0.75in}
\documentclass[conference]{IEEEtran}  
\IEEEoverridecommandlockouts                              
\usepackage[utf8]{inputenc}
\usepackage[T1]{fontenc}    
\usepackage{multicol}
\usepackage{hyperref}
\hypersetup{bookmarksopen,bookmarksnumbered,
pdfpagemode=UseOutlines,
colorlinks=true,
linkcolor=blue,
anchorcolor=blue,
citecolor=blue,
filecolor=blue,
menucolor=blue,
urlcolor=blue
}
\usepackage{graphicx}
\usepackage{amsmath}
\usepackage{upgreek}
\usepackage{amsfonts,amssymb}
\usepackage{bm}
\usepackage{bbm}

\usepackage{amsthm} 
\usepackage{thmtools}
\usepackage{mathtools}
\usepackage{epstopdf}
\usepackage{xspace}
\usepackage[numbers,sort&compress]{natbib}
\usepackage[font=footnotesize, skip=0pt]{caption} 
\usepackage[font=footnotesize]{subcaption}
\usepackage{units}
\usepackage{booktabs} 
\usepackage{tabulary}
\usepackage[usenames,dvipsnames]{color} 
\usepackage{diagbox}
\usepackage[ruled,vlined,linesnumbered]{algorithm2e}  
\usepackage{parskip}
\SetKwComment{Comment}{$\triangleright$\ }{}
\usepackage{ifthen,version}
\usepackage{soul}
\usepackage{fixltx2e}
\usepackage{csquotes}
\usepackage{microtype}      
\usepackage{lipsum}
\usepackage{csquotes}
\usepackage{tikz}
\let\labelindent\relax
\usepackage{enumitem}

\newcommand{\fullFigGap}[0]{\vspace{-1.5\baselineskip}} 

\newcommand{\xxnote}[3]{}
\ifx\hidenotes\undefined
  \usepackage{color}
  \renewcommand{\xxnote}[3]{\color{#2}{#1: #3}}
\fi

\newtheorem{problem}{Problem}
\newtheorem{definition}{Definition}
\newtheorem{proposition}{Proposition}

\newtheoremstyle{hypstyle}
{3pt} 
{3pt} 
{\itshape} 
{} 
{\bfseries} 
{.} 
{.5em} 
{} 

\theoremstyle{hypstyle} 
\newtheorem{observation}{O}

\newcommand{\eref}[1]{(\ref{#1})}
\newcommand{\sref}[1]{Section~\ref{#1}}
\newcommand{\figref}[1]{Fig.~\ref{#1}}

\newcommand{\algorithmStyle}[0]{\small}

\DeclareMathOperator*{\argmin}{arg\,min}
\DeclareMathOperator*{\argmax}{arg\,max}

\newcommand{\argmaxprob}[1]{\underset{#1}{\argmax}}
\newcommand{\argminprob}[1]{\underset{#1}{\argmin}}

\newcommand{\abs}[1]{\left|#1 \right|}

\newcommand{\expect}[2]{\mathbb{E}_{#1}\left[#2\right]}

\newcommand{\real}[0]{\mathbb{R}}

\newcommand{\bbm}{\begin{bmatrix}}
\newcommand{\ebm}{\end{bmatrix}}

\newcommand{\pair}[2]{\left( #1, #2\right)}
\newcommand{\set}[1]{\left\lbrace #1\right\rbrace}

\newcommand{\setst}[2]{\left\lbrace #1\;\middle|\;#2\right\rbrace}

\newcommand{\setInsert}[0]{\xleftarrow{\scriptscriptstyle +}}

\newcommand{\Normal}[0]{\mathcal{N}}

\newcommand{\stateSpace}[0]{\mathcal{X}}
\newcommand{\Dim}[0]{d}
\newcommand{\state}[0]{x}
\newcommand{\stateSet}[0]{X}
\newcommand{\obsStateSpace}[0]{\stateSpace_{\mathrm{obs}}}
\newcommand{\freeStateSpace}[0]{\stateSpace_{\mathrm{free}}}

\newcommand{\Path}[0]{\xi}
\newcommand{\ind}[0]{\tau}
\newcommand{\PathSet}[0]{\Xi}
\newcommand{\PathSetFeas}[0]{\PathSet_{\mathrm{feas}}}
\newcommand{\cost}[0]{c}
\newcommand{\costmax}[0]{c_{\mathrm{max}}}

\newcommand{\graph}[0]{G}
\newcommand{\vertexSet}[0]{V}
\newcommand{\edgeSet}[0]{E}
\newcommand{\weight}[0]{\cost}
\newcommand{\vertu}[0]{u}
\newcommand{\vertv}[0]{v}
\newcommand{\edge}[0]{e}

\newcommand{\connect}[0]{\mathtt{Link}}
\newcommand{\compose}[0]{\oplus}
\newcommand{\remove}[0]{\ominus}

\newcommand{\alg}[0]{\textsc{Alg}}
\newcommand{\denseGraph}[0]{\graph_{\mathrm{dense}}}
\newcommand{\denseVertexSet}[0]{\vertexSet_{\mathrm{dense}}}
\newcommand{\denseEdgeSet}[0]{\edgeSet_{\mathrm{dense}}}

\newcommand{\pp}[0]{\Lambda}

\newcommand{\start}[0]{\state_{s}}
\newcommand{\goal}[0]{\state_{g}}

\newcommand{\feat}[0]{y}
\newcommand{\featDim}[0]{m}

\newcommand{\policy}[0]{\pi}
\newcommand{\policyClass}[0]{\Pi}

\newcommand{\loss}[0]{\mathcal{L}}
\newcommand{\risk}[0]{\mathcal{R}}
\newcommand{\encoderParam}[0]{\phi}
\newcommand{\decoderParam}[0]{\theta}
\newcommand{\encoder}[0]{q_\encoderParam}
\newcommand{\decoder}[0]{p_\decoderParam}

\newcommand{\dataSet}[0]{\mathcal{D}}
\newcommand{\condVar}[0]{\feat}
\newcommand{\latentVar}[0]{z}
\newcommand{\latentDim}[0]{L}
\newcommand{\stateSetVar}[0]{\stateSet}
\newcommand{\stateVar}[0]{\state}

\newcommand{\inputSP}[0]{\stateSetVar_{\mathrm{sp}}}

\newcommand{\inputBottleneck}[0]{\stateSetVar_{\mathrm{bn}}}

\newcommand{\sparseGraph}[0]{\graph_{\mathrm{sparse}}}
\newcommand{\sparseEdgeSet}[0]{\edgeSet_{\mathrm{sparse}}}

\newcommand{\inflation}[0]{\eta}
\newcommand{\inflatedGraph}[0]{\graph_{\mathrm{inf}}}
\newcommand{\inflatedEdgeSet}[0]{\edgeSet_{\mathrm{inf}}}
\newcommand{\inflatedPath}[0]{\Path_{\mathrm{inf}}}

\newcommand{\densePath}[0]{\Path_{\mathrm{dense}}}

\newcommand{\vertexSetBN}[0]{\vertexSet_{\mathrm{bn}}}
\newcommand{\edgeSetBN}[0]{\edgeSet_{\mathrm{bn}}}
\newcommand{\PathBN}[0]{\Path_{\mathrm{bn}}}

\newcommand{\inputDiversity}[0]{\stateSetVar_{\mathrm{div}}}

\newcommand{\PathSetInv}[0]{\PathSet_{\mathrm{inv}}}
\newcommand{\PathSetDiv}[0]{\PathSet_{\mathrm{div}}}
\newcommand{\edgeSetSC}[0]{\edgeSet_{\mathrm{sc}}}

\newcommand{\inputLEGO}[0]{\stateSetVar_{\mathrm{lego}}}

\newcommand{\PathSetSparse}[0]{\PathSet_{\mathrm{sparse}}}
\newcommand{\vertexSetLEGO}[0]{\vertexSet_{\mathrm{lego}}}
\newcommand{\edgeSetInv}[0]{\edgeSet_{\mathrm{inv}}}

\newcommand{\algHalton}[0]{\textsc{Halton}\xspace}

\newcommand{\algGaussian}[0]{\textsc{Gaussian}\xspace}
\newcommand{\algRBB}[0]{\textsc{RBB}\xspace}
\newcommand{\algMAPRM}[0]{\textsc{MAPRM}\xspace}
\newcommand{\algWIS}[0]{\textsc{WIS}\xspace}

\newcommand{\algAbstract}[0]{\textsc{ExtractNodes}\xspace}

\newcommand{\algSP}[0]{\textsc{ShortestPath}\xspace}
\newcommand{\algBottleneck}[0]{\textsc{BottleneckNode}\xspace}
\newcommand{\algDiversity}[0]{\textsc{DiversePathSet}\xspace}
\newcommand{\algLEGO}[0]{\textsc{LEGO}\xspace}

\graphicspath{{figs/}}

\title{
LEGO: Leveraging Experience in Roadmap Generation for Sampling-Based Planning
}

\author{Rahul Kumar*$^{1,}$, Aditya Mandalika*$^{2}$, Sanjiban Choudhury*$^{2}$ and Siddhartha S.~Srinivasa*$^{2}$
\thanks{*This work was (partially) funded by the National Institute of Health R01 (\#R01EB019335), National Science Foundation CPS (\#1544797), National Science Foundation NRI (\#1637748), the Office of Naval Research, the RCTA, Amazon, and Honda.}
\thanks{$^{1}$Department of Computer Science, Indian Institute of Technology, Kharagpur {\tt\small \{vernwalrahul\}@iitkgp.ac.in}}%
\thanks{$^{2}$Paul G. Allen School of Computer Science and Engineering, University of Washington {\tt\small \{adityavk, sanjibac, siddh\}@cs.uw.edu}}
}
\begin{document}

\maketitle
\thispagestyle{empty}
\pagestyle{empty}


\begin{abstract}
We consider the problem of leveraging prior experience to generate roadmaps in sampling-based motion planning. A desirable roadmap is one that is sparse, allowing for fast search, with nodes spread out at key locations such that a low-cost feasible path exists. An increasingly popular approach is to learn a distribution of nodes that would produce such a roadmap. State-of-the-art is to train a conditional variational auto-encoder (CVAE) on the prior dataset with the shortest paths as target input. While this is quite effective on many problems, we show it can fail in the face of complex obstacle configurations or mismatch between training and testing.

We present an algorithm \algLEGO that addresses these issues by training the CVAE with target samples that satisfy two important criteria. Firstly, these samples belong only to \emph{bottleneck} regions along near-optimal paths that are otherwise difficult-to-sample with a uniform sampler. Secondly, these samples are spread out across \emph{diverse regions} to maximize the likelihood of a feasible path existing. We formally define these properties and prove performance guarantees for \algLEGO.
We extensively evaluate \algLEGO on a range of planning problems, including robot arm planning, and report significant gains over  heuristics as well as learned baselines. 
\end{abstract}

\section{Introduction}
\label{sec:introduction}

We examine the problem of leveraging prior experience in sampling-based motion planning. In this framework, the continuous configuration space of a robot is sampled to construct a graph or \emph{roadmap}~\cite{kavraki1996probabilistic,Lav06} where vertices represent robot configurations and edges represent
potential movements of the robot. A shortest path algorithm~\cite{hart1968formal} is then run to compute a path between any two vertices on the roadmap.
The main challenge is to place a \emph{small set of samples in key locations} such that the algorithm can find a high quality path with small computational effort as shown in Fig.~\ref{fig:graph_lego}. 

Typically, low dispersion samplers such as Halton sequences~\cite{Hal60} are quite effective in uniformly covering the space and thus bounding the solution quality~\cite{janson2015deterministic} (\figref{fig:graph_halton}). 
However, as they decrease dispersion uniformly in C-space, a narrow passage with $\delta$ clearance in a $d$-dimensional space requires $O((\frac{1}{\delta})^{d})$ samples to find a path. 
This motivates the need for biased sampling to \emph{selectively densify} in regions where there might be a narrow passage~\citep{holleman2000framework, hsu2005hybrid,burns2005sampling,hsu2006probabilistic,kurniawati2008workspace}. These techniques are applicable across a wide range of domains and perform quite well in practice. 

However, not all narrow passages are relevant to a given query. Biased sampling techniques, which do not have access to the likelihood of the optimal path passing through a region, can still drop samples in more regions than necessary.
Interestingly, the different environments that a robot operates in share a lot of structural similarity.
We can use information extracted from planning on one such environment to decide how to sample on another; we can \emph{learn} sampling distributions using tools such as a conditional variational auto-encoder (CVAE). \citet{ichter2017learning} propose a useful approximation to train a learner to sample along the \emph{predicted} shortest path: given a training dataset of worlds, compute shortest paths, and train a model to independently predict nodes belonging to the path. After all, the best a generative model can do is to sample only along the true shortest path. However, this puts \emph{all of the burden} on the learner. Any amount of prediction error, due to approximation or train-test mismatch, results in failure to find a feasible path. 

We argue that a sampler, instead of trying to predict the shortest path, needs to only identify key regions to focus sampling at, and let the search algorithm determine the shortest path. Essentially, we ask the following question:
\begin{quote}
How can we share the responsibility of finding the shortest path between the sampler and search?
\end{quote}

\begin{figure}[!t]
  \centering
  \begin{subfigure}[b]{0.45\linewidth}
    \centering
    \includegraphics[width=\linewidth]{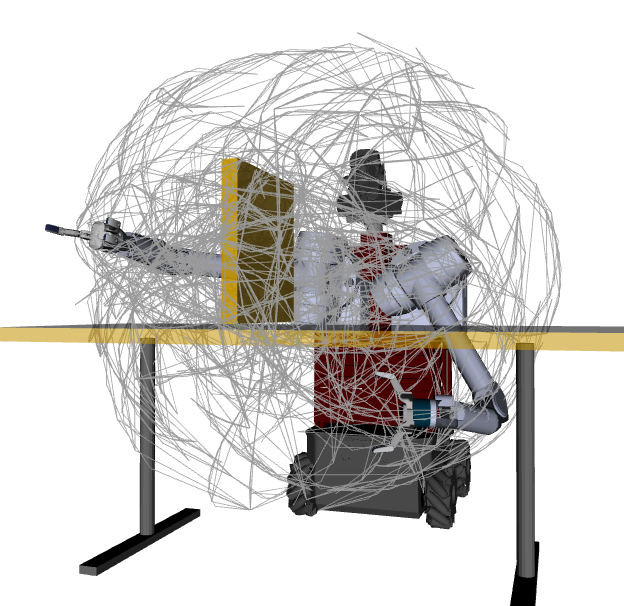}
    \caption{}
    \label{fig:graph_halton}
  \end{subfigure}
  \begin{subfigure}[b]{0.45\linewidth}
    \centering
    \includegraphics[width=\linewidth]{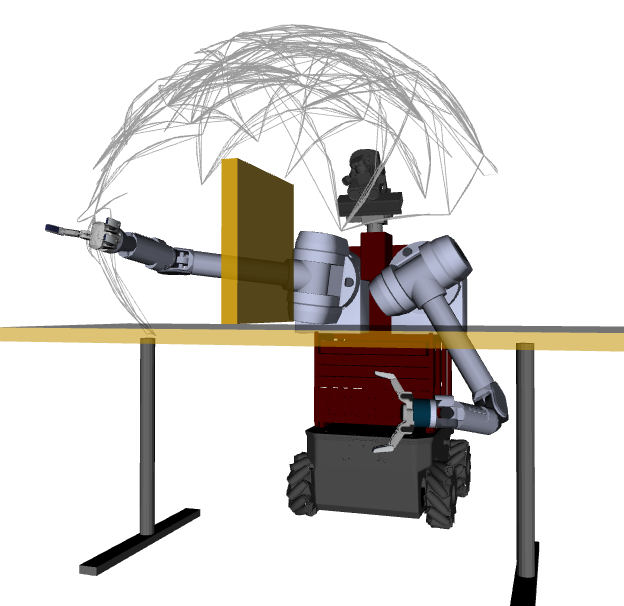}
    \caption{}
    \label{fig:graph_lego}
  \end{subfigure}

  \caption{
Comparison of roadmaps generated from (a) uniform Halton sequence sampler and (b) from a generative model trained using \algLEGO. The task is to plan from the shown configuration over the table and obstacle to the other side. The graph is visualized by end effector traces of the edges. \fullFigGap
  }
  \label{fig:intro}
\end{figure}


\begin{figure*}[!t]
    \centering
    \includegraphics[width=1.0\textwidth]{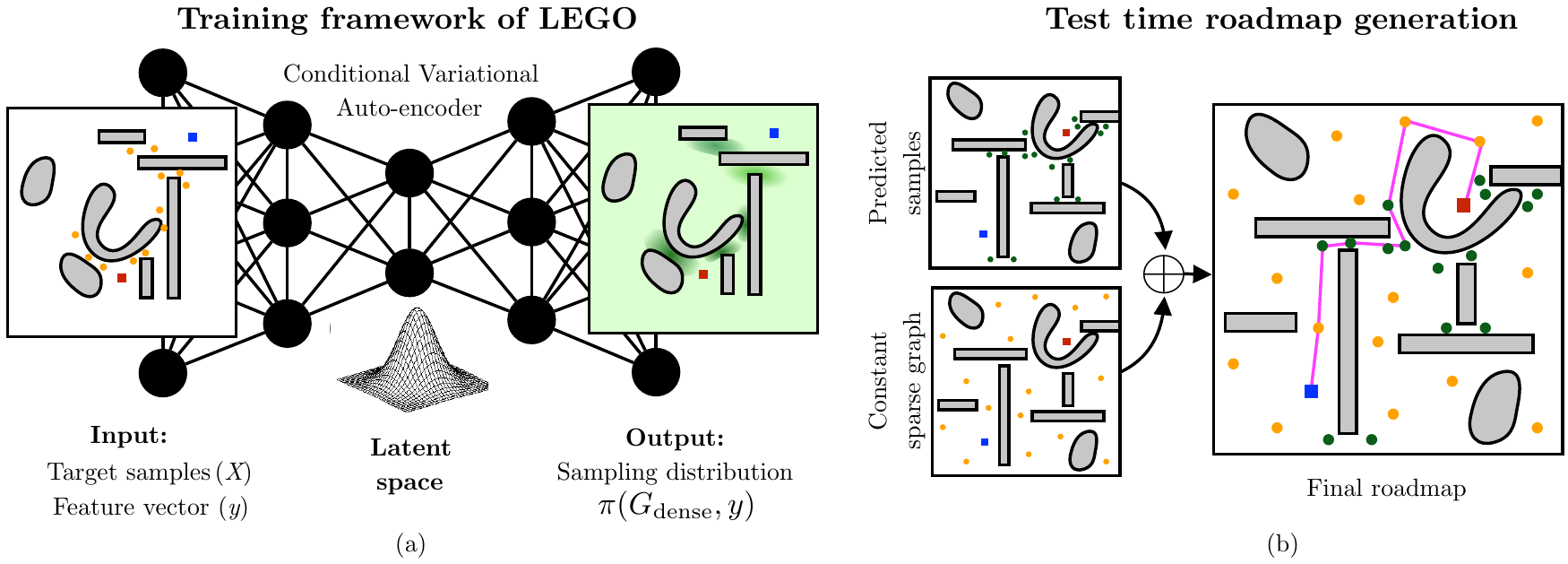}
    \caption{%
    \label{fig:overall_framework}
    The \algLEGO framework for training a CVAE to predict a roadmap. (a) The training process for learning a generative sampling distribution using a CVAE. The input is a pair of candidate samples and feature vector. (b) At test time, the model is sampled to get vertices which are then composed with a constant sparse graph to get a final roadmap. \fullFigGap
    }
\end{figure*}%

Our key insight is for the sampler to predict not the shortest path, but samples that possess two characteristics: (a) samples in bottleneck regions; these are regions containing near-optimal paths, but are difficult for a uniform sampler to reach, and (b) samples that exhibit diversity; train-test mismatch is common and to be robust to it we need to sample nodes belonging to a diverse set of alternate paths. The search algorithm can then operate on a sparse graph containing useful but diverse samples to compute the shortest path. 
We present an algorithmic framework, \textbf{L}everaging \textbf{E}xperience with \textbf{G}raph \textbf{O}racles (\algLEGO{}) summarized in \figref{fig:overall_framework}, for training a CVAE on a prior database of worlds to learn a generative model that can be used for roadmap construction. During training (\figref{fig:overall_framework}a), \algLEGO processes a uniform dense graph to identify a sparse subset of vertices. These vertices are a \emph{diverse} set of \emph{bottleneck} nodes through which a near-optimal path must pass. These are then fed into a CVAE~\cite{kingma2013auto} to learn a generative model. At test time (\figref{fig:overall_framework}b), the model is sampled to get a set of vertices which is additionally composed with a sparse uniform graph to get a final roadmap. This roadmap is then used by the search algorithm to find the shortest path.

We make the following contributions:
\begin{enumerate}
	\item A framework for training a CVAE to predict a roadmap with different target inputs. We identify two main shortcomings of the state-of-the-art \cite{ichter2017learning} that uses the shortest path as the target input - failures in approximation, and failures due to train-test mismatch (\sref{sec:framework}).
	\item $\algLEGO$, an algorithm that tackles both of these issues. It first generates multiple diverse shortest paths, and then extracts bottleneck nodes along such paths to use as the target input for the CVAE (\sref{sec:approach}).
	\item We show that \algLEGO outperforms several learning and heuristic sampling baselines on a set of $\mathbb{R}^2$, $\mathbb{R}^5$, $\mathbb{R}^7$, $\mathbb{R}^8$ and $\real^9$ problems. In particular, we show that it is robust to changes in training and test distribution (\sref{sec:results}).
\end{enumerate}

\section{Related Work}
\label{sec:relatedWork}

The seminal work of \citet{HsuLatMot99} provides a crisp analysis of the shortcomings of uniform sampling techniques in the presence of artifacts such as narrow passages. This has led to a plethora of non-uniform sampling approaches that densify selectively~\citep{holleman2000framework, hsu2005hybrid,burns2005sampling,hsu2006probabilistic,kurniawati2008workspace}. 

Adaptive sampling in the context of roadmaps aims to exploit structure of the environment to place samples in promising areas. A number of works exploited structure of the workspace to achieve this. While some of them attempt to sample between regions of collision to identify narrow passages~\citep{wilmarth1999maprm, holleman2000framework, hsu2003bridge, kurniawati2004workspace, yang2004adapting, van2005using}, others sample near or on the obstacles \cite{amato1998obprm, boor1999gaussian}. There are approaches that divide the configuration space into regions and either select different region-specific planning strategies~\citep{morales2005machine} or use entropy of samples in a particular region to refine sampling~\citep{rodriguez2008resampl}. Other methods try to model the free space to speed up planning~\cite{dalibard2011linear,pan2013faster,choudhury2016pareto}. While these techniques are quite successful in a large set of problems, they can place samples in regions where an optimal path is unlikely to traverse. 

A different class of solutions look at adapting sampling distributions online during the planning cycle. This requires a trade-off between exploration of the configuration space and exploitation of the current best solution. Preliminary approaches define a utility function to do so~\citep{tang2006obstacle,burns2005toward} or use online learning~\citep{kurniawati2008workspace}; however these are not amenable to using priors. \citet{diankov2007randomized} employs statistical techniques to sample around a search tree. \citet{zucker2008adaptive,kuo2018deep} formalize sampling as a model-free reinforcement learning problem and learn a parametric distribution. Since these problems are non i.i.d learning problems, they do require interactive learning and do not enjoy the strong guarantees of supervised learning. 

There has been a lot of recent effort on finding low dimensional structure in planning~\citep{vernaza2011learning}. In particular, generative modeling tools like variational autoencoders~\citep{doersch2016tutorial} have been used to great success~\citep{chen2016dynamic,ha2018approximate,ichter2018robot,zhang2018learning,qureshi2018deeply}. We base our work on \citet{ichter2017learning} where a CVAE is trained to learn the shortest path distribution.

\section{Problem Formulation}
\label{sec:problem_formulation}

Given a database of prior worlds, the overall goal is to learn a policy that predicts a roadmap which in turn is used by a search algorithm to efficiently compute a high quality feasible path.
Let $\stateSpace$ denote a $\Dim-$dimensional configuration space. Let $\obsStateSpace$ be the portion in collision and $\freeStateSpace = \stateSpace \setminus \obsStateSpace$ denote the free space. Let a path $\Path : [0, 1] \rightarrow \stateSpace$ be a continuous mapping from index to configurations. A path $\Path$ is said to be collision-free if $\Path(\ind) \in \freeStateSpace$ for all $\ind \in [0, 1]$. 
Let $\cost(\Path)$ be a cost functional that maps $\Path$ to a bounded non-negative cost $[0, \costmax]$. Moreover, we set $\cost(\emptyset) = \costmax$.  We define a \emph{motion planning problem} $\pp = \set{\start, \goal, \freeStateSpace}$ as a tuple of start configuration $\start \in \freeStateSpace$, goal configuration $\goal \in \freeStateSpace$ and free space $\freeStateSpace$. Given a problem, a path $\Path$ is said to be \emph{feasible} if it is collision-free, $\Path(0)=\start$ and $\Path(1)=\goal$. Let $\PathSetFeas$ denote the set of all feasible paths. We wish to solve the \emph{optimal} motion planning problem by finding a feasible path $\Path^*$ that minimizes the cost functional $\cost(.)$, i.e. $\cost(\Path^*) = \inf_{\Path \in \PathSetFeas} \cost(\Path)$.

We now embed the problem on a graph $\graph = \pair{\vertexSet}{\edgeSet}$ such that each vertex $\vertv \in \vertexSet$ is an element of $\vertv \in \stateSpace$. The graph follows a connectivity rule expressed as an indicator function $\connect: \stateSpace \times \stateSpace \rightarrow \{0, 1\}$ to denote if two configurations should have an edge\footnote{Note this does not involve collision checking. We consider undirected graphs for simplicity. However, it easily extends to directed graphs.}. The weight of an edge $\weight(\vertu, \vertv)$ is the cost of traversing the edge. We reuse $\Path$ to denote a path on the graph.

Let $\abs{\graph}$ denote the cardinality of the graph, i.e. the size of $\abs{\vertexSet}$\footnote{Alternatively we can also use the size of $\abs{\edgeSet}$}. We introduce a graph operation with the notation $\graph \setInsert \stateSet$ to compactly denote insertion of a new set of vertices $\stateSet$, i.e. $\vertexSet \leftarrow \vertexSet \cup \stateSet$, and edges, $\edgeSet \leftarrow \edgeSet \cup \setst{\pair{\vertu}{\vertv}}{\vertu \in \stateSet,~\vertv \in \graph{},~\connect\pair{\vertu}{\vertv}=1 }$. 

A graph search algorithm $\alg$ is given a graph $\graph$ and a planning problem $\pp$. First, it adds the start-goal pair to the graph, i.e $\graph' = \graph \setInsert \set{\start, \goal}$. It then collision checks edges against $\freeStateSpace$ till it finds and returns the shortest feasible path $\Path^*$. The cost of such a path can hence be found by evaluating $\cost(\alg(\graph, \pp))$. If $\alg$ is unable to find any feasible path, it returns $\emptyset$ which corresponds to $\costmax$. 

\begin{definition}[Dense Graph]
We assume we have a \emph{dense graph} $\denseGraph = \pair{\denseVertexSet}{\denseEdgeSet}$ that is sufficiently large to connect the space i.e. for any plausible planning problem, it contains a sufficiently low-cost feasible path. 
\end{definition}
Henceforth, we care about competing with $\denseGraph$. We reiterate that searching this graph, $\alg(\denseGraph, \pp)$, is too computationally expensive to perform online. 

We wish to learn a mapping from features extracted from the problem to a sparse subgraph of $\denseGraph$. Let $\feat \in \real^\featDim$ be a feature representation of the planning problem. Let $\policy(\denseGraph, \feat)$ be a \emph{subgraph predictor oracle} that maps the feature vector to a subgraph $\graph \subset \denseGraph$, $\abs{\graph} \leq \abs{\denseGraph{}}$. We wish to solve the following optimization problem:

\begin{problem}[Optimal Subgraph Prediction] 
\label{prob:optimal_subgraph}
Given a joint distribution $P(\pp, \feat)$ of problems and features, and a dense graph $\denseGraph$, compute a subgraph predictor oracle $\policy$ that minimizes the ratio of the costs of the shortest feasible paths in the subgraph and the dense graph:
\begin{equation}
	\small
	\policy^* = \argminprob{\policy \in \policyClass} \; \expect{ \pair{\pp}{\feat} \sim P(\pp,\feat) }{ \frac{\cost(\alg(\policy(\denseGraph,\feat), \pp))}{\cost(\alg(\denseGraph, \pp))} }
\end{equation}
\end{problem}

\section{Framework for Predicting Roadmaps}
\label{sec:framework}

\begin{figure*}[!t]
    \centering
    \includegraphics[width=1.0\textwidth]{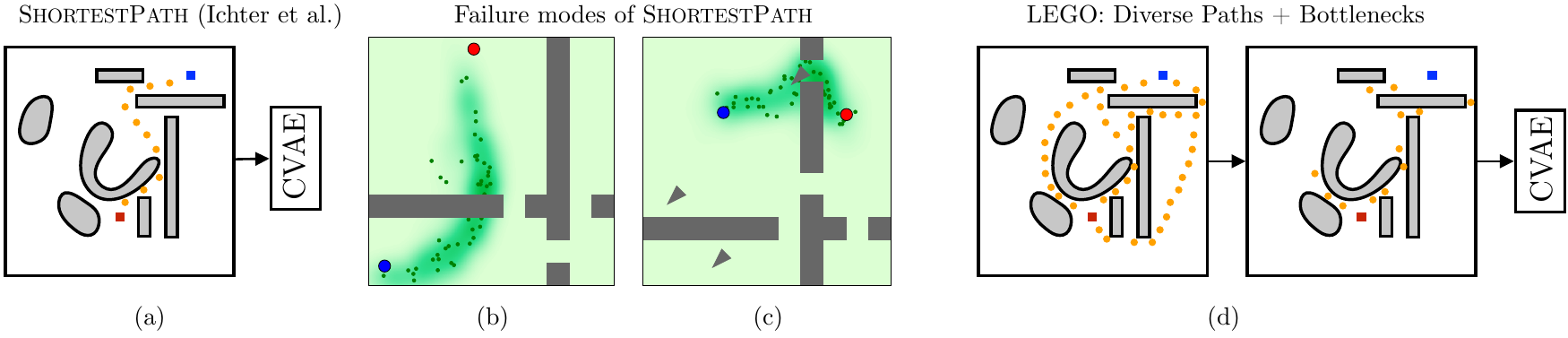}
    \caption{%
    \label{fig:cvae_input}
    (a) The training input generated by \algSP. (b) Failure of \algSP model to route through gaps (c) Another failure of the model due to unexpected obstacles (d) The training input generated by \algLEGO. Diverse shortest paths are generated followed by extraction of bottleneck nodes. \fullFigGap
    }
\end{figure*}%

We now present a framework for training graph predicting oracles as illustrated in \figref{fig:overall_framework}(a). This is a generalization of the approach presented in \cite{ichter2017learning}. The framework applies three main approximations. First, instead of predicting a subgraph $\graph \subset \denseGraph$, we learn a mapping that directly predicts states $\state \in \stateSpace$ in continuous space.\footnote{For cases where a subgraph is preferred, e.g. $\denseGraph$ lies on a constraint manifold, one can design a projection operator $\mathcal{P} : \stateSpace \rightarrow \denseVertexSet$} Secondly, instead of solving a structured prediction problem, we learn an i.i.d sampler that will be invoked repeatedly to get a set of vertices. These vertices are then connected according to an underlying connection rule, such as $k$-NN, to create a graph. Thirdly, we compose the sampled graph with a \emph{constant sparse graph} $\sparseGraph \subset \denseGraph, \abs{\sparseGraph} \leq N$. This ensures that the final predicted graph has some minimal coverage. \footnote{Since $\denseGraph$ is a Halton graph, we use the first $N$ Halton sequences.} 

The core component of the framework is a Conditional Variational Auto-encoder (CVAE)~\cite{sohn2015learning} which is used for approximating the desired sample distribution. CVAE is an extension of a traditional variational auto-encoder~\cite{kingma2013auto} which is a directed graphical model with low-dimensional Gaussian latent variables. CVAE is a \emph{conditional} graphical model which makes it relevant for our application where conditioning variables are features of the planning problem. We provide a high level description for brevity, and refer the reader to~\cite{doersch2016tutorial} for a comprehensive tutorial.

Here $\state \in \stateSpace$ is the output random variable, $\latentVar \in \real^\latentDim$ is the latent random variable and $\condVar \in \real^\featDim$ is the conditioning variable. We wish to learn two \emph{deterministic mappings} - an \emph{encoder} and a \emph{decoder}. An encoder maps $(\stateVar, \condVar)$ to a mean and variance of a Gaussian $\encoder(\latentVar | \stateVar, \condVar)$ in latent space, such that it is ``close'' to an isotropic Gaussian $\Normal(0, I)$. The decoder maps this Gaussian and $y$ to a distribution in the output space $\decoder(\stateVar | \latentVar, \condVar)$. This is achieved by maximizing the following objective $\loss \left( \stateVar, \condVar; \decoderParam, \encoderParam \right)$:
\begin{equation}
\label{eq:cvae_loss}
	  -D_{KL}\left( \encoder(\latentVar | \stateVar, \condVar) \, || \, \Normal(0, I) \right) + \frac{1}{\latentDim} \sum_{l=1}^{\latentDim} \log \decoder(\stateVar | \condVar, \latentVar^{(l)})
\end{equation}
Note that the encoder is needed only for training. At test time, only the decoder is used to map samples from an isotropic Guassian in the latent space to samples in the output space.

We train the CVAE by passing in a dataset $\dataSet = \{ \stateSetVar_i, \feat_i \}_{i=1}^{D}$. $\feat_i$ is the feature vector (conditioning variable) extracted from the planning problem $\pp_i$. $\stateSetVar_i$ is the desired set of nodes extracted from the dense graph $\denseGraph$ that we want our learner to predict. Hence we train the model by maximizing the following objective. 
\begin{equation}
	\risk(\dataSet; \decoderParam, \encoderParam) = \frac{1}{\abs{\dataSet}} \sum_{i=1}^{\abs{\dataSet}} \sum_{j=1}^{\abs{\stateSetVar_i}} \loss \left( \stateVar_j, \condVar_i; \decoderParam, \encoderParam \right)
\end{equation}

\subsection{General Train and Test Procedure}

To summarize, the overall training framework is as follows:
\begin{enumerate}
	\item Load a database of planning problems $\pp_i$ and corresponding feature vectors $\feat_i$.
	\item For each $\pp_i$, extract relevant nodes from the dense graph $\denseGraph$ by invoking $\stateSetVar_i = \algAbstract(\pp_i, \denseGraph)$.
	\item Feed dataset $\dataSet = \{ \stateSetVar_i, \feat_i \}_{i=1}^D$ as input to CVAE.
	\item Train CVAE and return learned decoder $\decoder(\stateVar | \condVar, \latentVar)$.
\end{enumerate}

At test time, given a planning problem $\pp$, the graph predicting oracle $\policy(\denseGraph, \feat)$ performs the following set of steps:
\begin{enumerate}
	\item Extract feature vector $\feat$ from planning problem $\pp$.
	\item Sample $N$ nodes using decoder $\decoder(\stateVar | \condVar, \latentVar)$.
	\item Connect nodes to create a graph $\graph$. Compose sampled graph with a constant sparse graph $\graph \leftarrow \graph \compose \sparseGraph$. 
\end{enumerate}

The focus of this work is on examining variants of the the node extraction function $\stateSetVar = \algAbstract(\pp, \denseGraph)$. While the parameters of the CVAE are certainly relevant (discussed in Appendix \ref{sec:appendix_framework}), in this paper we ask the question:
\begin{quote}
What is a good input $\stateSetVar$ to provide to the CVAE?
\end{quote}
To that end, we explore the following schemes:
\begin{enumerate}
	\item \algSP: Extract nodes $\inputSP$ belonging to the shortest path. This is the baseline approach (\sref{sec:shortest_path}).
	\item \algBottleneck: Extract nodes $\inputBottleneck$ that correspond to bottlenecks along the shortest path (\sref{sec:approach:bottleneck}).
	\item \algDiversity: Extract nodes $\inputDiversity$ belonging to multiple diverse shortest paths (\sref{sec:approach:diversity}).
	\item \algLEGO: Extract nodes $\inputLEGO$ that correspond to bottlenecks along multiple diverse diverse shortest paths. This is our proposed approach (\sref{sec:approach:lego}).
\end{enumerate}

\section{The \algSP (\citet{ichter2017learning}) procedure}
\label{sec:shortest_path}

We examine the scheme applied in~\cite{ichter2017learning} of using nodes belonging to the shortest path on the dense graph as input for training the CVAE. 
The rationality for this scheme is that the distribution of states belonging to the shortest path might lie on a manifold that can be captured by the latent space of the CVAE. This hypothesis is validated across many high-dimensional planning domains. 

We argue that the presented results should not be entirely surprising. The intrinsic difficulty of a planning problem stems from having to search in multiple potential homotopy classes to find a feasible high quality solution. This often manifests in problems involving mazes, bugtraps or narrow passages where the search has to explore and backtrack frequently. Simply increasing the dimension of the problem does not necessarily render it difficult. On the contrary, since the volume of free space increases substantially, there is often an abundance of feasible paths. The challenge, of course, is to find a manifold on which such paths lie with high probability. This is where we found the CVAE to be critical - it learns to \emph{interpolate} between the start and goal along a low dimensional manifold. 


However, we are interested in more \emph{difficult} problems where such interpolations would break down. Based on extensive evaluations of this \algSP scheme, we were able to identify two concrete vulnerabilities: 
\begin{enumerate}
  \item \emph{Failure to route through gaps}: \figref{fig:cvae_input}(b) shows the output of the CVAE when there is a gap through which the search has to route to get to the goal. The model gets stuck in a poor local minimum between linearly interpolating start-goal and routing through the gap since the network is not expressive enough to map the feature vector to such a path. This is tantamount to burdening the sampler to solve the planning problem. 
  \item \emph{Presence of unexpected obstacles in test data}: \figref{fig:cvae_input}(c) shows the output of the CVAE when there are small, unexpected obstacles in the test data which were not present in the training data. The learned distribution samples over this obstacle as it only predicts what it thinks is the shortest path. Even if we were to have such examples in the training data, unless the feature extractor detects such obstacles, the problem remains.
\end{enumerate}

\section{Approach}
\label{sec:approach}


In this section, we present \algLEGO (Leveraging Experience with Graph Oracles), an algorithm to train a CVAE to predict sparse yet high quality roadmaps. We do so by tackling head-on the challenges identified in in~\sref{sec:framework}. Firstly, we recognize that the learner does not have to directly predict the shortest path. Instead, we train it to predict only \emph{bottleneck nodes} that can assist the underlying search in finding a near-optimal solution. Secondly, the roadmap must be robust to prediction errors of the learner. We safeguard against this by training the learner to predict a \emph{diverse set of paths} with the hope that at-least one of them is feasible. 

\subsection{Bottleneck Nodes}
\label{sec:approach:bottleneck}


We begin by noting that $\sparseGraph$ has a uniform coverage over the entire configuration space. Hence, the learner only has to contribute a critical set of nodes that allow $\sparseGraph$ to represent paths that are near-optimal with respect to the path in $\denseGraph$. We call these \emph{bottleneck nodes} as they correspond to regions that are difficult for a uniform sampler to cover. We define $\inputBottleneck = \algBottleneck(\pp, \denseGraph)$ as:

\begin{definition}[Bottleneck Nodes]
\label{def:bottleneck_nodes}
Given a dense graph $\denseGraph$, find the smallest set of nodes which in conjunction with a sparse subgraph $\sparseGraph$ contains a near-optimal path, i.e.
\begin{equation}
\label{eq:bottleneck_nodes}
\begin{aligned}
\argmin_{\vertexSet \subset \denseVertexSet} & & & \abs{\vertexSet} \\
\text{s.t.} & & & \frac{\cost(\alg( \sparseGraph \compose \vertexSet , \pp))}{\cost(\alg(\denseGraph, \pp))} \leq 1 + \epsilon  
\end{aligned}
\end{equation}
\end{definition}
Here $\graph \compose \vertexSet'$ represents a merge operation, i.e. $\vertexSet \leftarrow \vertexSet \cup \vertexSet'$, $\edgeSet \leftarrow \edgeSet \cup \setst{\pair{\vertu}{\vertv}}{\vertu \in \vertexSet', \pair{\vertu}{\vertv} \in \denseEdgeSet }$. 

\begin{algorithm}[t]
  \caption{ \algBottleneck }\label{alg:bottleneck_node_generation}
  \algorithmStyle
  \SetKwInOut{Input}{Input}
  \SetKwInOut{Output}{Output}

  \Input{ Planning problem $\pp$, Bottleneck tolerance $\epsilon$ \\ \ Dense path $\densePath^*$, Sparse graph $\sparseGraph$}
  \Output{ Bottleneck nodes $\vertexSetBN$}

  $\inflatedGraph \gets \sparseGraph \compose \densePath^*$ \Comment*[r]{Add to sparse graph}
  $\eta \gets 1$\;
  \While{$ \cost(\alg( \inflatedGraph , \pp)) \leq (1 + \epsilon) \cost(\densePath^* ) $ \label{line:bn:iteration_start}}
  {
    $\eta \gets \eta + \delta \eta$ \Comment*[r]{Increase inflation}
    \For{$ (\vertu, \vertv) \in \inflatedEdgeSet \setminus \sparseEdgeSet$}
    {
      $\weight(\vertu, \vertv) \gets \eta \weight(\vertu, \vertv)$ \label{line:bn:inflate} \Comment*[r]{Inflate added edges}
    }
  }
  $\inflatedPath^* \gets \alg( \inflatedGraph , \pp)$ \Comment*[r]{Shortest inflated path}
  $\vertexSetBN \gets \densePath^* \cap \inflatedPath^*$ \Comment*[r]{Bottleneck nodes}
  \Return{$\vertexSetBN$}\;
\end{algorithm}

The optimization~\sref{eq:bottleneck_nodes} is combinatorially hard. We present an approximate solution in Algorithm~\ref{alg:bottleneck_node_generation}. We use the optimal path $\densePath^*$ on the dense graph and create an \emph{inflated graph} $\inflatedGraph(\inflation)$ by composing $\sparseGraph \compose \densePath^*$ and inflating weights of newly added edges by $\inflation$ (Line~\ref{line:bn:inflate}). The idea is to disincentivize the search from using any of the newly added edges. This inflation factor is increased till a near-optimal path can no longer be found (Lines~\ref{line:bn:iteration_start}-\ref{line:bn:inflate}). At this point, the additional vertices that the shortest path on this inflated path pass through are essential to achieve near-optimality. This is formalized by the following guarantee:

\begin{proposition}[Bounded bottleneck edge weights]
Let $\edgeSetBN \gets \inflatedPath^* \setminus \sparseEdgeSet$ be the chosen bottleneck edges, $\edgeSetBN^*$ be the optimal bottleneck edges and $\densePath^*$ be the optimal path on $\denseGraph$.
\begin{equation}
\sum_{\edge_i \in \edgeSetBN} \weight(\edge_i) \leq \sum_{\edge_i \in \edgeSetBN^*} \weight(\edge_i) + \frac{(1 + \epsilon) \cost(\densePath^*)}{\inflation}
\end{equation}
\end{proposition}

\begin{proof}(Sketch)
Let $\vertexSetBN^*$ be the optimal bottleneck nodes and $\edgeSetBN^*$ be the optimal bottleneck edges. Let $\PathBN^*$ be the path returned by $\alg( \sparseGraph \compose \vertexSetBN^* , \pp)$. From Definition~\ref{def:bottleneck_nodes}, the following holds:

\begin{equation*}
\begin{aligned}
\sum_{\edge_i \in \edgeSetBN^*} \weight(\edge_i) + \sum_{\edge_i \in \PathBN^* \setminus \edgeSetBN^*} \weight(\edge_i) &\leq (1+\epsilon)  \cost(\densePath^*) \\
\sum_{\edge_i \in \PathBN^* \setminus \edgeSetBN^*} \weight(\edge_i) &\leq (1+\epsilon)  \cost(\densePath^*) \\
\end{aligned}
\end{equation*} 

Since $\inflatedPath^*$ is the shortest path on the inflated graph, we have:
\begin{equation*}
\begin{aligned}
\sum_{\edge_i \in \edgeSetBN} \weight(\edge_i) &+ \inflation \sum_{\edge_i \in \inflatedPath^* \setminus \edgeSetBN} \weight(\edge_i) \\
&\leq \inflation \sum_{\edge_i \in \edgeSetBN^*} \weight(\edge_i) + \sum_{\edge_i \in \PathBN^* \setminus \edgeSetBN^*} \weight(\edge_i) \\
\end{aligned}
\end{equation*} 

Putting the two inequalities together we have:
\begin{equation*}
\begin{aligned}
\inflation \sum_{\edge_i \in \inflatedPath^* \setminus \edgeSetBN} \weight(\edge_i)  &\leq \inflation \sum_{\edge_i \in \edgeSetBN^*} \weight(\edge_i) + \sum_{\edge_i \in \PathBN^* \setminus \edgeSetBN^*} \weight(\edge_i) \\
\sum_{\edge_i \in \inflatedPath^* \setminus \edgeSetBN} \weight(\edge_i)  &\leq \sum_{\edge_i \in \edgeSetBN^*} \weight(\edge_i) + \frac{\sum_{\edge_i \in \PathBN^* \setminus \edgeSetBN^*} \weight(\edge_i)}{\inflation} \\
\sum_{\edge_i \in \inflatedPath^* \setminus \edgeSetBN} \weight(\edge_i)  &\leq \sum_{\edge_i \in \edgeSetBN^*} \weight(\edge_i) + \frac{(1 + \epsilon) \cost(\densePath^*)}{\inflation} \\
\end{aligned}
\end{equation*} 

\end{proof}

\figref{fig:bd_experiments} illustrates the samples generated by (a) \algSP{} and (b) \algLEGO{} trained with samples from \algBottleneck{}; and the successful routing through narrow passages using samples from \algLEGO.

\subsection{Diverse PathSet}
\label{sec:approach:diversity}

In this training scheme, we try to ensure the roadmap is \emph{robust} to errors introduced by the learner. 
One antidote to this process is \emph{diversity} of samples. Specifically, we want the roadmap to have enough diversity such that if the predicted shortest path is in fact infeasible, there are low cost alternates. 

We set this up as a two player game between a planner and an adversary. The role of the adversary is to invalidate as many shortest paths on the dense graph $\denseGraph$ as possible with a fixed budget of edges that it is allowed to invalidate. The role of the planner is to find the shortest feasible path on the invalidated graph and add this to the set of diverse paths $\PathSetDiv$. The function $\inputDiversity = \algDiversity(\pp, \denseGraph)$ then returns nodes belonging $\PathSetDiv$. We formalize this as:
\begin{definition}[Diverse PathSet]
\label{def:diverse_path_set}
 We begin with a graph $\graph^0 = \denseGraph$. At each round  $i$ of the game, the adversary chooses a set of edges to invalidate:
\begin{equation}
\label{eq:cover_ksp}
\edgeSet^*_i = \argmaxprob{\edgeSet_i \subset \edgeSet, \abs{\edgeSet_i} \leq \ell} \cost(\alg( \graph^{i-1} \remove \edgeSet_i , \pp))  
\end{equation}
and the graph is updated $\graph^i = \graph^{i-1} \remove \edgeSet^*_i$. The planner choose the shortest path $\Path_i = \alg( \graph^{i} , \pp))$ which is added to the set of diverse paths $\PathSetDiv$.
\end{definition}

The optimization problem~\eref{eq:cover_ksp} is similar to a set cover problem (NP-Hard~\cite{feige1998threshold}) where the goal is to select edges to cover as many paths as possible. If we knew the exact set of paths to cover, it is well known that a greedy algorithm will choose a near-optimal set of edges~\cite{feige1998threshold}. We have the inverse problem - we do not know how many consecutive shortest paths can be covered with a budget of $\ell$ edges. 

\begin{algorithm}[t]
  \caption{ \algDiversity }\label{alg:diverse_pathset_generation}
  \algorithmStyle
  \SetKwInOut{Input}{Input}
  \SetKwInOut{Output}{Output}

  \Input{ Planning problem $\pp$, Pathset size $k$, \\ \ Dense graph $\denseGraph$, Sparse graph $\sparseGraph$}
  \Output{ Diverse pathset $\PathSetDiv$}

  $\graph \gets \denseGraph$\;
  $\PathSetDiv \gets \emptyset$\;
  \For{$i = 1, \cdots, k$}
  {
    $\PathSet \gets \alg^{L}( \graph , \pp)$ \Comment*[r]{L-shortest paths}
    $\edgeSet_i \gets \emptyset, \PathSetInv \gets \emptyset$\;
    \While{ $\abs{\edgeSet_i} < \ell$}
    {
      $\PathSet \gets \setst{\Path}{\Path \in \PathSet, \Path \cap \edgeSet_i = \emptyset}$\;
      \For{$j = \abs{\edgeSet_i}, \cdots, \ell$\label{line:diverse:greedy_start}}
      {
        $\edge_j \gets \argmaxprob{\edge \in \edgeSet} \underset{\Path \in \PathSet, \edge \notin \Path}{\min} \cost(\Path)$ \Comment*[r]{Greedy}
        $\edgeSet_i \gets \edgeSet_i \cup \{ \edge_j \}$ \Comment*[r]{Add edge to set}
        $\PathSet_j \gets \setst{\Path \in \PathSet}{\edge_j \in \Path}$\;
        $\PathSetInv \gets \PathSetInv \cup \PathSet_j$ \Comment*[r]{Invalidate paths}
        $\PathSet \gets \PathSet \setminus \PathSet_j$\label{line:diverse:greedy_end}\;
      }
      $\edgeSetSC \gets \textsc{SetCover}(\PathSetInv)$ \label{line:diverse:set_cover} \Comment*[r]{Greedy cover}
      \If{$\abs{\edgeSetSC} \leq \abs{\edgeSet_i}$}
      {
        $\edgeSet_i \gets \edgeSetSC$ \Comment*[r]{Use better cover}
      }
    }
    $\graph \gets \graph \remove \edgeSet_i $ \Comment*[r]{Remove edges}
    $\Path_i \gets \alg( \graph , \pp)$ \Comment*[r]{New shortest path}
    $\PathSetDiv \gets \PathSetDiv \cup \Path_i$ \Comment*[r]{Add to diverse pathset}
  } 
\Return{$\PathSetDiv$}\;
\end{algorithm}

Algorithm~\ref{alg:diverse_pathset_generation} describes the procedure. We greedily choose a set of edges to invalidate as many consecutive shortest paths till we exhaust our budget (Lines~\ref{line:diverse:greedy_start}-\ref{line:diverse:greedy_end}). We then apply greedy set cover (Line~\ref{line:diverse:set_cover}). If it leads to a better solution, we continue repeating the process. At termination, we ensure:
\begin{proposition}[Near-optimal Invalidated EdgeSet]
Let $\PathSetInv$ be the contiguous set of shortest paths invalidated by Algorithm~\ref{alg:diverse_pathset_generation} using a budget of $\ell$. Let $\ell^*$ be the size of the optimal set of edges that could have invalidated $\PathSetInv$.
\begin{equation}
\ell \leq (1 + \log \abs{\PathSetInv}) \ell^*
\end{equation}
\end{proposition}

\begin{proof} (Sketch)
We briefly explain the equivalence to a set cover problem. Each path in $\PathSetInv$ corresponds to an element that has to be covered. Each edge $\edge \in \denseEdgeSet$ corresponds to a set of paths in $\PathSetInv$, where each path in the set contains the edge. Invalidating the edge invalidates all paths in the set. 

Line~\ref{line:diverse:set_cover} invokes a greedy set cover algorithm which at every iteration chooses the edge which covers the largest number of uncovered paths. Let $\ell_{\mathrm{greedy}}$ be the number of edges selected by the greedy algorithm, and $\ell^*$ be the optimal. From~\cite{feige1998threshold}, we have the following near-optimality guarantee:
\begin{equation*}
\ell_{\mathrm{greedy}} \leq (1 + \log \abs{\PathSetInv}) \ell^*
\end{equation*}

If $\ell_{\mathrm{greedy}} \leq \ell$, i.e. we have budget remaining, we continue adding edges that can only invalidate more paths in Lines~\ref{line:diverse:greedy_start}-\ref{line:diverse:greedy_end}. This continues till the budget is exhausted. 

\end{proof}

\figref{fig:bd_experiments} illustrates the samples generated by (a) \algSP{} and (b) \algLEGO trained with samples from \algDiversity{}; and the robustness to unexpected obstacles exhibited by \algLEGO{}.

\subsection{Combining Diversity with Bottleneck Extraction}
\label{sec:approach:lego}

\begin{algorithm}[t]
  \caption{ \algLEGO }\label{alg:lego_node_generation}
  \algorithmStyle
  \SetKwInOut{Input}{Input}
  \SetKwInOut{Output}{Output}

  \Input{ Planning problem $\pp$, Bottleneck tolerance $\epsilon$, Pathset \\ \ size $k$, Dense graph $\denseGraph$, Sparse graph $\sparseGraph$}
  \Output{ \algLEGO nodes $\vertexSetLEGO$}

  $\PathSetDiv \gets \algDiversity(\pp, k, \denseGraph, \sparseGraph)$\label{line:lego:diversity}\;
  $\vertexSetLEGO \gets \emptyset$\;
  \While{$\PathSetDiv \neq \emptyset$}
  {
  	$\Path^* \gets \argminprob{\Path \in \PathSetDiv} \; \cost(\Path)$\label{line:lego:iterate_start} \Comment*[r]{Pick diverse path}
  	$\PathSetSparse \gets  \alg^{L}( \sparseGraph , \pp)$\;
  	$\PathSetSparse \gets \setst{\Path}{\Path \in \PathSetSparse, \cost(\Path) \leq (1 + \epsilon) \cost(\Path^*)}$\;
  	$\edgeSetInv \gets \textsc{SetCover}(\PathSetSparse)$\Comment*[r]{Edges to remove}
  	$\sparseGraph \gets \sparseGraph \remove \edgeSetInv$\label{line:lego:iterate_end} \Comment*[r]{Invalidate paths}
  	$\vertexSetBN \gets \algBottleneck(\pp, \epsilon, \Path^*, \sparseGraph)$\label{line:lego:bottleneck}\;
  	$\vertexSetLEGO \gets \vertexSetLEGO \cup \vertexSetBN$ \Comment*[r]{Add to \algLEGO nodes}
  }
  \Return{$\vertexSetLEGO$}\;
\end{algorithm}

We present \algLEGO in Algorithm~\ref{alg:lego_node_generation}  which combines the characteristics of \algBottleneck and \algDiversity to extract a set of diverse bottleneck nodes. We first find a set of diverse paths on the dense graph (Line~\ref{line:lego:diversity}). We then iterate over each path, and adversarially invalidate edges of the sparse graph to ensure it does contain a feasible shorter path (Line~\ref{line:lego:iterate_start}-\ref{line:lego:iterate_end}). The bottleneck nodes for this path are extracted and added to the set of nodes to be returned (Line~\ref{line:lego:bottleneck}).

\begin{figure}[!t]
\centering
  \begin{subfigure}[b]{0.24\columnwidth}
    \centering
    \includegraphics[height=6.1em]{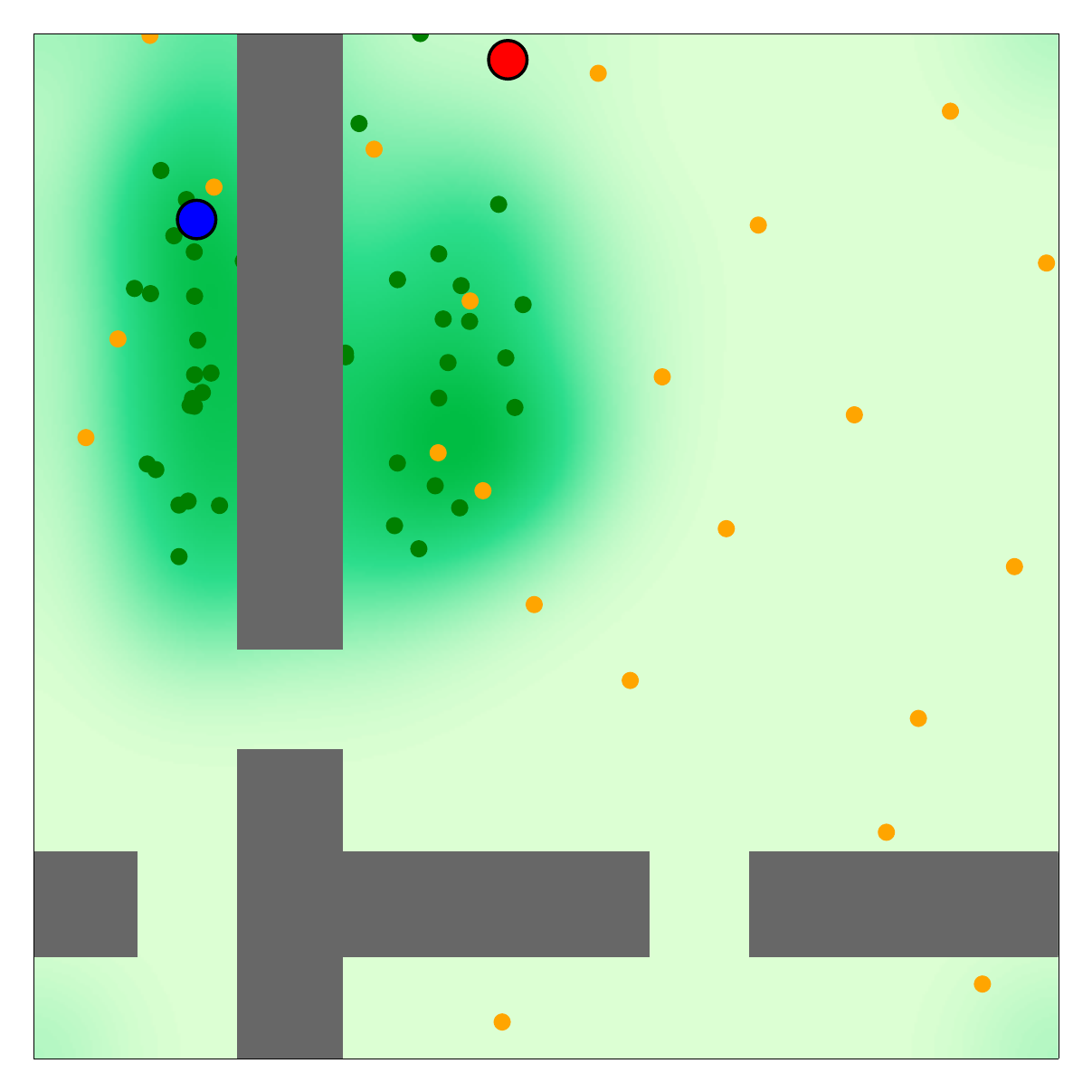}
    \caption{}
    \label{fig:bottleneck_sp}     
    \vspace*{1em}
  \end{subfigure}
  \hfill
  \begin{subfigure}[b]{0.24\columnwidth}
    \centering
    \includegraphics[height=6.1em]{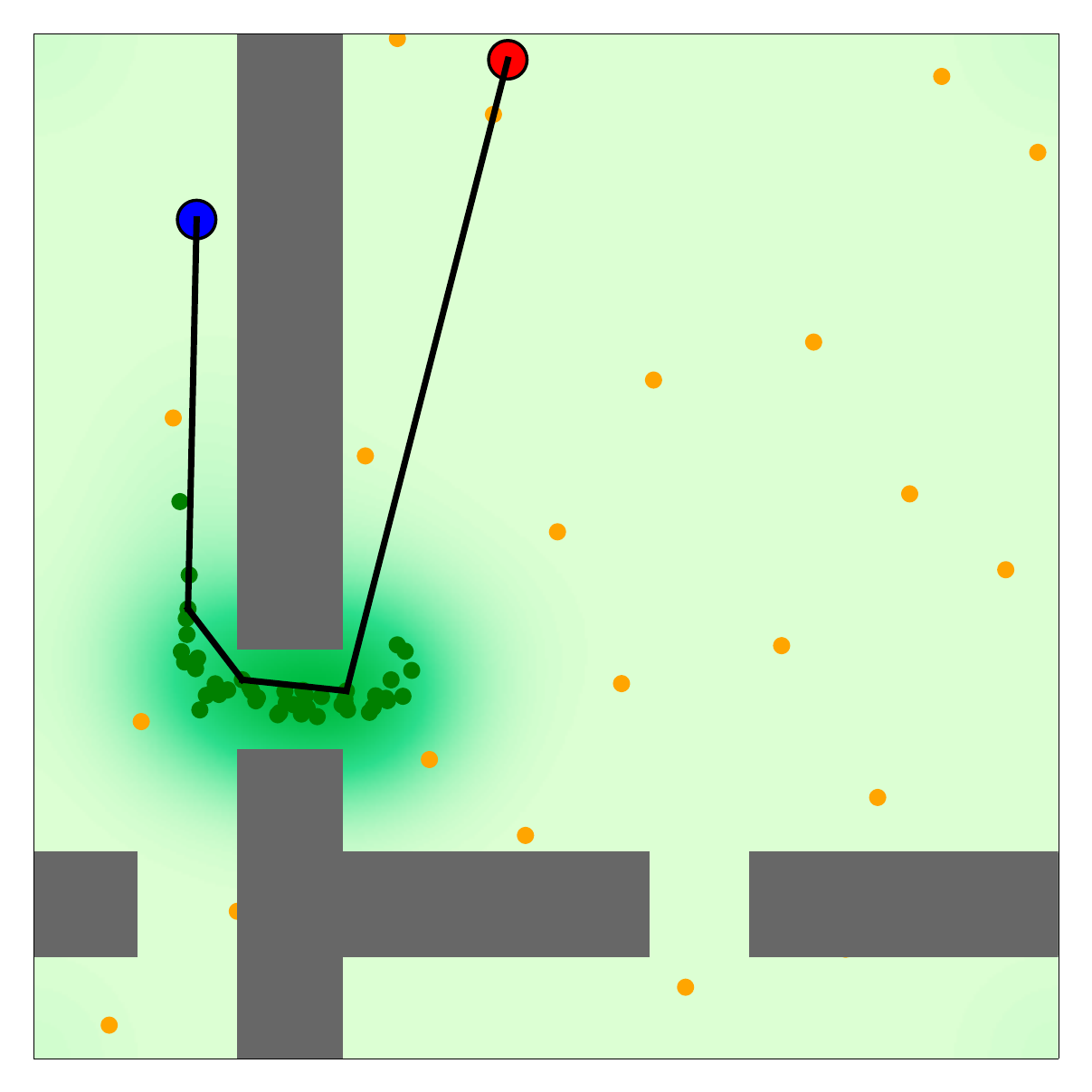}
    \caption{}
    \label{fig:bottleneck_bottleneck}     
    \vspace*{1em}
  \end{subfigure}
  \hfill
  \begin{subfigure}[b]{0.24\columnwidth}
    \centering
    \includegraphics[height=6.1em]{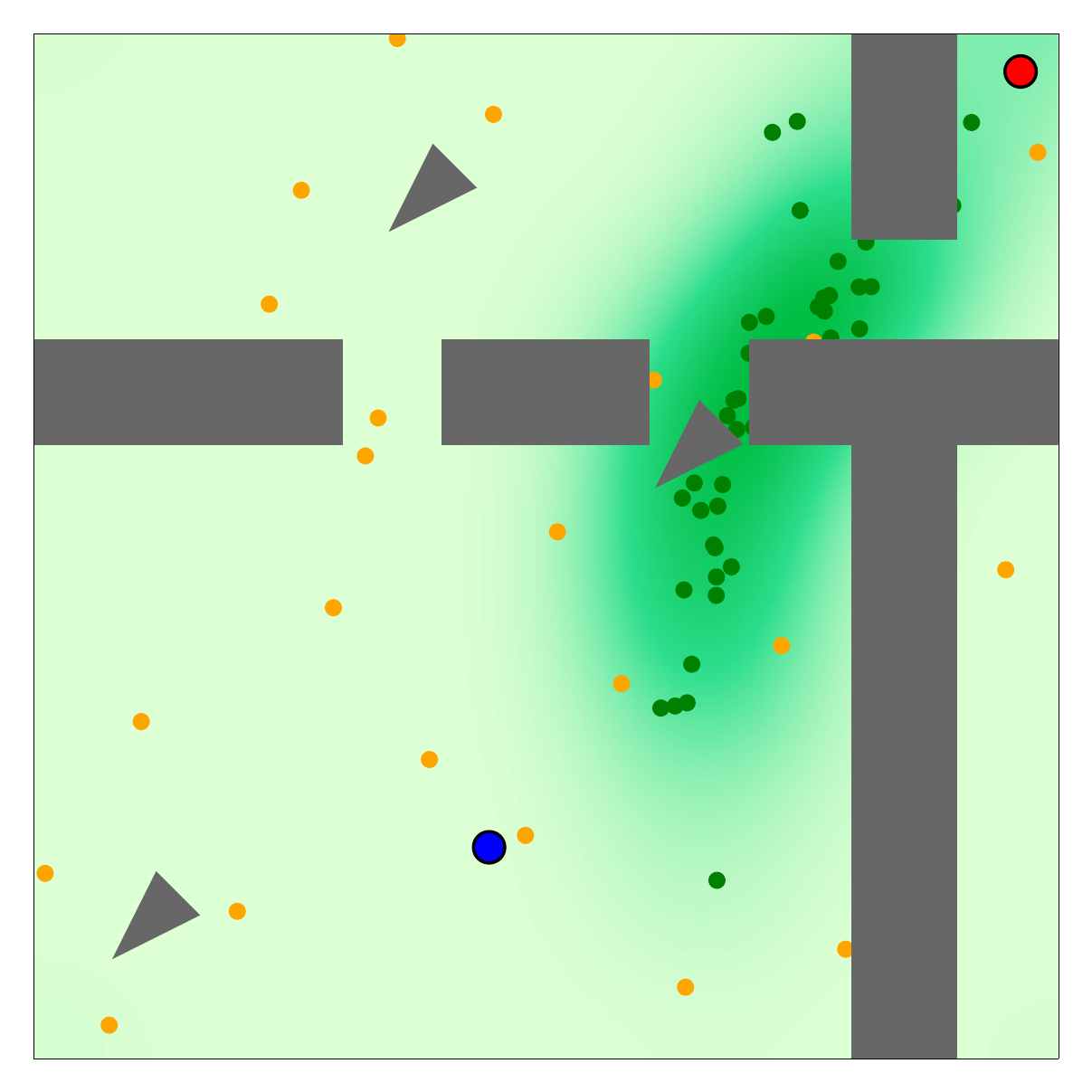}
    \caption{}
    \label{fig:diverse_sp}     
    \vspace*{1em}
  \end{subfigure}  
  \hfill
  \begin{subfigure}[b]{0.24\columnwidth}
    \centering
    \includegraphics[height=6.1em]{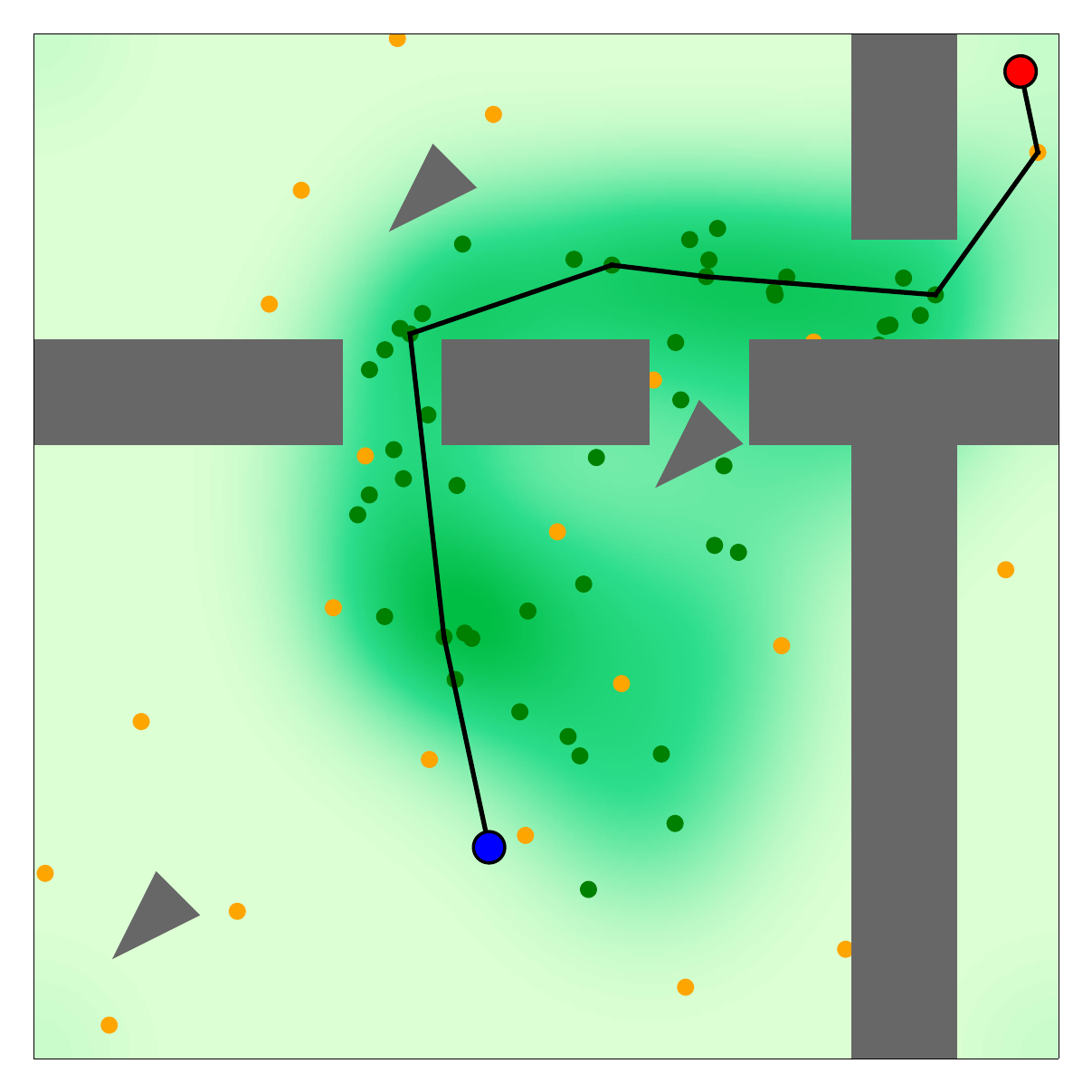}
    \caption{}
    \label{fig:diverse_diverse}     
    \vspace*{1em}
  \end{subfigure}
\caption{Comparison of samples (distribution illustrated as heatmap) generated by \algSP{} (a, c) and \algDiversity{} (b, d) in different environments. In the first environment (left), \algLEGO{} (b) is trained with samples from \algBottleneck{}. In the second environment (right), \algLEGO{} (d) is trained with samples from \algDiversity{}. In both instances, \algSP fails to find a solution.}


\label{fig:bd_experiments}
\end{figure}

\section{Experimental Results}
\label{sec:results}

In this section we evaluate the performance of \algLEGO on various problem domains and compare it against other samplers. 
We consider samplers that do not assume offline computation or learning such as Medial-Axis PRM (\algMAPRM) \cite{wilmarth1999maprm, holleman2000framework},~Randomized Bridge Sampler (\algRBB)~\cite{hsu2003bridge}, Workspace Importance Sampler (\algWIS{})~\cite{kurniawati2004workspace}, a Gaussian sampler, \algGaussian{} \cite{boor1999gaussian}, and a uniform Halton sequence sampler, \algHalton{} \cite{Hal60}. 
Additionally, we also compare our framework against the state-of-the-art learned sampler \algSP \cite{ichter2018robot} upon which our work is based.

\paragraph{Evaluation Procedure}
For a given sampler and a planning problem, we invoke the sampler to generate a fixed number of samples. 
We then evaluate the performance of the samplers on three metrics: a) sampling time b) success rate in solving shortest path problem and c) the quality of the solution obtained, on the graph constructed with the generated samples. 

\paragraph{Problem Domains}
To evaluate the samplers, we consider a spectrum of problem domains.
The $\real^2$ problems have random rectilinear walls with random narrow passages (Fig. \ref{fig:experiments}(a)). These passages can be small, medium or large in width. The n-link arms are a set of $n$ line-segments fixed to a base moving in an uniform obstacle field (Fig. \ref{fig:experiments}(b)). The n-link snakes are arms with a free base moving through random rectilinear walls with passages (Fig. \ref{fig:experiments}(c)). Finally, the manipulator problem has a 7DoF robot arm~\citep{Srinivasa_2012_7081} manipulating a stick in an environment with varying clutter (Fig. \ref{fig:experiments}(d)). Two variants are considered - constrained ($\real^7$), when the stick is welded to the hand, and unconstrained, when the stick can slide along the hand ($\real^8$). 



\begin{table*}[!htpb]
\renewcommand*{\arraystretch}{1}
\small
\centering
\caption{Average time (sec.) by sampling algorithms to generate 200 samples over 100 planning problems}
\begin{tabulary}{\textwidth}{LCCCCCCC}\toprule
	  & \multicolumn{5}{c}{ {\bf Non-Learned Samplers} } & \multicolumn{2}{c}{ {\bf Learned Samplers} } \\
      & {\bf \algHalton} & {\bf \algMAPRM} & {\bf \algRBB} & {\bf \algGaussian} & {\bf \algWIS} & {\bf \algSP} & {\bf \algLEGO} \\ \midrule
Point Robot (2D)   & $0.0036$ & $0.53$ & $0.22$ & $0.02$ & $0.25$ & $0.006$ & $0.006$\\
N-link Arm (3D)    & $0.0058$ & $-$ & $23.96$ & $1.95$ & $0.36$ & $0.016$ & $0.016$\\  
N-link Arm (7D)    & $0.0071$ & $-$ & $37.24$ & $3.77$ & $1.12$ & $0.017$ & $0.017$\\  
Snake Robot (5D)   & $0.0069$ & $39.56$ & $142.21$ & $3.43$ & $0.54$ & $0.013$ & $0.013$\\ 
Snake Robot (9D)   & $0.0074$ & $40.01$ & $180.43$ & $8.71$ & $2.11$ & $0.017$ & $0.017$\\
Manipulator (7D)   & $0.0072$ & $-$ & $-$ & $3.24$ & $-$ & $0.018$ & $0.018$ \\ 
Manipulator (8D)   & $0.0078$ & $-$ & $-$ & $3.33$ & $-$ & $0.018$ & $0.018$ \\ 
\bottomrule
\end{tabulary}
\label{tab:timing_results}
\end{table*}


\begin{table*}[!htpb]
\renewcommand*{\arraystretch}{1}
\small
\centering
\caption{Success Rates of different algorithms on 100 trials over different datasets (reported with a $95\%$ C.I.)}
\begin{tabulary}{\textwidth}{LCCCCCCC}\toprule
	  & \multicolumn{5}{c}{ {\bf Non-Learned Samplers} } & \multicolumn{2}{c}{ {\bf Learned Samplers} } \\
      & {\bf \algHalton} & {\bf \algMAPRM} & {\bf \algRBB} & {\bf \algGaussian} & {\bf \algWIS} & {\bf \algSP} & {\bf \algLEGO} \\ \midrule
\multicolumn{8}{c}{ {\bf 2D Point Robot Planning} }   \\
2D Large (easy)   & $0.73 \pm 0.08$ & $0.73 \pm 0.08$ & $0.74 \pm 0.09$	& $0.65 \pm 0.09$ & $0.78 \pm 0.08$ & $0.86 \pm 0.07$	& $\mathbf{0.97 \pm 0.03}$\\ 
2D Medium   & $0.48 \pm 0.08$ & $0.63 \pm 0.09$ & $0.61 \pm 0.09$	& $0.48 \pm 0.10$ & $0.63 \pm 0.09$ & $0.69 \pm 0.09$ & $\mathbf{0.89 \pm 0.06}$\\ 
2D Small (hard)  	& $0.36 \pm 0.09$ & $0.53 \pm 0.09$ & $0.48 \pm 0.08$ & $0.32 \pm 0.09$ & $0.52 \pm 0.09$ & $0.59 \pm 0.09$ & $\mathbf{0.83 \pm 0.07}$\\ 
\multicolumn{8}{c}{ {\bf N-Link Arm} }   \\
3D    & $0.39 \pm 0.09$ & $-$ & $0.54 \pm 0.09$ & $0.46 \pm 0.10$ & $0.52 \pm 0.10$ & $0.61 \pm 0.09$ & $\mathbf{0.74 \pm 0.08}$\\
7D    & $0.29 \pm 0.09$ & $-$ & $0.46 \pm 0.09$ & $0.41 \pm 0.09$ & $0.46 \pm 0.09$ & $0.57 \pm 0.10$ & $\mathbf{0.71 \pm 0.08}$\\
\multicolumn{8}{c}{ {\bf N-Link Snake Robot} }   \\
5D    & $0.41 \pm 0.09$ & $0.42 \pm 0.09$ & $0.48 \pm 0.10$ & $0.41 \pm 0.09$ & $0.50 \pm 0.10$ & $0.77 \pm 0.08$ & $\mathbf{0.84 \pm 0.07}$\\ 
9D    & $0.49 \pm 0.09$ & $0.45 \pm 0.09$ & $0.52 \pm 0.10$ & $0.51 \pm 0.10$ & $0.53 \pm 0.09$ & $0.82 \pm 0.07$ & $\mathbf{0.86 \pm 0.07}$\\ 
\multicolumn{8}{c}{ {\bf Manipulator Arm Planning} }   \\
Unconstrained (8D)  & $0.24 \pm 0.09$ & $-$ & $-$	& $-$ & $-$ & $0.81 \pm 0.08$ & $\mathbf{0.82 \pm 0.07}$ \\ 
Constrained (7D)   & $0.09 \pm 0.05$ & $-$ & $-$	& $-$ & $-$ & $0.58 \pm 0.09$ & $\mathbf{0.70 \pm 0.09}$ \\ 
\bottomrule
\end{tabulary}
\label{tab:benchmark_results}
\end{table*}

\paragraph{Experiment Details}
For the learned samplers \algSP and \algLEGO, we use $4000$ training worlds and $100$ test worlds. Dense graph is an $r-$disc Halton graph~\citep{janson2015deterministic}: $2000$ vertices in $\real^2$ to $30,000$ vertices in $\real^8$. 
The CVAE was implemented in TensorFlow~\cite{tensorflow2015-whitepaper} with 2 dense layers of 512 units each.
Input to the CVAE is a vector encoding source and target locations and an occupancy grid. Training time over 4000 examples ranged from 20 minutes in $\real^2$ to 60 minutes in $\real^8$ problems. At test time, we time-out samplers after $5$ sec. The code is open sourced\footnote{\url{https://github.com/personalrobotics/lego}} with more details in \cite{vernwal2018roadmaps}.

\subsection{Performance Analysis}

\paragraph{Sampling time}
Table~\ref{tab:timing_results} reports the average time each sampler takes for $200$ samples across $100$ test instances. \algSP and \algLEGO are the fastest. \algMAPRM and \algRBB both rely on heavy computation with multiple collision checking steps. \algWIS, by tetrahedralizing the workspace and identifying narrow passages, is relatively faster but slower than the learners. Unfortunately, some of the baselines time-out on manipulator planning problem due to expense of collision checking. 


\paragraph{Success Rate}

Table~\ref{tab:benchmark_results} reports the success rates ($95\%$ confidence intervals) on $100$ test instances when sampling $500$ vertices. Success rate is the fraction of problems for which the search found a feasible solution. \algLEGO has the highest success rate. The baselines are competitive in $\real^2$, but suffer for higher dimensional problems. 
\paragraph{Normalized Path Cost}
This is the ratio of cost of the computed solution w.r.t. the cost of the solution on the dense graph. Fig.~\ref{fig:experiments} shows the normalized cost for \algHalton, \algSP and \algLEGO - these were the only baselines that consistently had bounded $95\%$ confidence intervals (i.e. when success rate is $\geq 60\%$). \algSP has the lowest cost, however \algLEGO is within $10\%$ bound of the optimal. 



\begin{figure*}[!ht]
\centering
  \begin{subfigure}[b]{0.16\textwidth}
    \centering
    \includegraphics[height=7.56em]{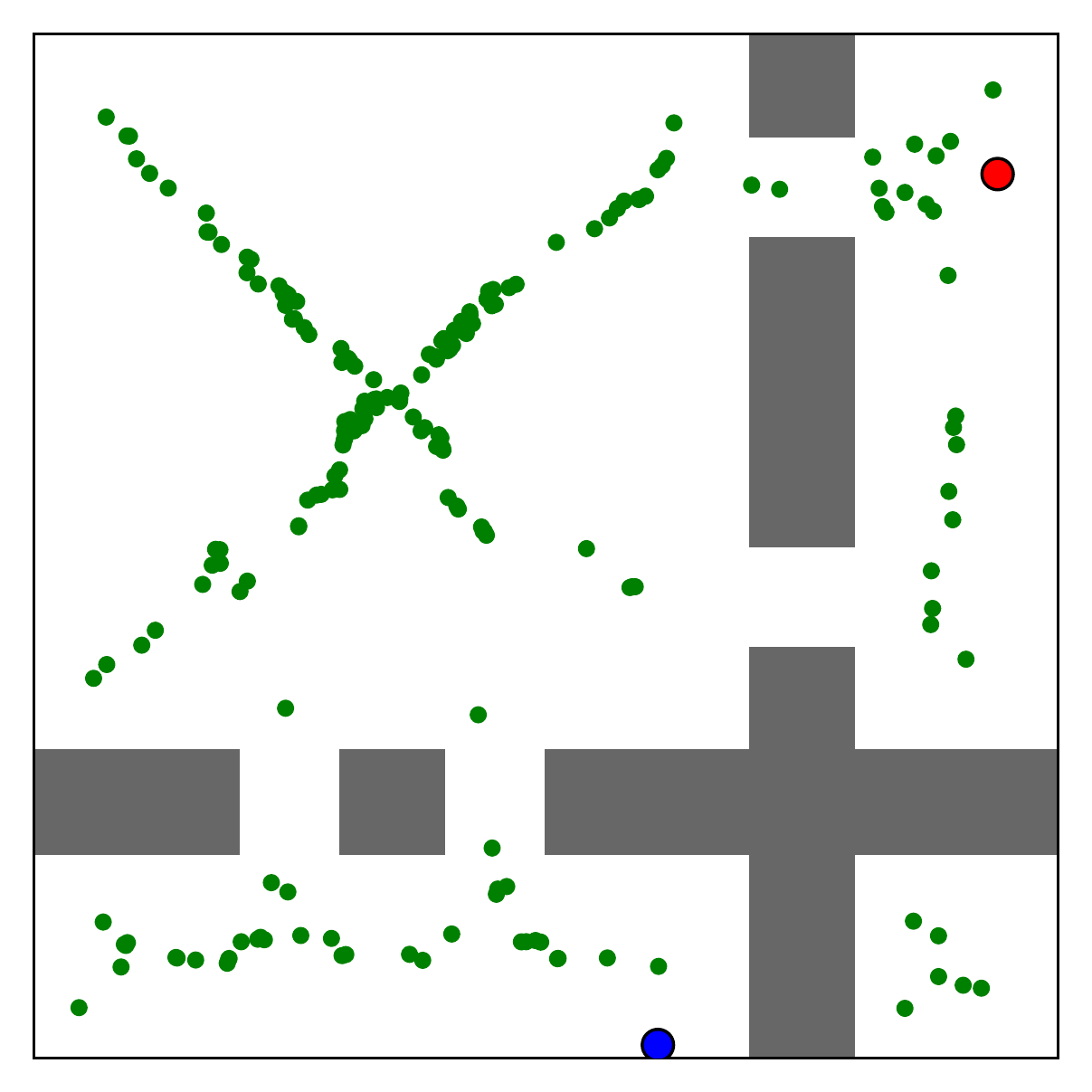}
    \caption{\bf \algMAPRM}
    \label{fig:baseline:maprm}
  \end{subfigure}
    \begin{subfigure}[b]{0.16\textwidth}
    \centering
    \includegraphics[height=7.56em]{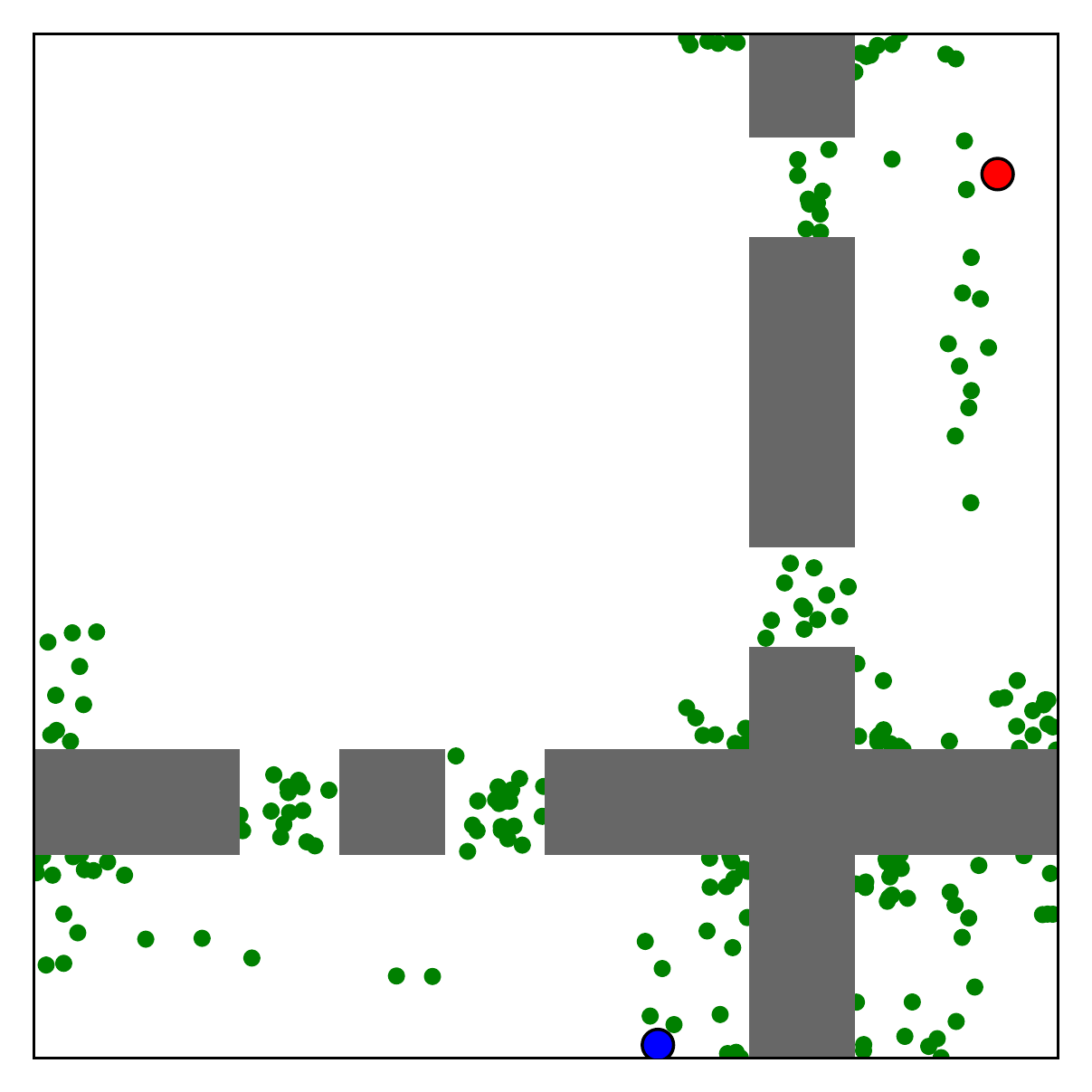}
    \caption{\bf \algRBB}
    \label{fig:baseline:rbb}
  \end{subfigure}
    \begin{subfigure}[b]{0.16\textwidth}
    \centering
    \includegraphics[height=7.56em]{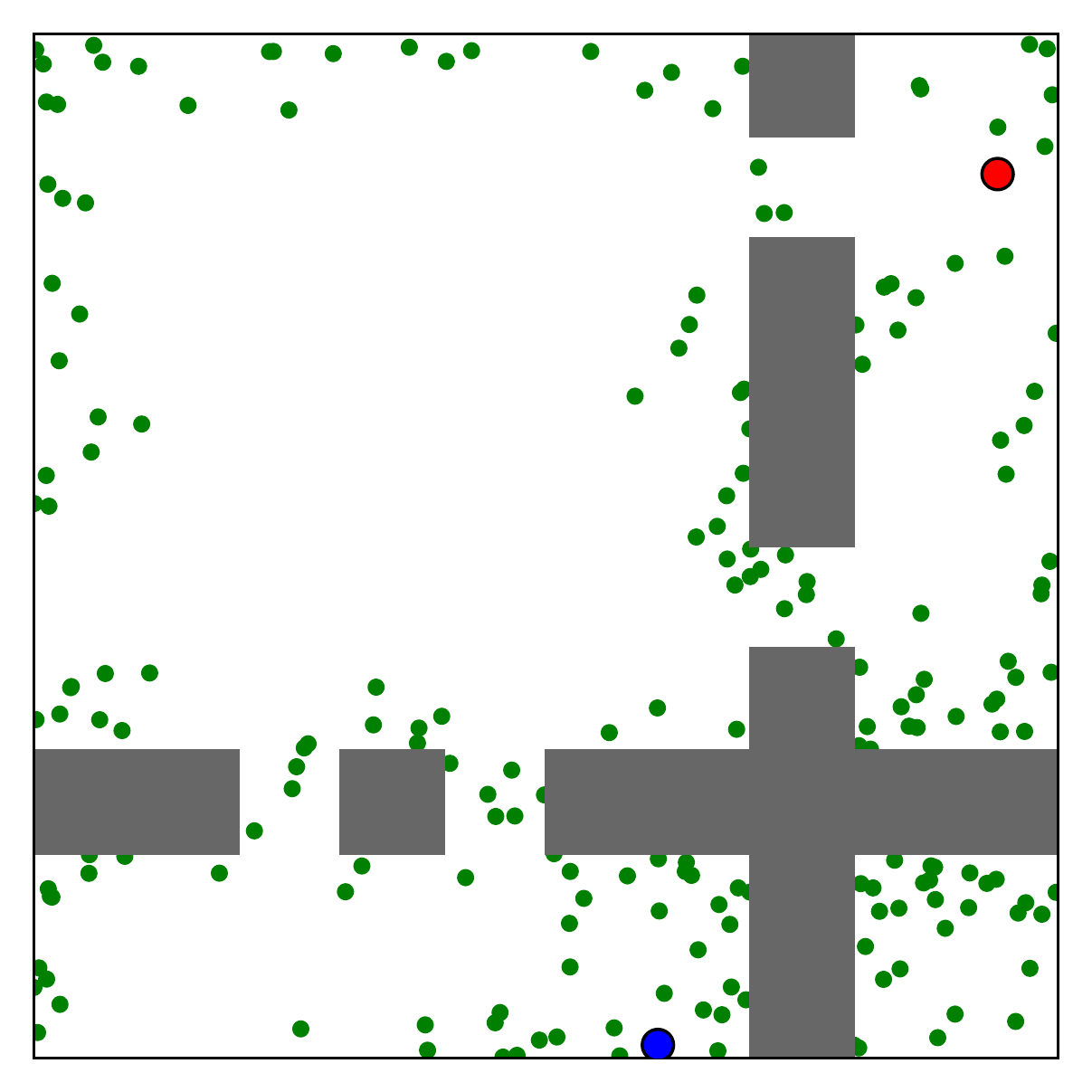}
    \caption{\bf \algGaussian}
    \label{fig:baseline:gaussian}
  \end{subfigure}
    \begin{subfigure}[b]{0.16\textwidth}
    \centering
    \includegraphics[height=7.56em]{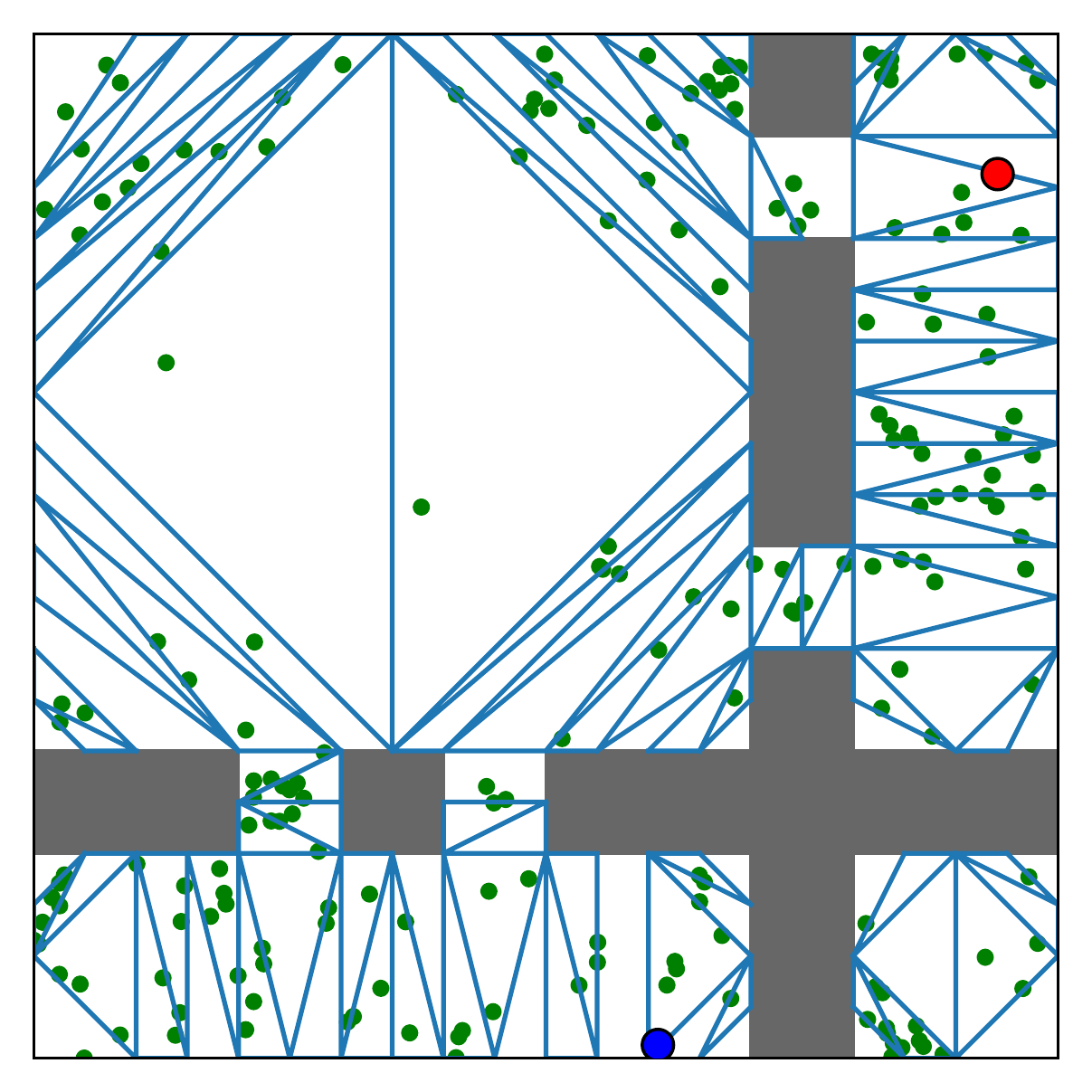}
    \caption{\bf \algWIS}
    \label{fig:baseline:wis}
  \end{subfigure}
    \begin{subfigure}[b]{0.16\textwidth}
    \centering
    \includegraphics[height=7.56em]{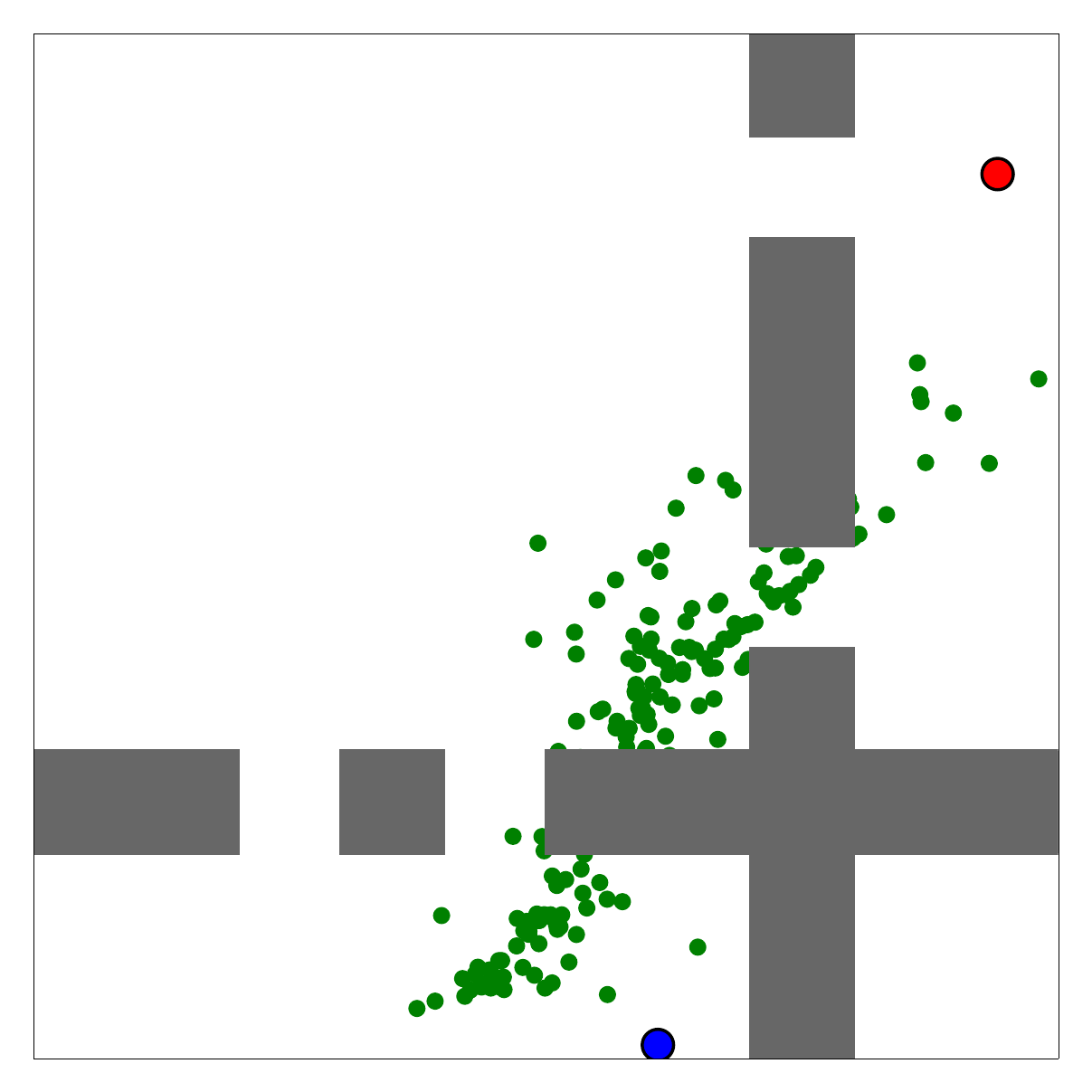}
    \caption{\bf \algSP}
    \label{fig:baseline:sp}
  \end{subfigure}
    \begin{subfigure}[b]{0.16\textwidth}
    \centering
    \includegraphics[height=7.56em]{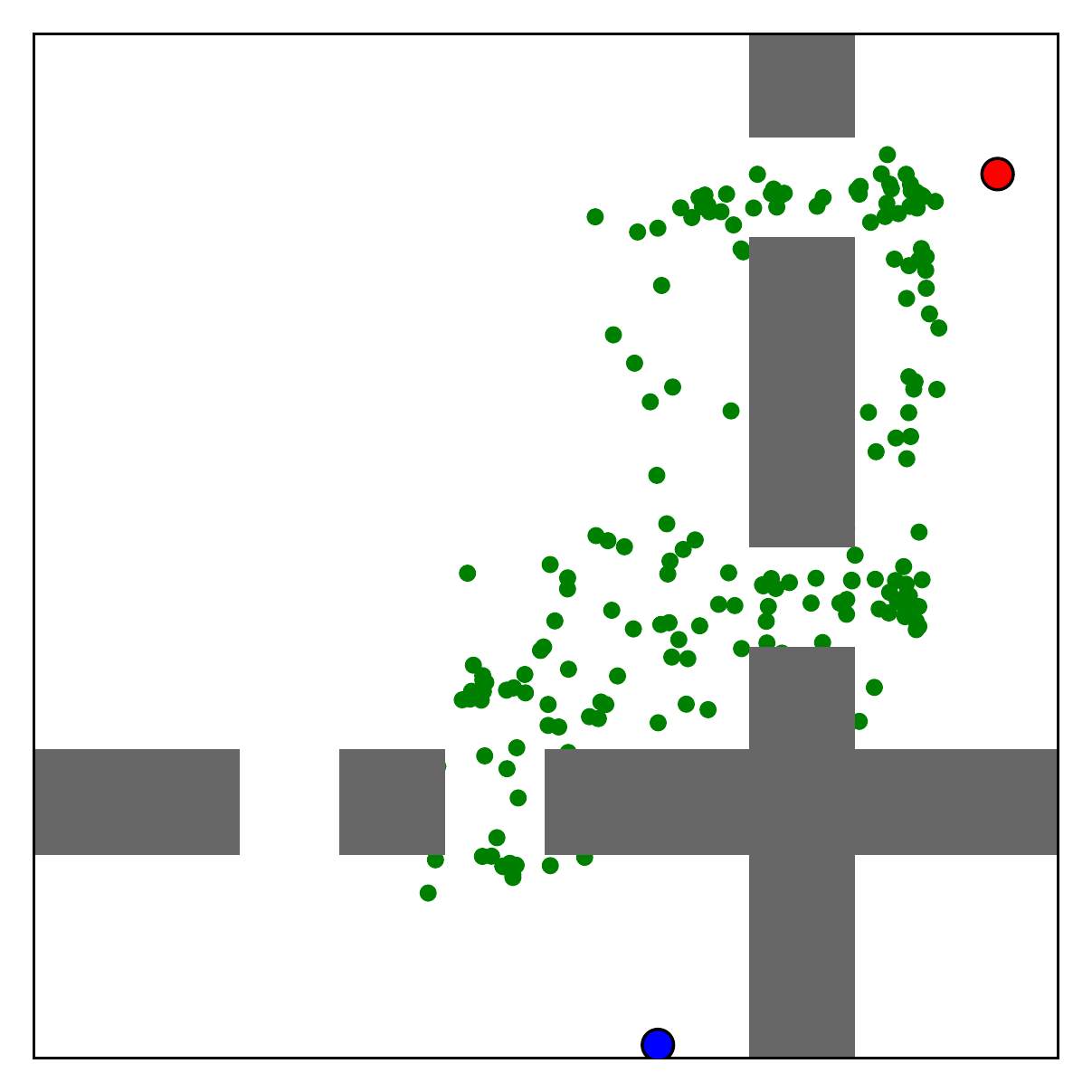}
    \caption{\bf \algLEGO}
    \label{fig:baseline:lego}
  \end{subfigure}
  \vspace{0.25em}
\caption{Comparison of samples (green) generated by all baseline algorithms on a 2D problem, planning from start (blue) to goal (red).}
\label{fig:baselines}
\end{figure*}

\begin{figure*}[!t]
    \centering
    \includegraphics[width=1.0\textwidth]{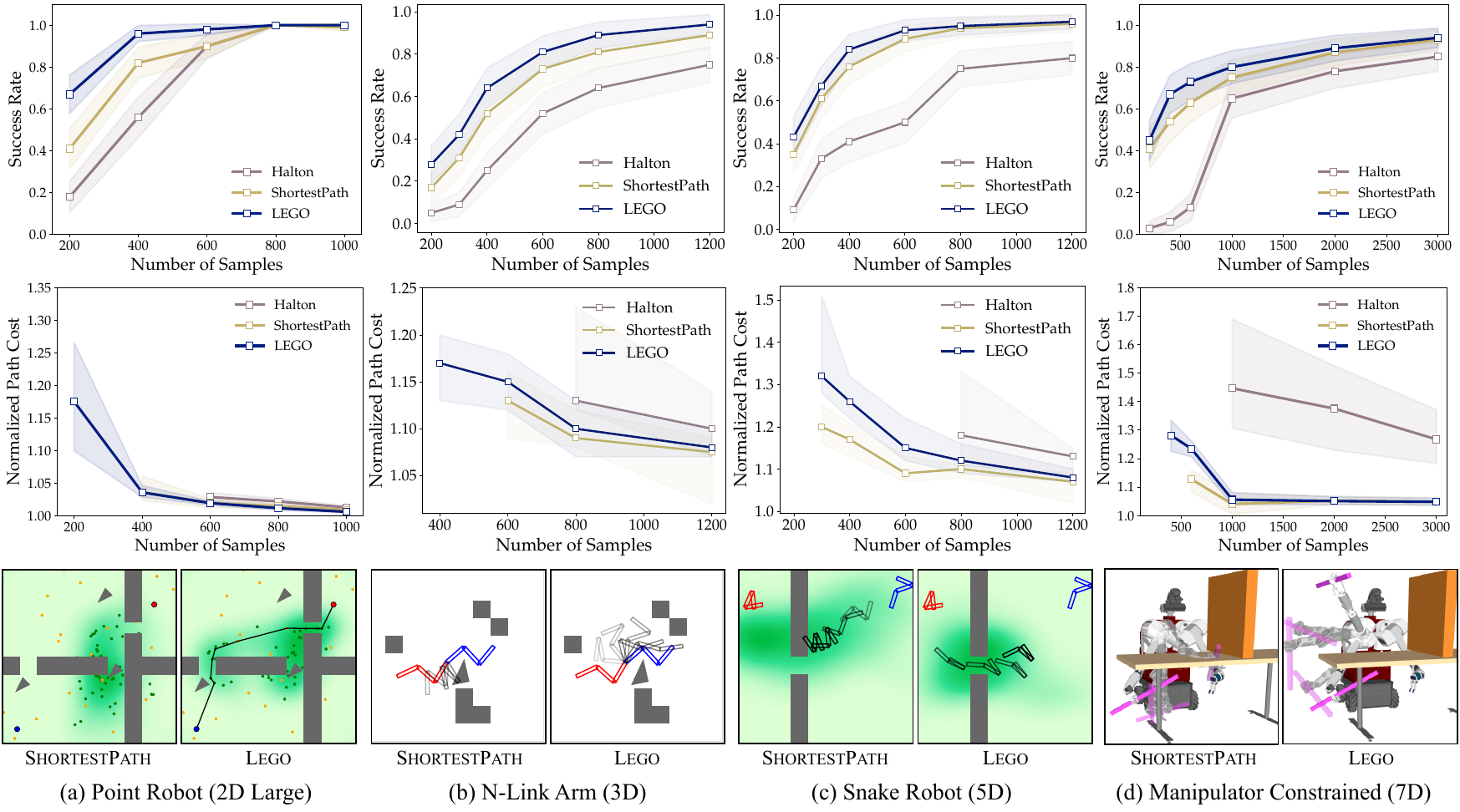}
    \caption{%
    \label{fig:experiments}
    Comparisons of \algSP against \algLEGO on different problem domains. In each problem domain (column), success rate (top), normalized path lengths (middle) and solutions determined on the roadmaps constructed using samples generated by the two samplers (bottom) are shown. \fullFigGap
    }
\end{figure*}%


\begin{figure}[!ht]
\centering
  \begin{subfigure}[b]{0.48\columnwidth}
    \centering
    \includegraphics[width=\columnwidth]{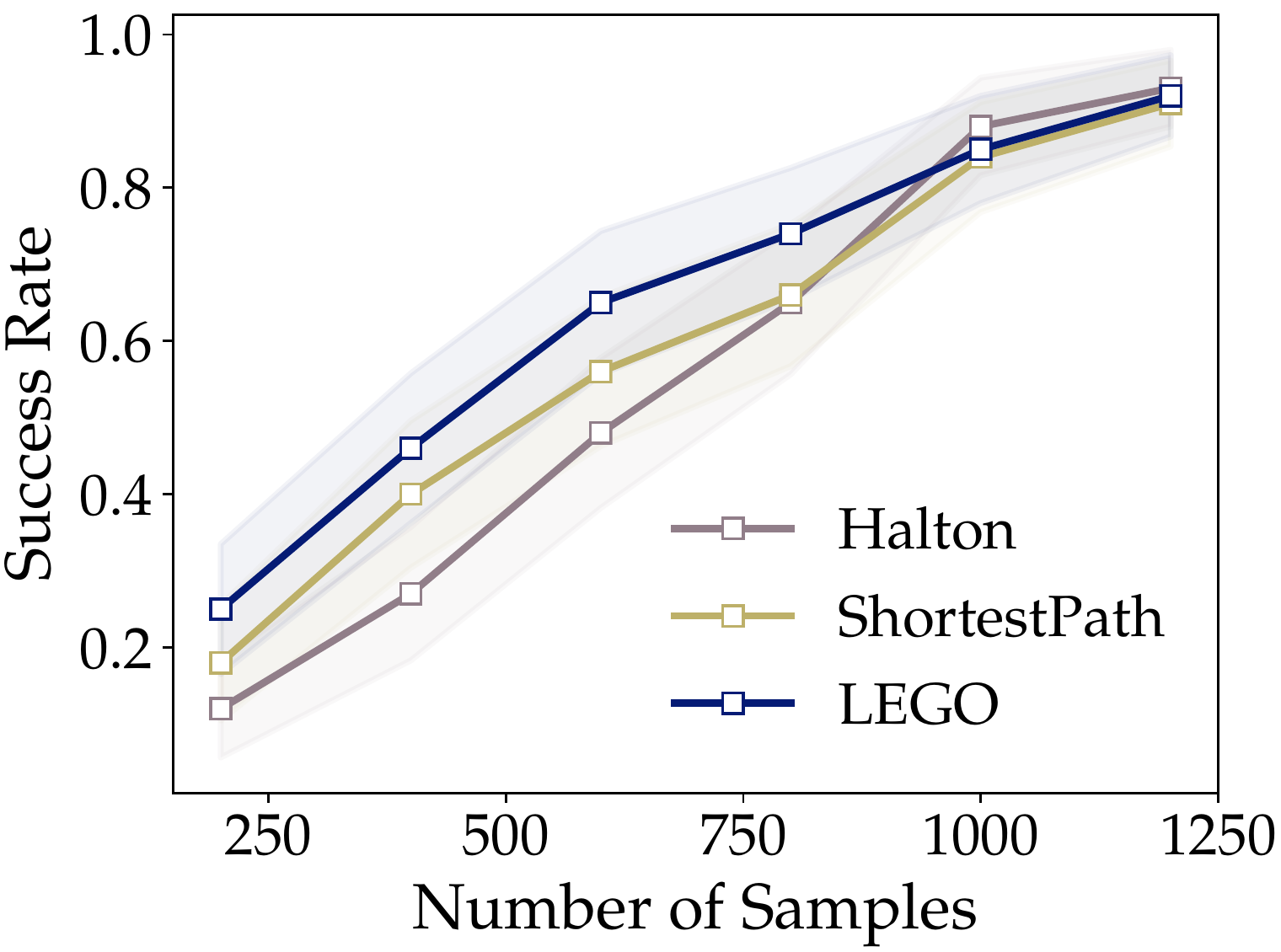}
    \caption{Corrupted Environment 1}
    \label{fig:mismatch_result:env2_sr}
  \end{subfigure}
  \begin{subfigure}[b]{0.48\columnwidth}
    \centering
    \includegraphics[width=\columnwidth]{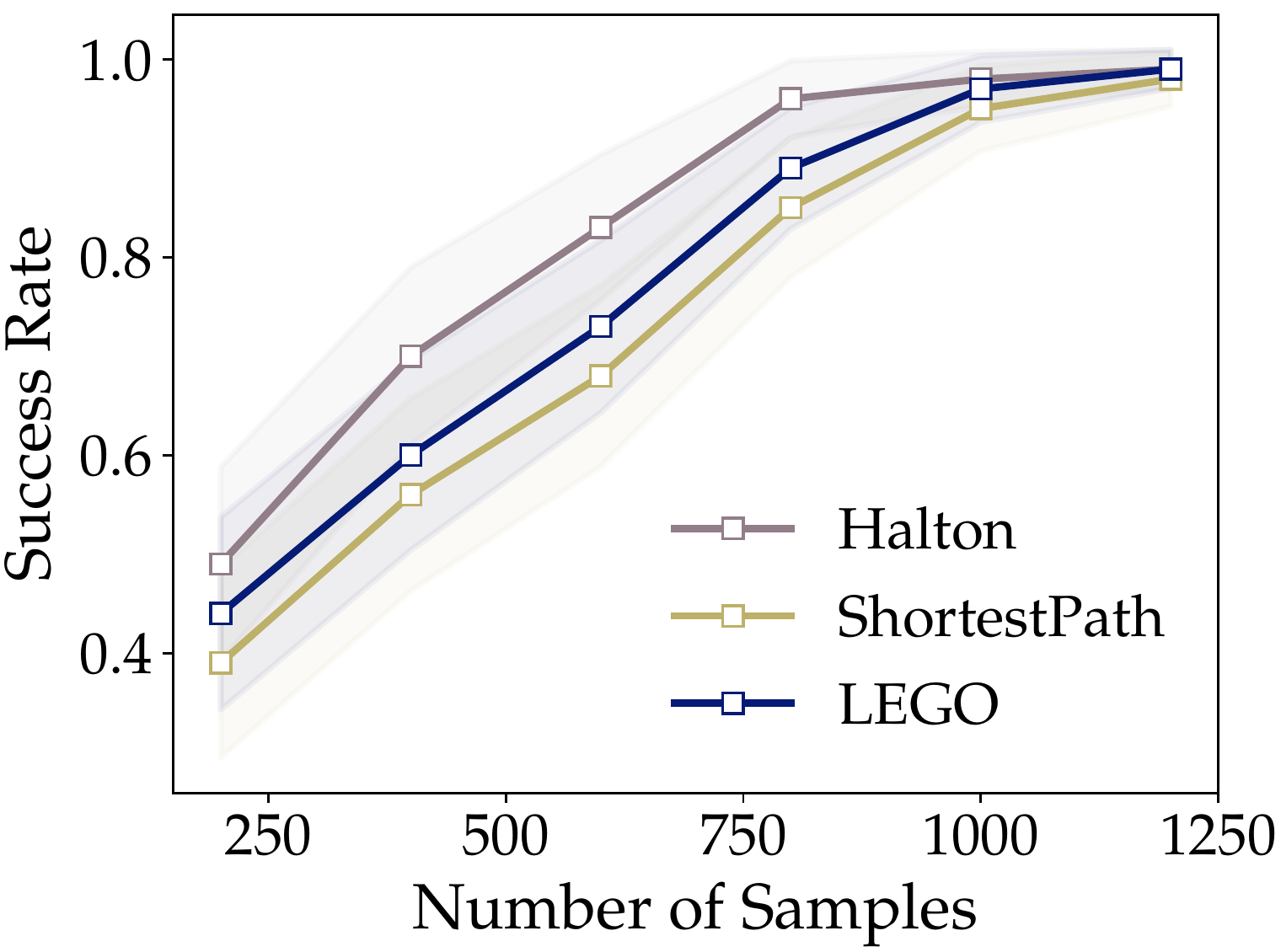}
    \caption{Corrupted Environment 2}
    \label{fig:mismatch_result:env2_sr}
  \end{subfigure}
  \begin{subfigure}[b]{0.24\columnwidth}
    \centering
    \includegraphics[width=\columnwidth]{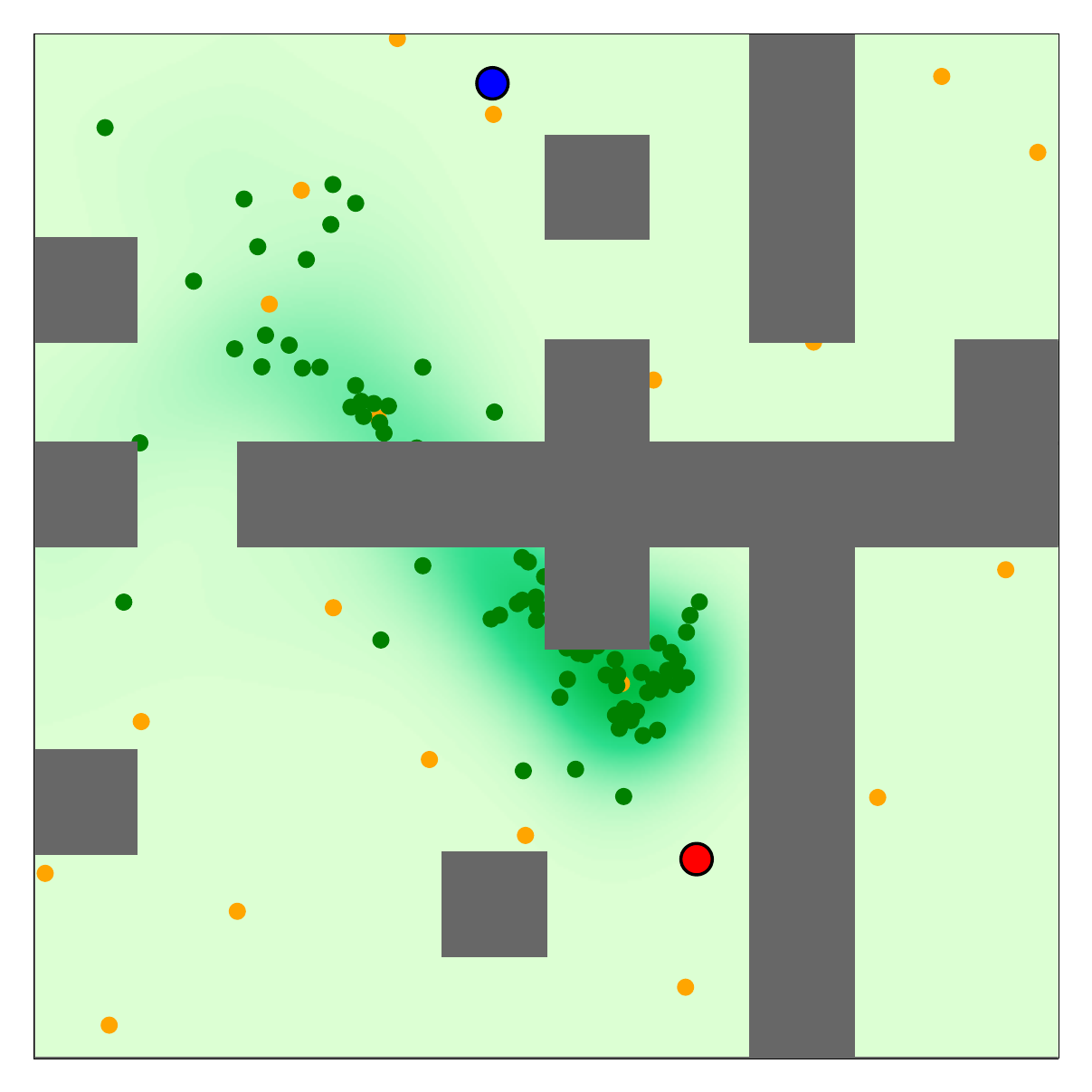}
    \caption{\scriptsize \algSP}
    \label{fig:mismatch_result:env2_sp}
  \end{subfigure}
  \begin{subfigure}[b]{0.24\columnwidth}
    \centering
    \includegraphics[width=\columnwidth]{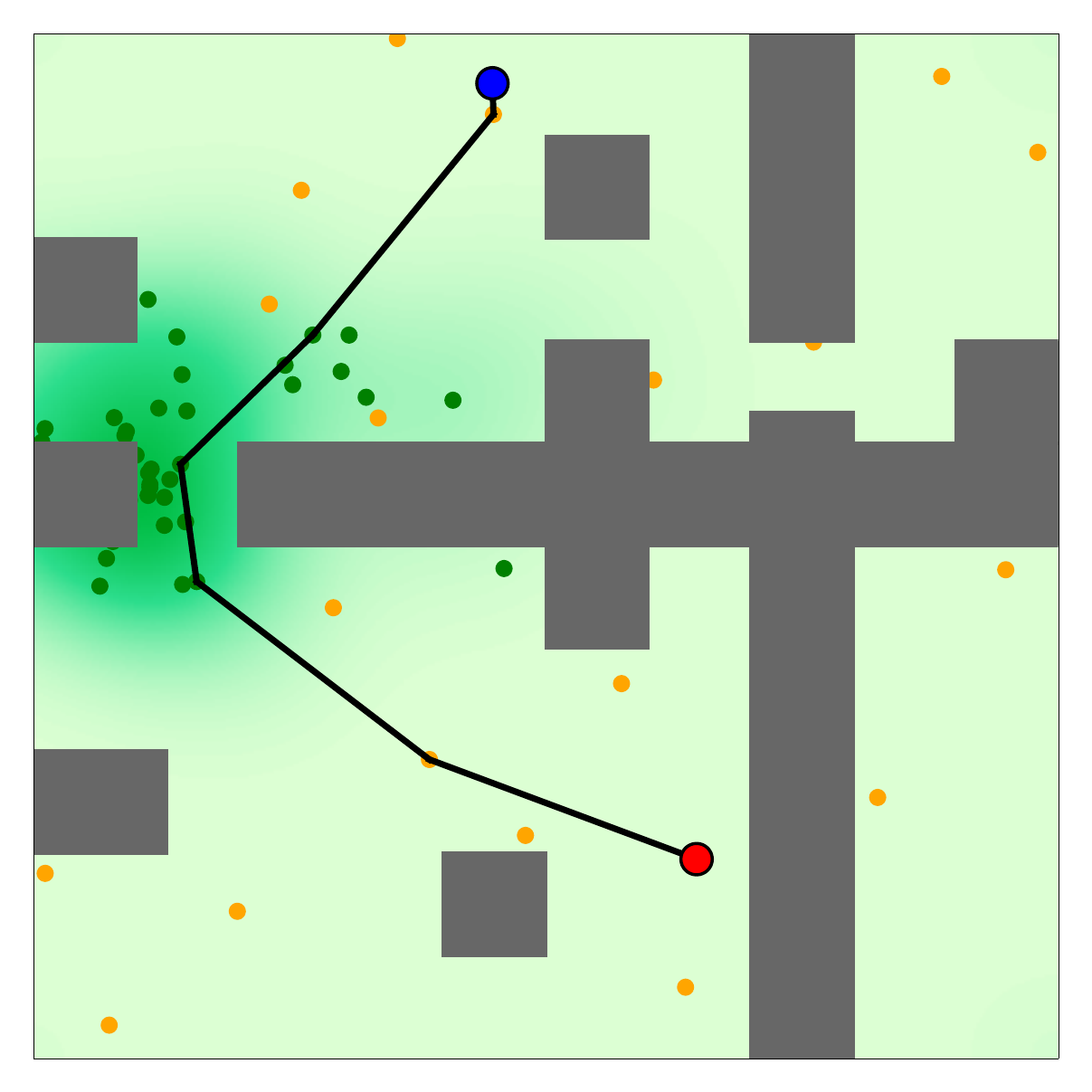}
    \caption{\scriptsize \algLEGO}
    \label{fig:mismatch_result:env2_sp}
  \end{subfigure}
    \begin{subfigure}[b]{0.24\columnwidth}
    \centering
    \includegraphics[width=\columnwidth]{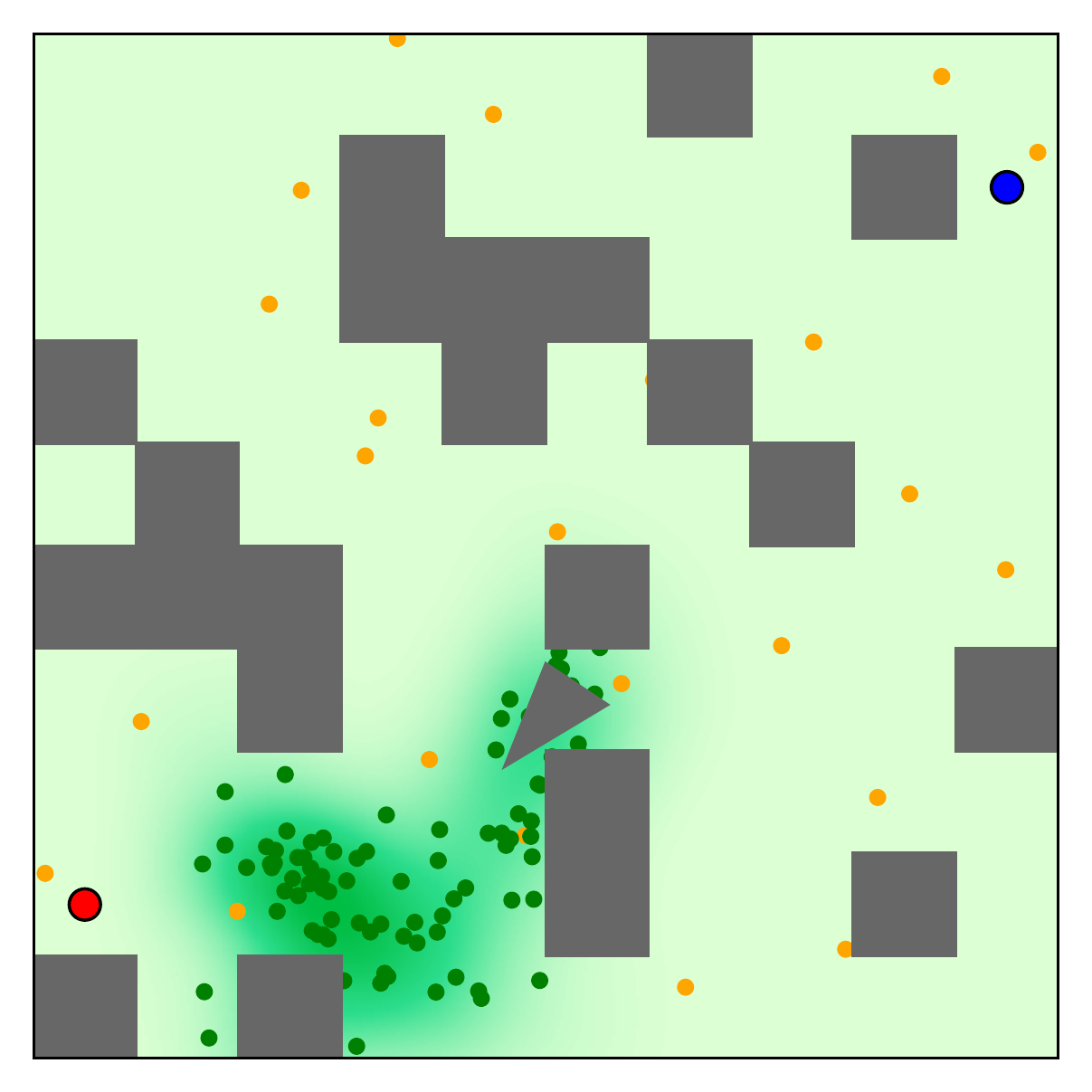}
    \caption{\scriptsize \algSP}
    \label{fig:mismatch_result:env2_sp}
  \end{subfigure}
    \begin{subfigure}[b]{0.24\columnwidth}
    \centering
    \includegraphics[width=\columnwidth]{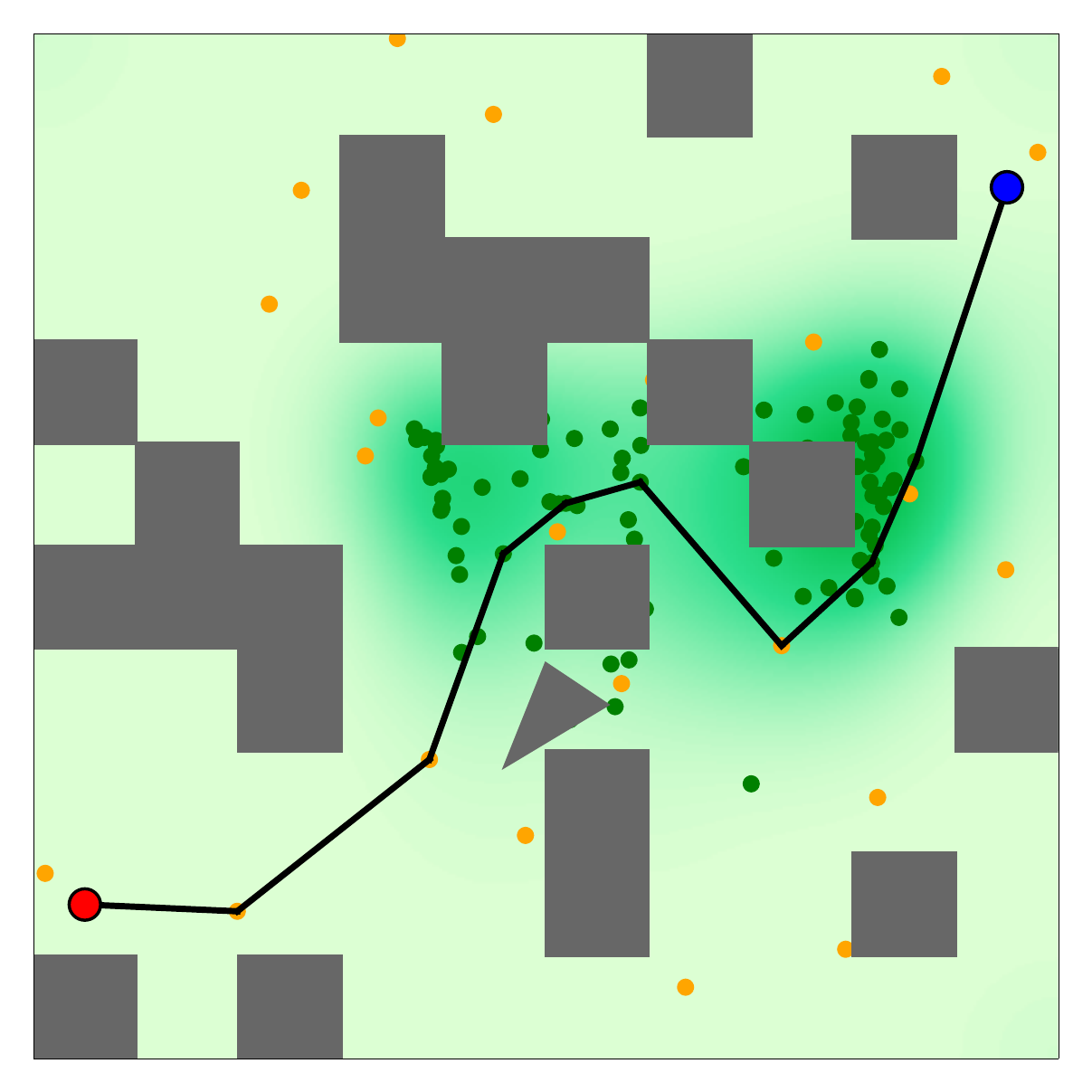}
    \caption{\scriptsize \algLEGO}
    \label{fig:mismatch_result:env2_sp}
  \end{subfigure}
\caption{Comparison of samplers on corrupted environments, i.e., different from training dataset. Success rate on (a) less corrupted environment 1: mixture of walls and random squares and (b) more corrupted environment 2: only squares. Output of \algSP and \algLEGO on environment 1 (c,d) and environment 2 (e,f).}
\label{fig:mismatch_result}
\end{figure}


\subsection{Observations}
We report on some key observations from Table~\ref{tab:benchmark_results} and \figref{fig:experiments}.
\begin{observation}
\algLEGO consistently outperforms all baselines
\end{observation}\vspace{-0.5em}
As shown in Table~\ref{tab:benchmark_results}, \algLEGO has the best success rate (for $500$ samples) on all datasets. 
The second row in~\figref{fig:experiments} shows that \algLEGO is within $10\%$ bound of the optimal path.
\begin{observation}
\algLEGO places samples only in regions where the optimal path may pass.
\end{observation}\vspace{-0.5em}
\figref{fig:baselines} shows samples generated by various baseline algorithms on a 2D problem. The heuristic baselines use various strategies to identify important regions - \algMAPRM finds medial axes, \algRBB finds bridge points, \algGaussian samples around obstacles, \algWIS divides up space non-uniformly and samples accordingly. However, these methods places samples everywhere irrespective of the query. \algSP takes the query into account but fails to find the gaps. \algLEGO does a combination of both -- it finds the right gaps.
\begin{observation}
\algLEGO has a higher performance gain on harder problems (narrow passages) as it focuses on bottlenecks.
\end{observation}\vspace{-0.5em}
Table~\ref{tab:benchmark_results} shows how success rates vary in 2D problems with small / medium / large gaps. As the gaps gets narrower, \algLEGO outperforms more dominantly. The \algBottleneck component in \algLEGO seeks the bottleneck regions (\figref{fig:bottleneck_bottleneck}).

For manipulator planning $\real^8$ problems, when stick is unconstrained, \algLEGO and \algSP are almost identical. We attribute this to such problems being easier, i.e. the shortest path simply slides the stick out of the way and plans to the goal. When the stick is constrained, \algLEGO does far better. \figref{fig:experiments}(d) shows that \algLEGO is able to sample around the table while \algSP cannot find this path. 

\begin{observation}
\algLEGO is robust to a certain degree of train-test mismatch as it encourages diversity. 
\end{observation}\vspace{-0.5em}
\figref{fig:mismatch_result} shows the success rate of learners on a 2D test environment that has been corrupted. Environment 1 is less corrupted than environment 2. \figref{fig:mismatch_result}(a) shows that on environment 1, \algLEGO is still the best sampler. \algSP (\figref{fig:mismatch_result}(c)) ignores the corruption in the environment and fails. \algLEGO (\figref{fig:mismatch_result}(d)) still finds the correct bottleneck. \figref{fig:mismatch_result}(b) shows that all learners are worse than \algHalton. \algSP(\figref{fig:mismatch_result}(e)) densifies around a particular constrained region while \algLEGO(\figref{fig:mismatch_result}(e)) still finds a path due to the \algDiversity component sampling in multiple bottleneck regions. 



\section{Discussion}
\label{sec:discussion}
\vspace{-2mm}
We present a framework for training a generative model to predict roadmaps for sampling-based motion planning. We build upon state-of-the-art methods that train the CVAE using the shortest path as target input. We identify important failure modes such as complex obstacle configurations and train-test mismatch. Our algorithm \algLEGO directly addresses these issues by training the CVAE using \emph{diverse bottleneck nodes} as target input. We formally define these terms and provide provable algorithms to extract such nodes. Our results indicate that the predicted roadmaps outperform competitive baselines on a range of problems.

Using priors in planning is a double edged sword. While one can get astounding speed ups by focusing search on a tiny portion of C-space~\citep{ichter2017learning}, any problem not covered in the dataset can lead to catastrophic failures. This is symptomatic of the fundamental problem of \emph{over-fitting} in machine learning. While one could ensure the training data covers all possible environments~\citep{tobin2017domain}, an algorithmic solution is to explore regularization techniques for planning. We argue \algDiversity can be viewed as a form of regularization.

We can also include a more informed conditioning vector that captures the state of the search, e.g., the length of the current shortest path. This is similar to Informed RRT*~\citep{gammell2014informed}. Finally, we wish to scale to problems with varying workspace where a global planner guides the sampler to focus on relevant parts of the workspace~\citep{HsuLatMot99, zheng2016generating}.


\footnotesize{
\bibliographystyle{unsrtnat}
\bibliography{reference}

\clearpage

\fontsize{10}{12}\selectfont

\begin{appendices}

\section{CVAE Framework}
\label{sec:appendix_framework}
We refer the reader to \cite{doersch2016tutorial} for technical details and a comprehensive tutorial on CVAE. In \sref{sec:cvae_architecture} we describe the CVAE architecture implemented to train \algLEGO and \algSP algorithms. In Sections \ref{sec:cvae_latent_variable_dimension} and \ref{sec:cvae_regularization_parameter}, we study two parameters that determine the performance of the CVAE generative model.

\subsection{Architecture}
\label{sec:cvae_architecture}

The entire CVAE module (\figref{fig:cvae_arch}) takes as input the training samples $X$, which in case of \algLEGO are the samples in bottleneck regions and along diverse paths. Additionally, the CVAE takes as input a vector of external features $y$, upon which the generative model is also conditioned upon. In the problems we consider, these features include information regarding the environment such as the poses of the obstacles and the start-goal pair. A standard CVAE model consists of an encoder and a decoder, often represented by neural networks trained using the input samples and the external features.

During training, the encoder network takes as input the high-dimensional vector of features including the training sample and the other external features and encodes it into a low-dimensional latent variable vector. The latent variable is then fed into the decoder network along with the vector of external features as an input which outputs a sample in the configuration space. This sample output by the decoder is used to minimize an objective function which aims to fundamentally reduce the divergence between the probability distribution of the training samples and the learned generative model to be able to closely reconstruct the training samples set. During testing, only the decoder network is used to generate the required samples. The decoder takes as input a latent variable sampled from standard normal distribution as well as the vector of external features to generate useful samples. 

In our implementation of the CVAE, both encoder and decoder networks have two fully connected hidden layers with 512 units each. The specifics of the external features used in each of the planning problems considered in \sref{sec:results} are discussed in \sref{sec:appendix_exp_training}. The behavior of the generative model, in addition to the features used, also depends on certain parameters. We study the effect of these parameters and their design choices in our implementation in the following subsections.

\begin{figure}[!ht]
  \centering
  \includegraphics[width=\columnwidth]{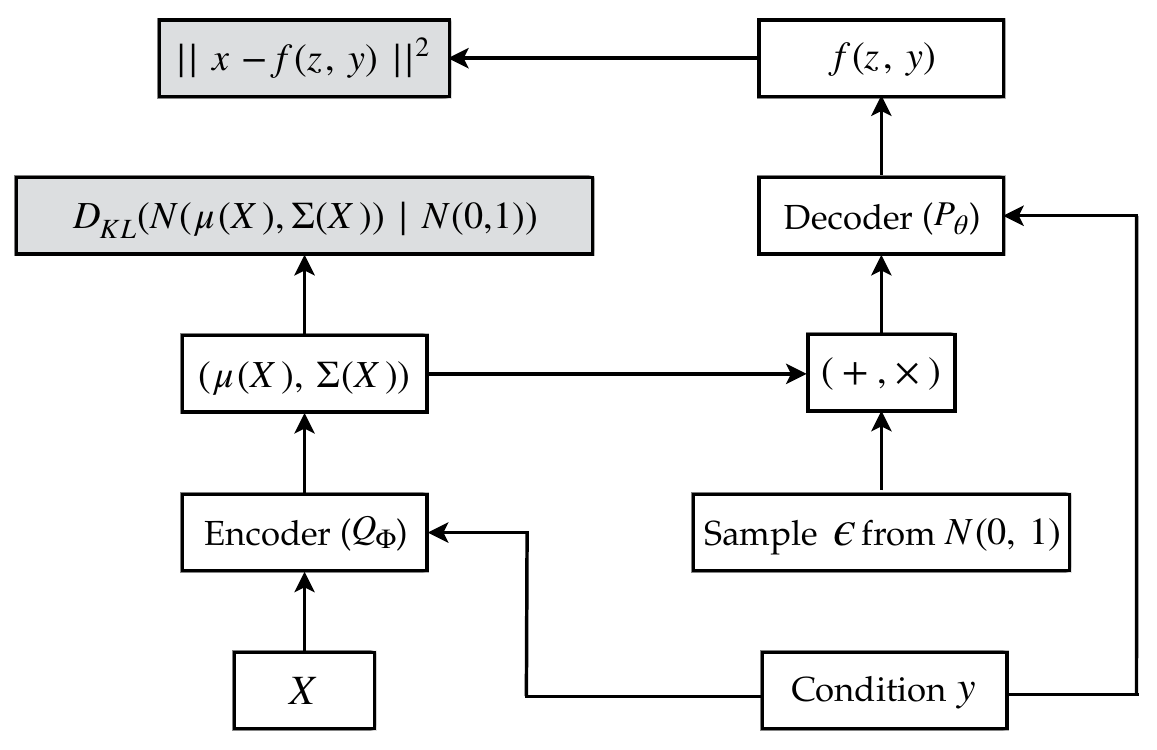}
  \caption{A simple illustration of the CVAE framework setup for training with $X$ and $y$ together denoting the input to the CVAE. \fullFigGap
  }
  \label{fig:cvae_arch}
  \vspace{2mm}
\end{figure}

\subsection{Dimensionality of the Latent Variable} 
\label{sec:cvae_latent_variable_dimension}
The latent variable captures the information available to the model through the training examples in a lower dimensional latent space. The dimensionality of the latent variable denotes how efficiently the model can capture the sources of variability required to regenerate data similar to the training examples. Theoretically, a model with larger latent dimension is at least as good as a model with lower latent dimension. However, in practice, when the latent variable dimension is high, it becomes computationally expensive for methods like stochastic gradient descent to reduce the KL divergence between the true and the approximated distributions over the latent variables conditioned on the training examples. \figref{fig:dimensionality_figs} shows the behaviors exhibited by the trained generative model for different latent variable dimensions. We choose latent variable dimension of 3 for $\real^2,~\real^5$ problems and 5 for $\real^7,~\real^8$ and $\real^9$ problems. 

\begin{figure}[!ht]
  \centering
  \begin{subfigure}[b]{0.32\linewidth}
    \centering
    \includegraphics[width=\linewidth]{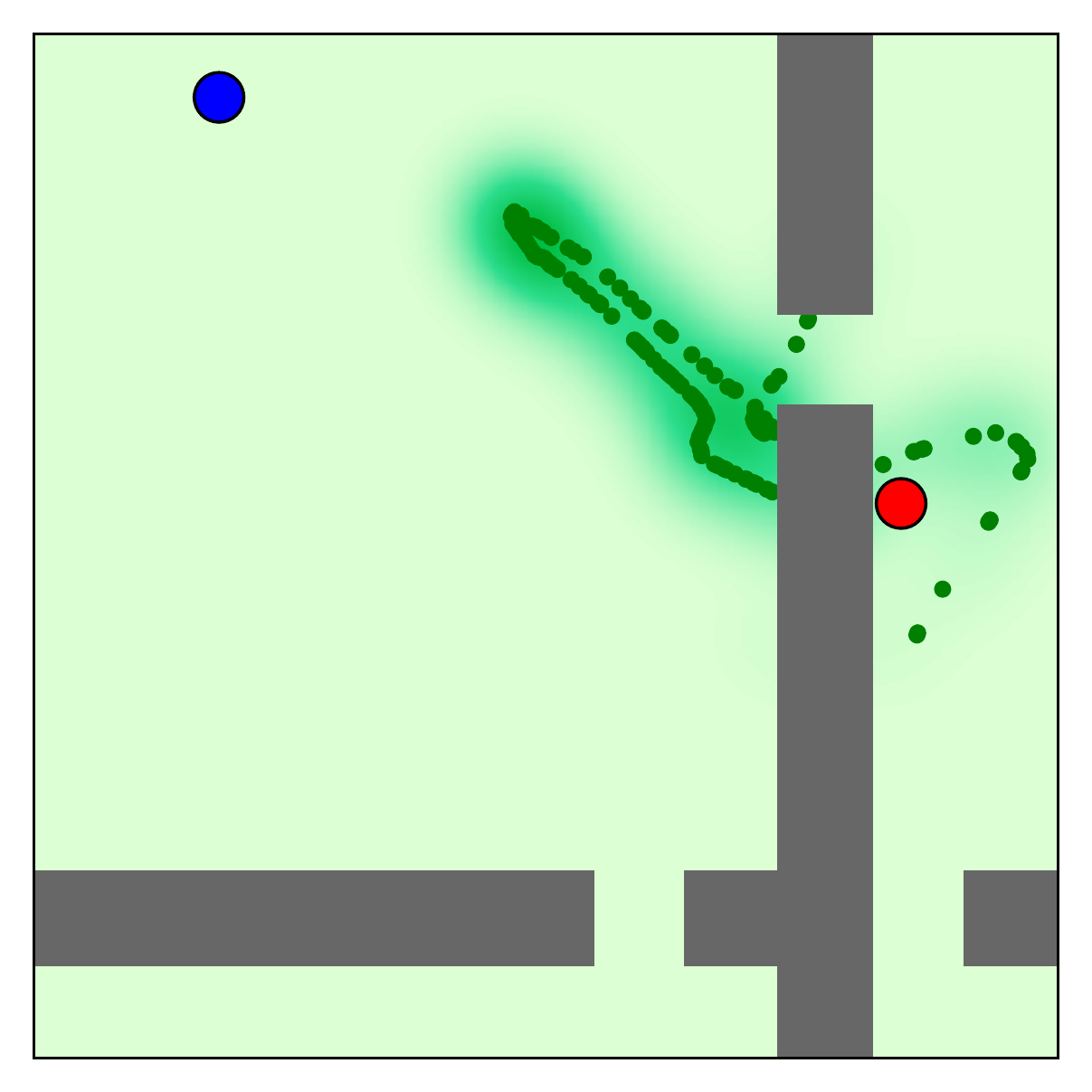}
    \caption{}
    \label{fig:regularization_a}
  \end{subfigure} \hfill
  \begin{subfigure}[b]{0.32\linewidth}
    \centering
    \includegraphics[width=\linewidth]{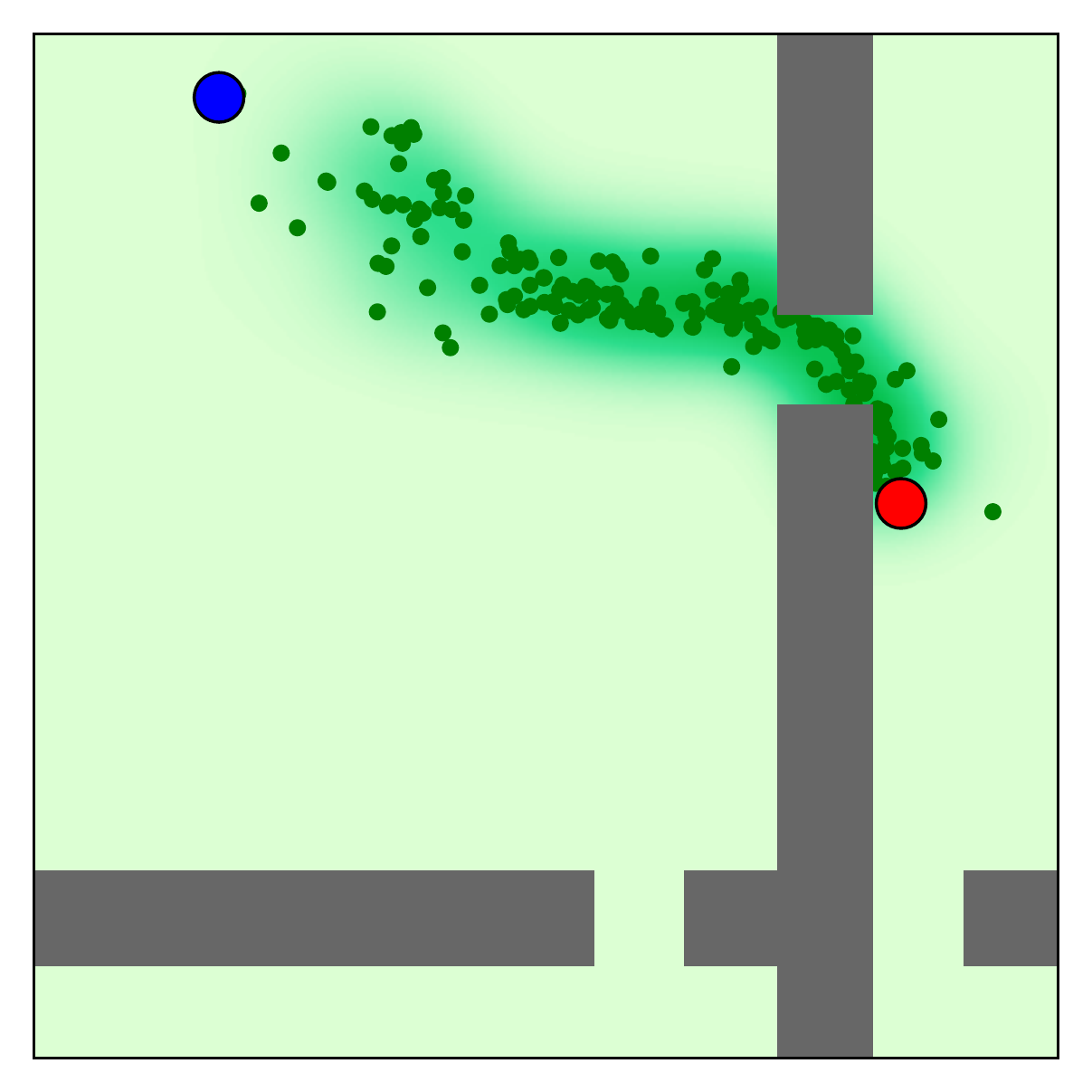}
    \caption{}
    \label{fig:regularization_b}
  \end{subfigure} \hfill
  \begin{subfigure}[b]{0.32\linewidth}
    \centering
    \includegraphics[width=\linewidth]{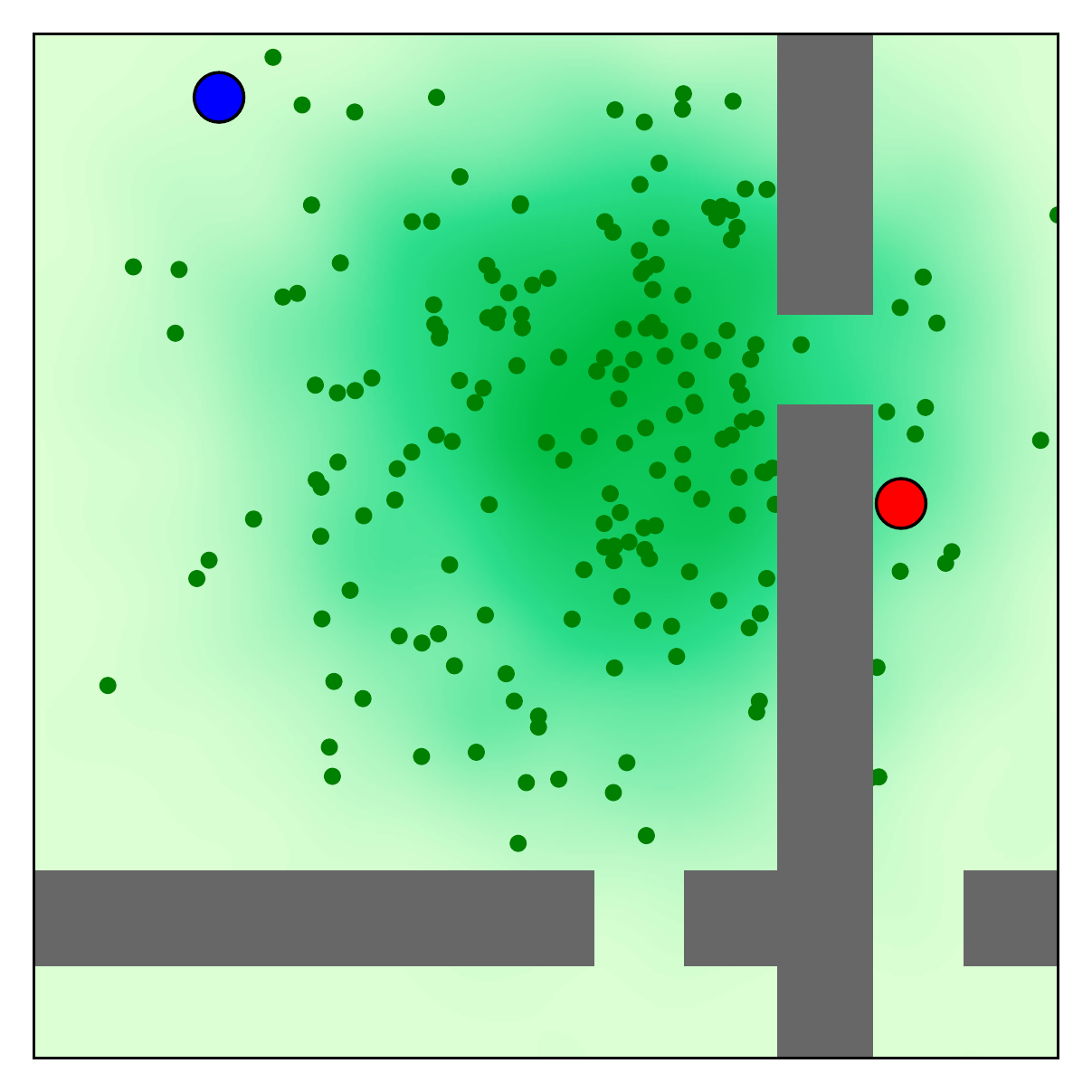}
    \caption{}
    \label{fig:regularization_c}
  \end{subfigure}
  \caption{Samples generated by CVAE trained with different latent variable dimensions (a) 1 (b) 3 and (c) 7. \fullFigGap
  }
  \label{fig:dimensionality_figs}
  \vspace{2mm}
\end{figure}

\subsection{Regularization Parameter} 
\label{sec:cvae_regularization_parameter}
Although VAEs are generally devoid of regularization parameters, one could introduce the parameter in modifying the objective function the CVAE aims to minimize when learning the generative model. The objective function in a CVAE is given by:
\begin{equation}
\texttt{Reconstruction Loss} + \lambda \times \texttt{KL Divergence}
\label{eqn:cvae_loss_function}
\end{equation}
The reconstruction loss ensures that the training data can be explained with the data generated by the model and therefore minimizing it ensures proper reconstruction of the training examples. On the other hand, the second term captures the divergence between the prior distribution over latent variable and the posterior given the training examples. Minimizing it ensures that the two distributions are similar. When the value of $\lambda$ is zero, the behavior of the corresponding VAE is similar to a traditional autoencoder in its capability to reconstruct the training examples. When the value of $\lambda$ is equal to 1, the objective function is as in a VAE. However this often leads to \emph{over-pruning} \cite{yeung2017tackling}  where many of the dimensions of the latent variable are ignored in an attempt to reduce the KL divergence. By tuning the value of $\lambda$ between 0 and 1, one could weigh the two objectives appropriately to obtain the desired generative model behavior (\figref{fig:regularization_figs}).

\begin{figure}[!ht]
  \centering
  \begin{subfigure}[b]{0.32\linewidth}
    \centering
    \includegraphics[width=\linewidth]{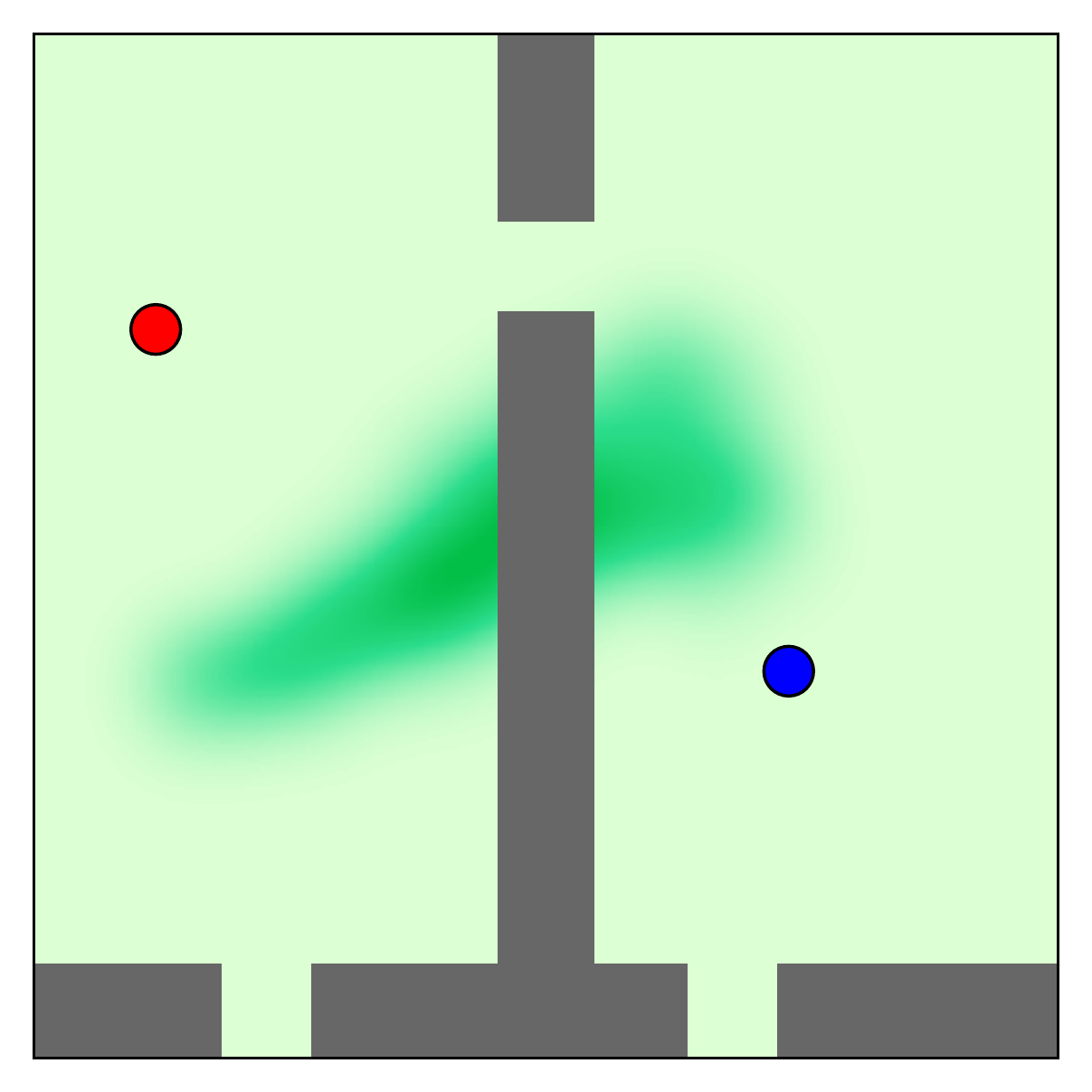}
    \caption{}
    \label{fig:regularization_a}
  \end{subfigure} \hfill
  \begin{subfigure}[b]{0.32\linewidth}
    \centering
    \includegraphics[width=\linewidth]{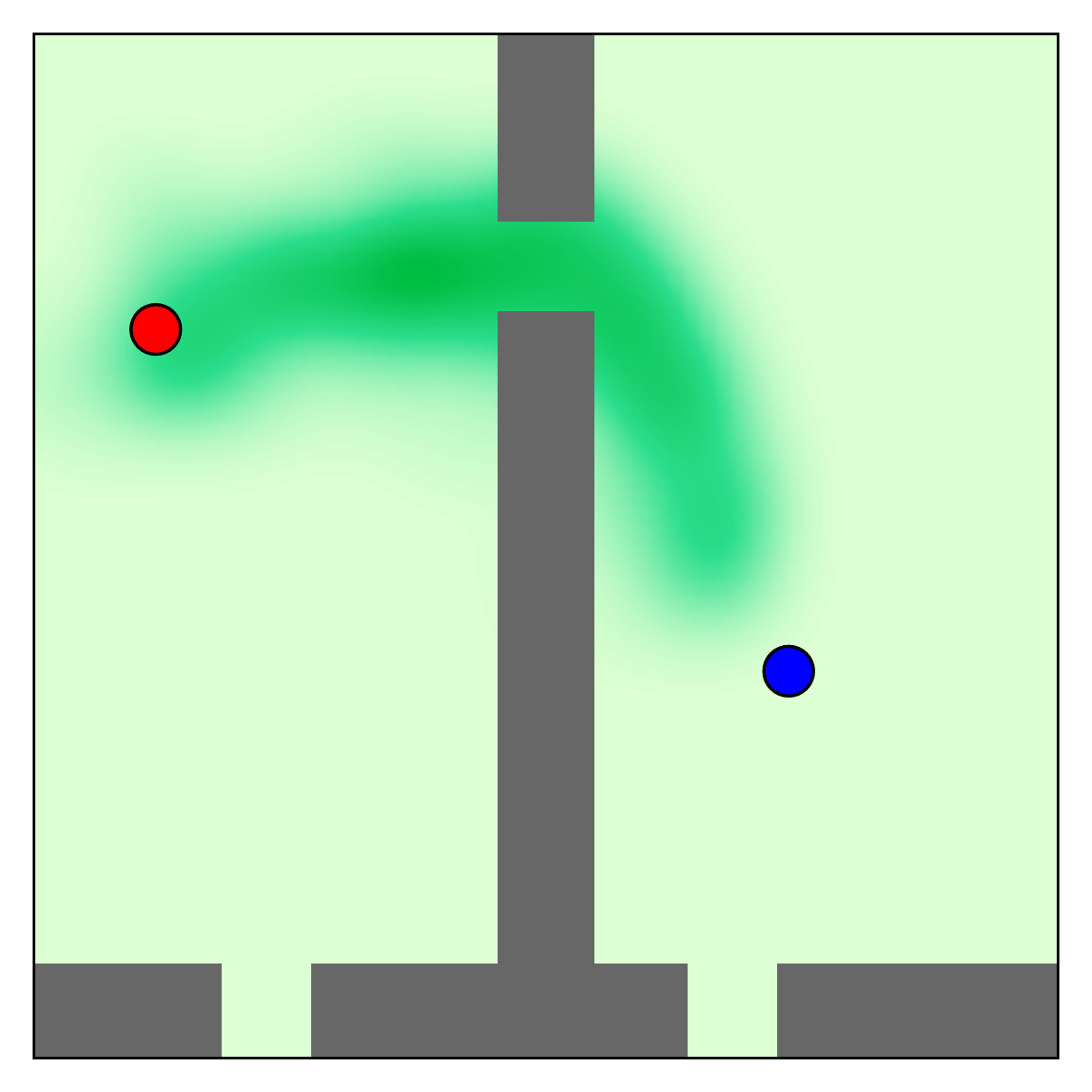}
    \caption{}
    \label{fig:regularization_c}
  \end{subfigure} \hfill
  \begin{subfigure}[b]{0.32\linewidth}
    \centering
    \includegraphics[width=\linewidth]{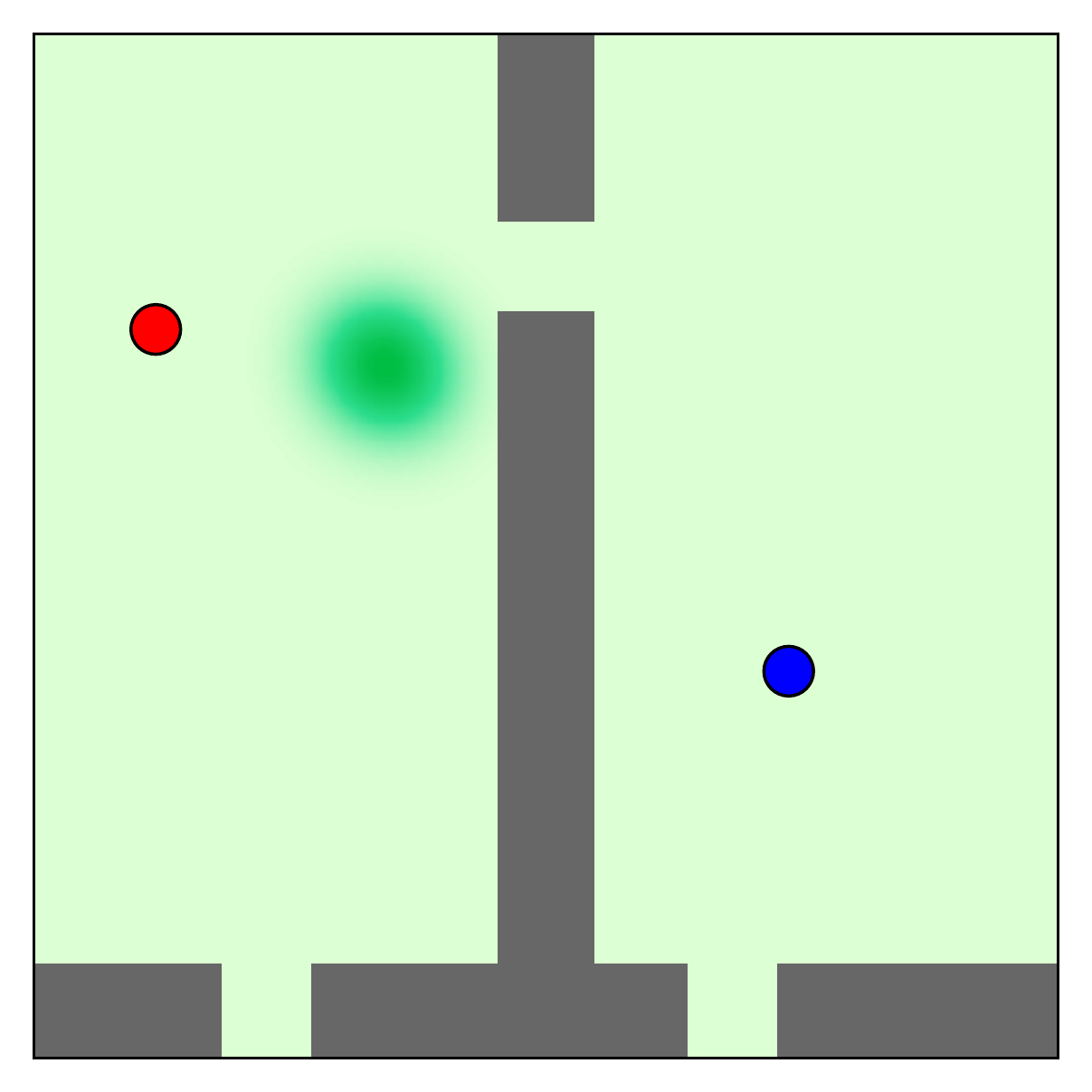}
    \caption{}
    \label{fig:regularization_b}
  \end{subfigure}
  \caption{
  Learned distributions for the narrow passage problem for different values of regularization parameter ($\lambda$), (a) $2 \times 10^{-8}$ (b) $2 \times 10^{-4}$ (chosen value of $\lambda$) (c) $2 \times 10^{-2}$. \fullFigGap
  }
  \label{fig:regularization_figs}
  \vspace{2mm}
\end{figure}

\section{Experiments}
\label{sec:appendix_experiments}

In this section, we discuss the offline computation involved in training the CVAE for different planning environments considered in \sref{sec:results}.

\subsection{Training Procedure}
\label{sec:appendix_exp_training}

\paragraph{2D Point Robot Planning} The training data consisted of 20 randomly generated environments as shown in \figref{fig:2d_environments} with 20 planning problems (start-goal pairs) in each of the environments. The environments were randomized in positions of the vertical and horizontal walls and the narrow passages through them. The CVAE was conditioned upon a vector of 102 features which included the start-goal pair (4 features) as well the $10 \times 10$ occupancy grid (100 features). The dataset generation took 4-5 hours while the training time was around 25 minutes. The CVAE was trained using samples from $\denseGraph$ with 3000 samples. The CVAE was trained to sample configurations (in $\real^2$) of the point robot.

\begin{figure}[!ht]
  \centering
  \begin{subfigure}[b]{0.32\linewidth}
    \centering
    \includegraphics[width=\linewidth]{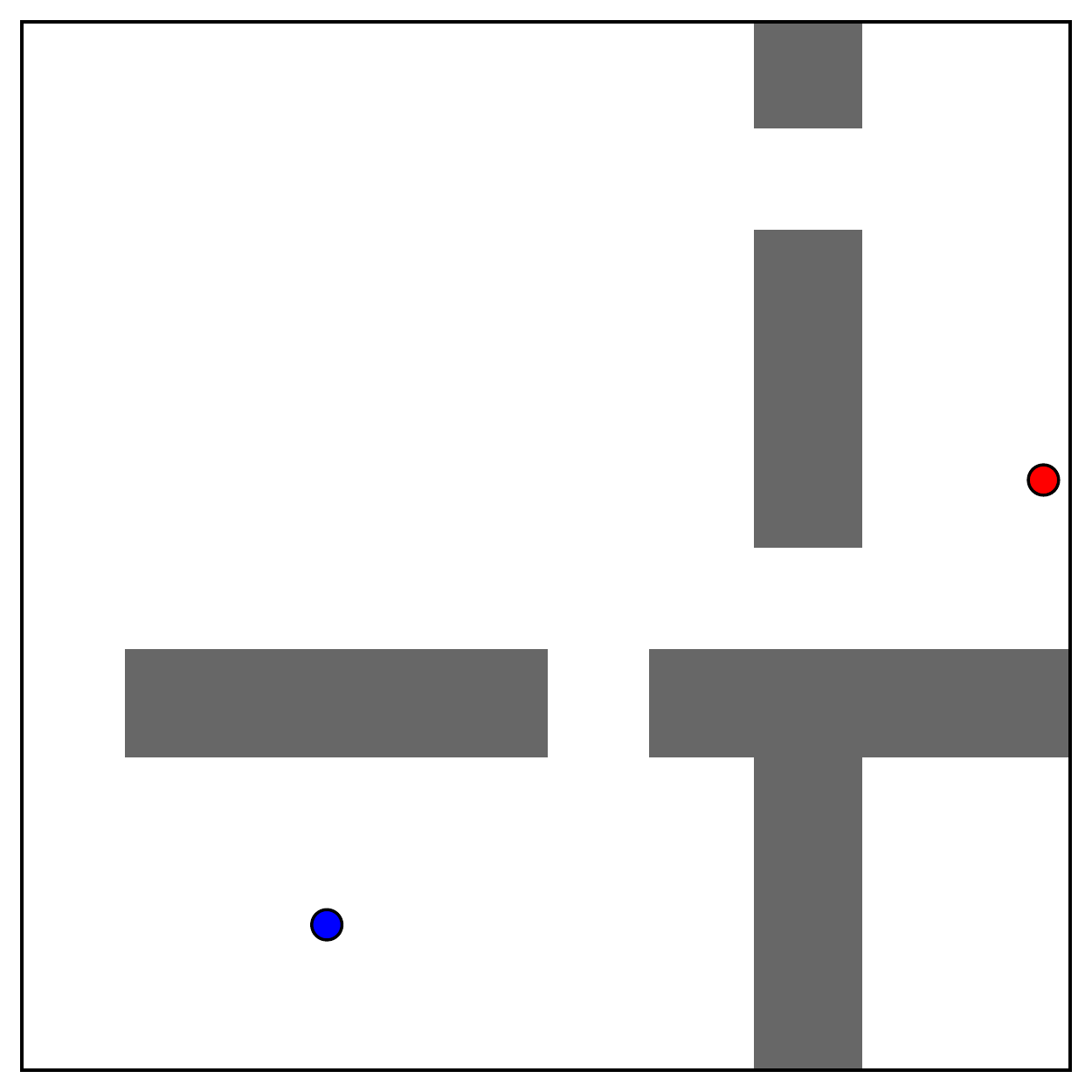}
    \caption{}
    \label{fig:regularization_a}
  \end{subfigure} \hfill
  \begin{subfigure}[b]{0.32\linewidth}
    \centering
    \includegraphics[width=\linewidth]{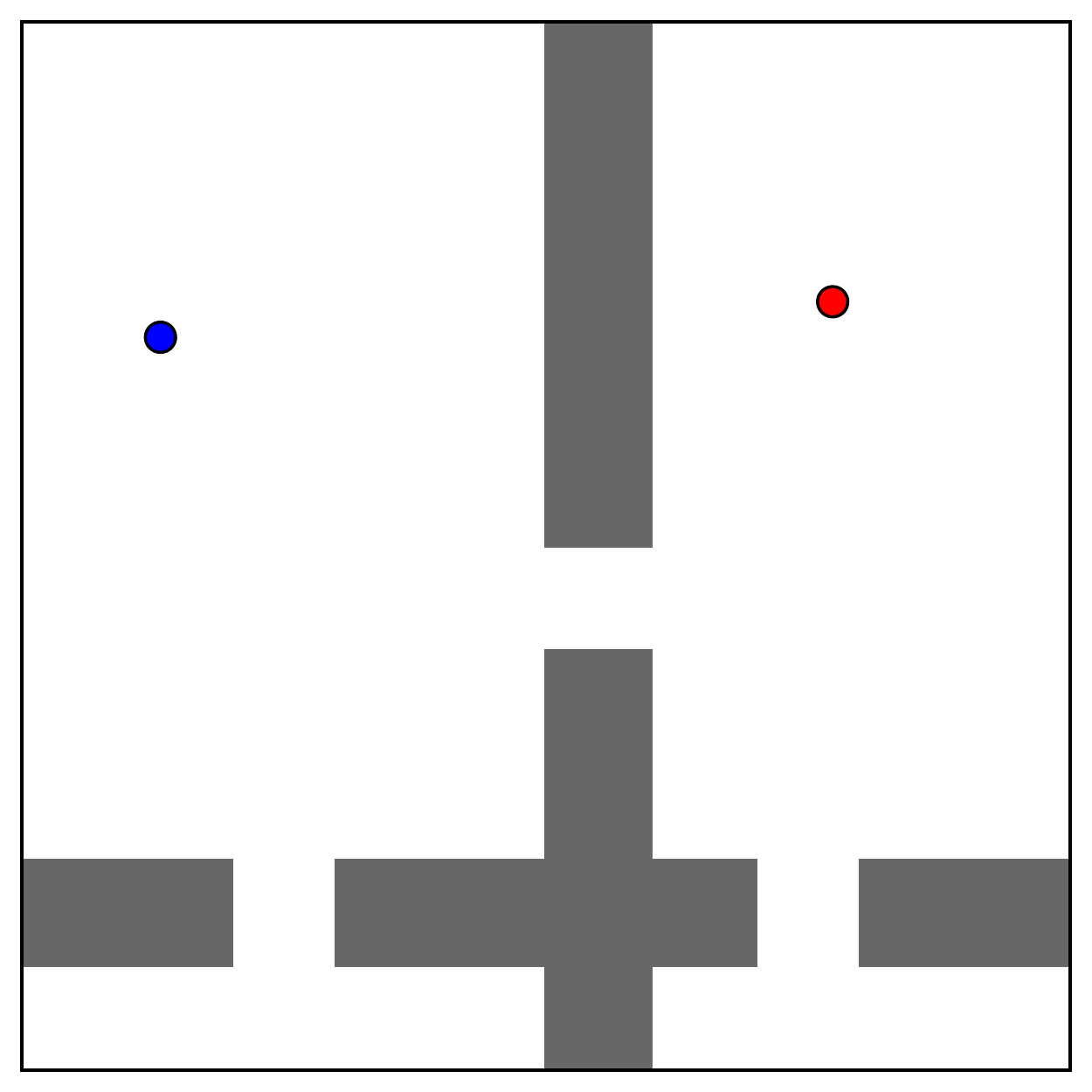}
    \caption{}
    \label{fig:regularization_b}
  \end{subfigure} \hfill
  \begin{subfigure}[b]{0.32\linewidth}
    \centering
    \includegraphics[width=\linewidth]{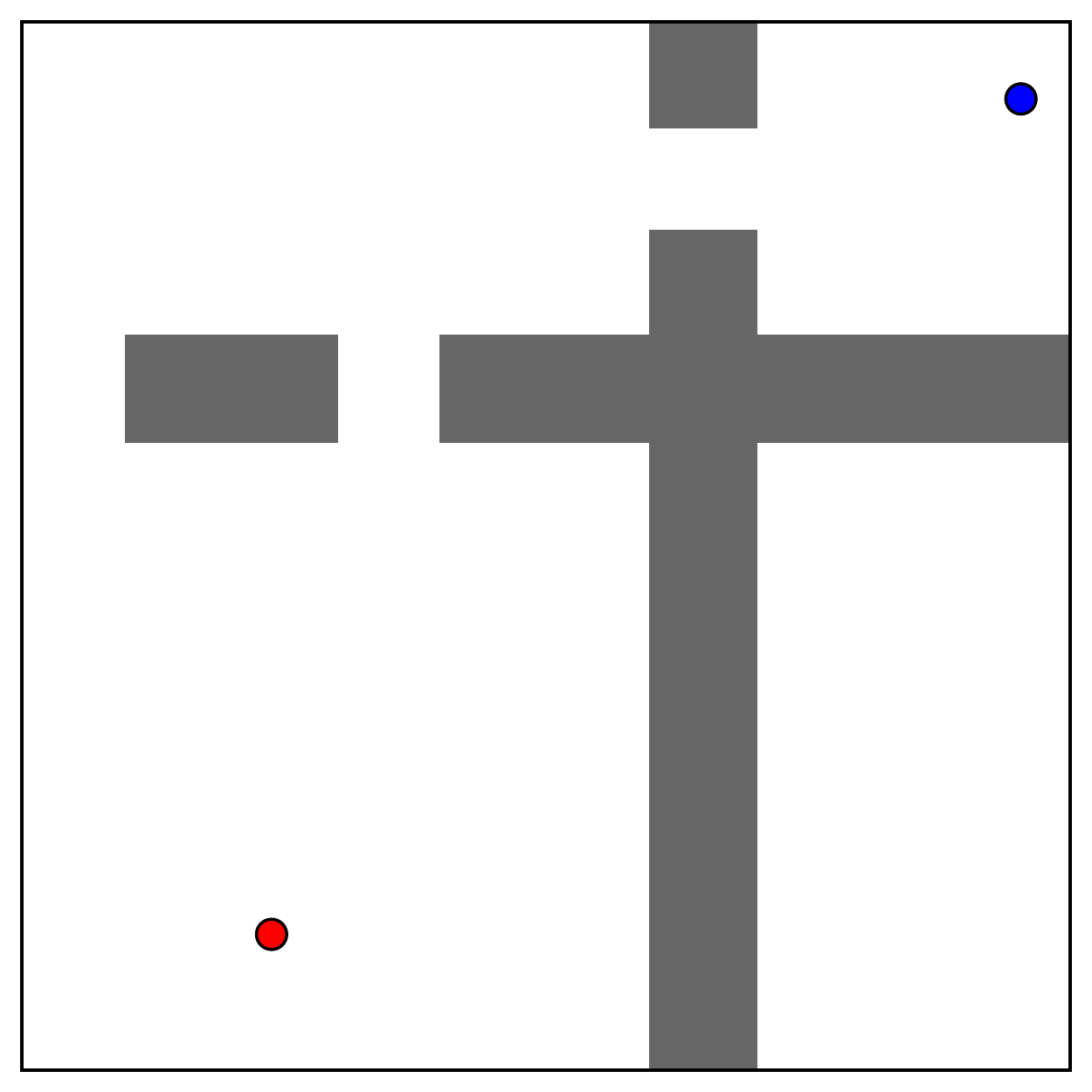}
    \caption{}
    \label{fig:regularization_c}
  \end{subfigure}
  \caption{
  Environments sampled in $\real^2$ to train the CVAE.\fullFigGap
  }
  \label{fig:2d_environments}
  \vspace{2mm}
\end{figure}

\paragraph{N-Link Arm Planning} The training procedure for the robot in $\real^3, ~\real^5$ consisted of a $\denseGraph$ with 6000 samples which was used to plan for 20 planning problems in each of 20 randomly generated 2D environments. \figref{fig:arm_environments} visualizes some of the environments sampled to train the CVAE. The red and blue positions show the start and goal states respectively. The environment has randomly placesd obstacles. The CVAE was conditioned on a vector of features which included the start-goal pair as well the $10 \times 10$ occupancy grid (100 features). The dataset generation took 6-7 hours while the training time was close to 30 minutes.

\begin{figure}[!ht]
  \centering
  \begin{subfigure}[b]{0.32\linewidth}
    \centering
    \includegraphics[width=\linewidth]{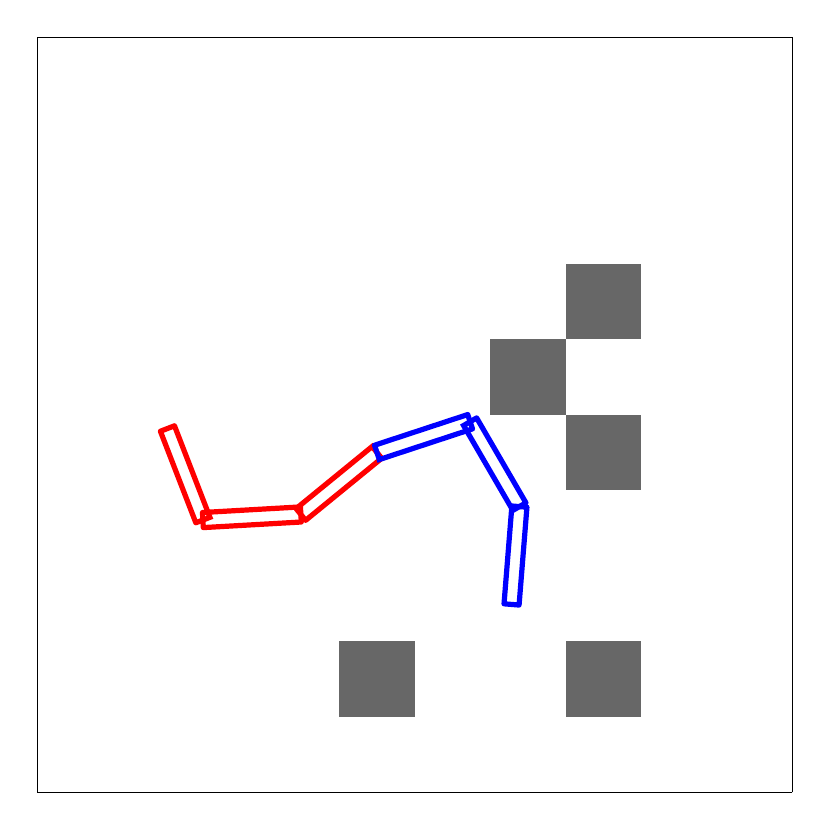}
    \caption{}
    \label{fig:arm_a}
  \end{subfigure} \hfill
  \begin{subfigure}[b]{0.32\linewidth}
    \centering
    \includegraphics[width=\linewidth]{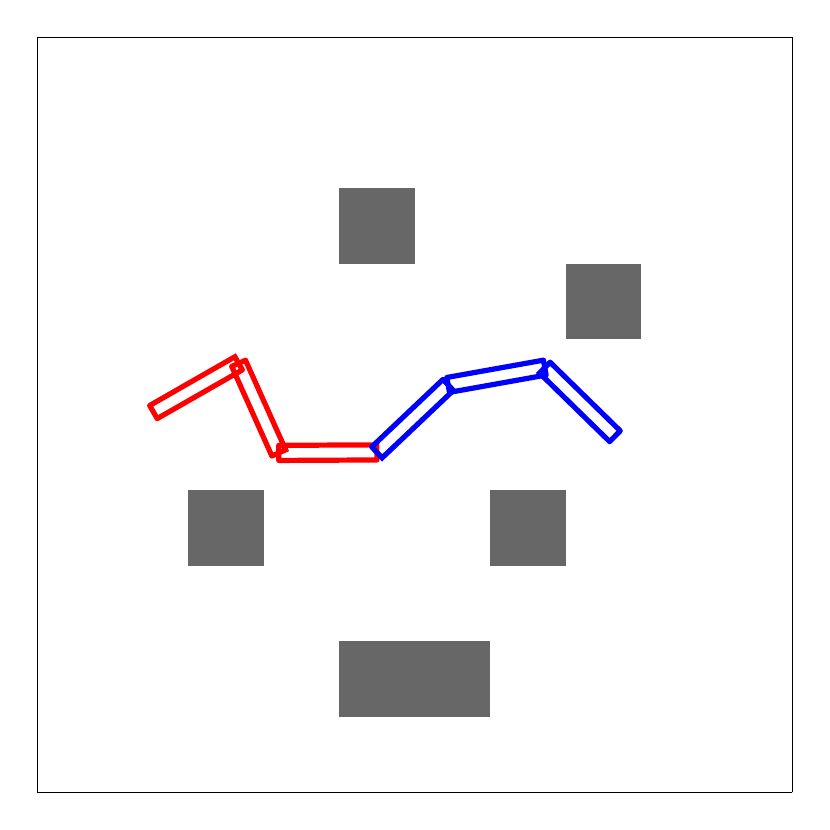}
    \caption{}
    \label{fig:arm_b}
  \end{subfigure} \hfill
  \begin{subfigure}[b]{0.32\linewidth}
    \centering
    \includegraphics[width=\linewidth]{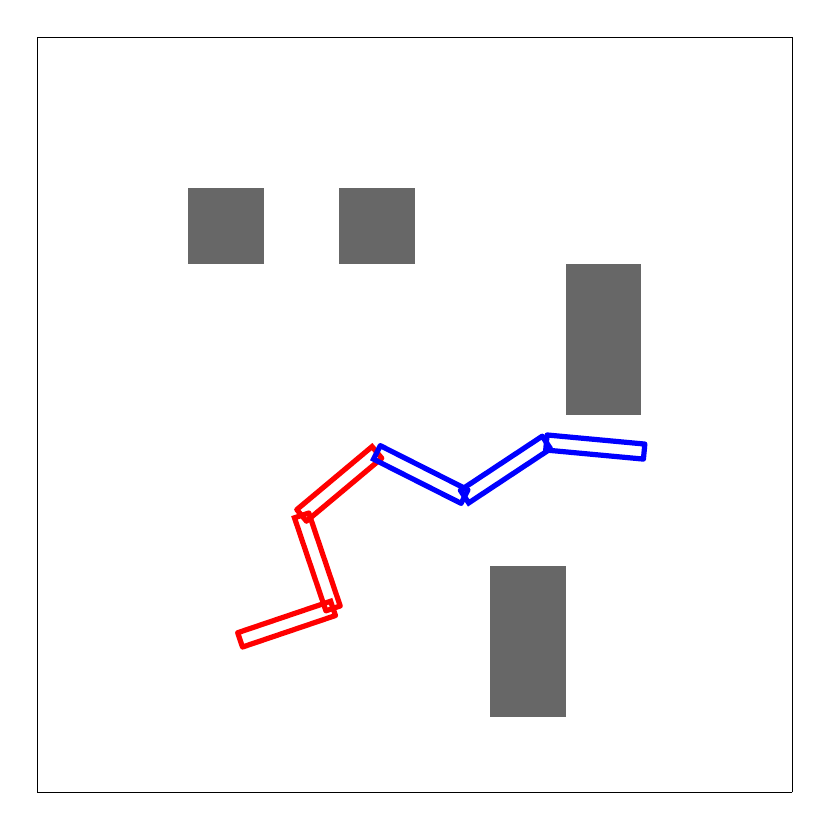}
    \caption{}
    \label{fig:arm_c}
  \end{subfigure}
  \caption{
  Environments sampled in $\real^3$ to train the CVAE.\fullFigGap
  }
  \label{fig:arm_environments}
  \vspace{2mm}
\end{figure}

\paragraph{Snake Robot Planning} For $\real^5$, the training procedure was similar to that in the $\real^2$ problems. The training procedure for the robot in $\real^9$ consisted of a $\denseGraph$ with 6000 samples which was used to plan for 20 planning problems in each of 20 randomly generated 2D environments. \figref{fig:snake_environments} visualizes some of the environments sampled to train the CVAE. The red and blue positions show the start and goal states respectively. The environments were modified in the wall being horizontal or vertical, the offset in its position, and the position of the narrow passage through it. The CVAE was conditioned on a vector of 118 features which included the start-goal pair (18 features) as well the $10 \times 10$ occupancy grid (100 features). The dataset generation took 6-7 hours while the training time was close to 30 minutes. The CVAE was trained to sample configurations of the snake robot that included the base location as well as the revolute joint angles between each of the links.

\begin{figure}[!ht]
  \centering
  \begin{subfigure}[b]{0.32\linewidth}
    \centering
    \includegraphics[width=\linewidth]{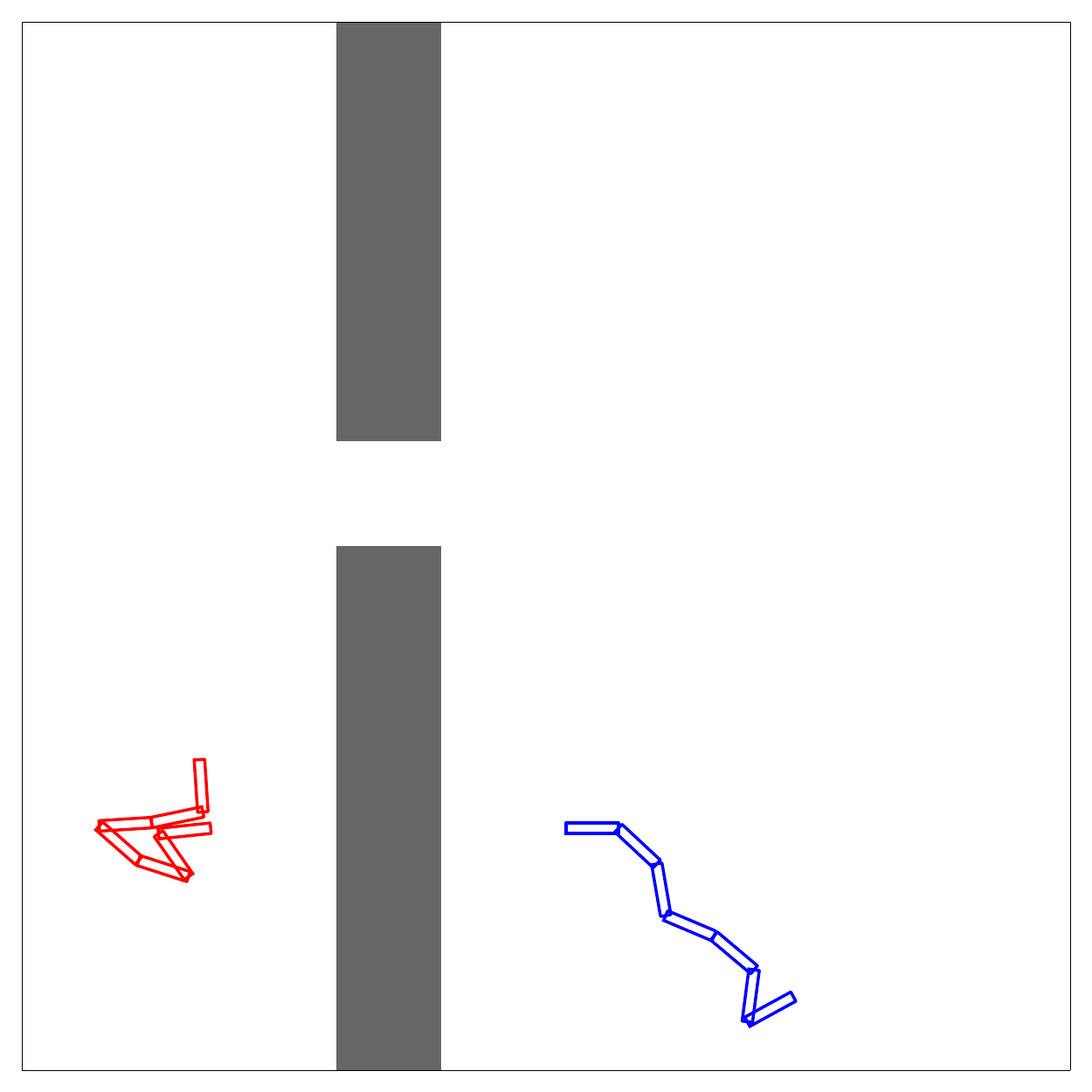}
    \caption{}
    \label{fig:snake_a}
  \end{subfigure} \hfill
  \begin{subfigure}[b]{0.32\linewidth}
    \centering
    \includegraphics[width=\linewidth]{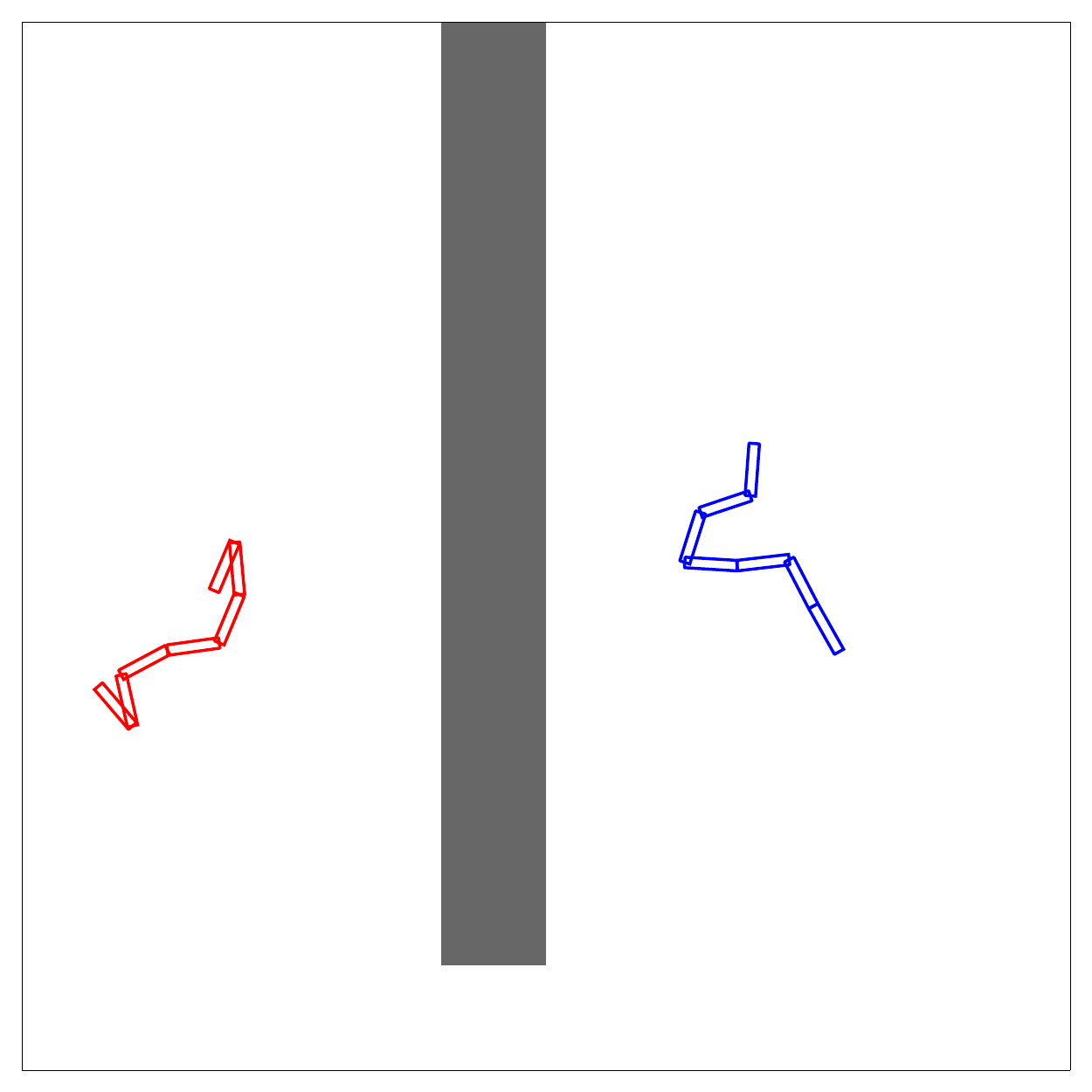}
    \caption{}
    \label{fig:snake_b}
  \end{subfigure} \hfill
  \begin{subfigure}[b]{0.32\linewidth}
    \centering
    \includegraphics[width=\linewidth]{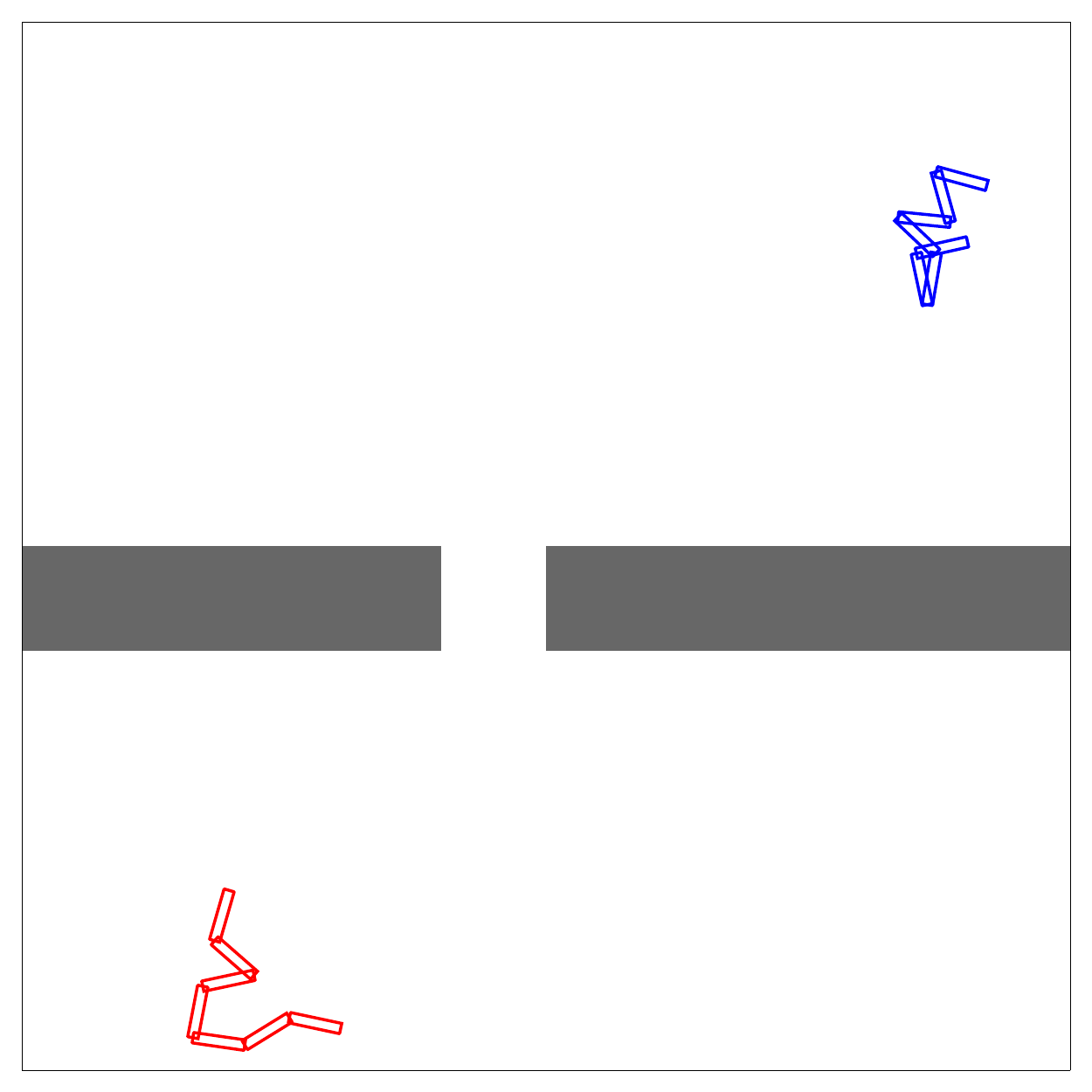}
    \caption{}
    \label{fig:snake_c}
  \end{subfigure}
  \caption{
  Environments sampled in $\real^9$ to train the CVAE.\fullFigGap
  }
  \label{fig:snake_environments}
  \vspace{2mm}
\end{figure}

\paragraph{Manipulator Arm Planning} The training data consisted of 20 random environments where the obstacles in the environment were arbitrarily repositioned. In each of the randomly generated environment, 50 planning problems were considered as an input to the train the CVAE model. \figref{fig:herb_environments} visualized three such environments, where the positions of the table and that of the obstacle on the table are modified along with start and goal configurations. The CVAE in the constrained problem was conditioned on a vector of 46\footnote{48 in the unconstrained problem since the configuration of the robot includes an additional degree of freedom.} features which included the start and goal configurations (14 features) and the poses of the table and the obstacle represented as $4 \times 4$ homogeneous matrices (32 features). The dataset was generated in 7-8 hours while the training took around an hour. Samples from a $\denseGraph$ with 30,000 configurations were used to train the CVAE. The CVAE learned to sample the robot configurations which included the joint angles at the seven revolute joints of the arm in the constrained example. The unconstrained $\real^8$ example consisted of an additional prismatic joint value denoting where the stick is held in the hand.

\begin{figure}[!ht]
  \centering
  \begin{subfigure}[b]{0.32\linewidth}
    \centering
    \includegraphics[width=\linewidth]{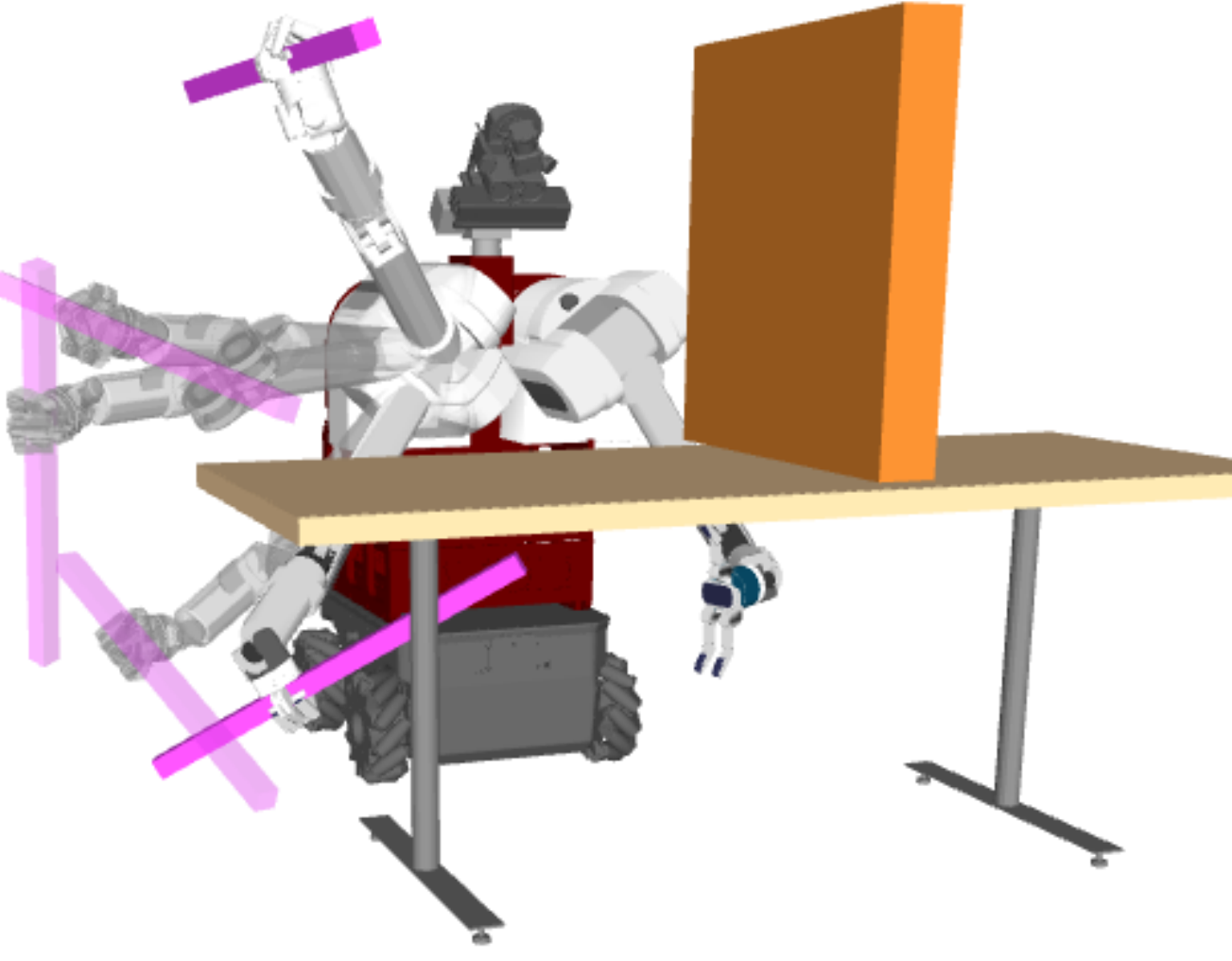}
    \caption{}
    \label{fig:herb_appendix_a}
  \end{subfigure} \hfill
  \begin{subfigure}[b]{0.32\linewidth}
    \centering
    \includegraphics[width=\linewidth]{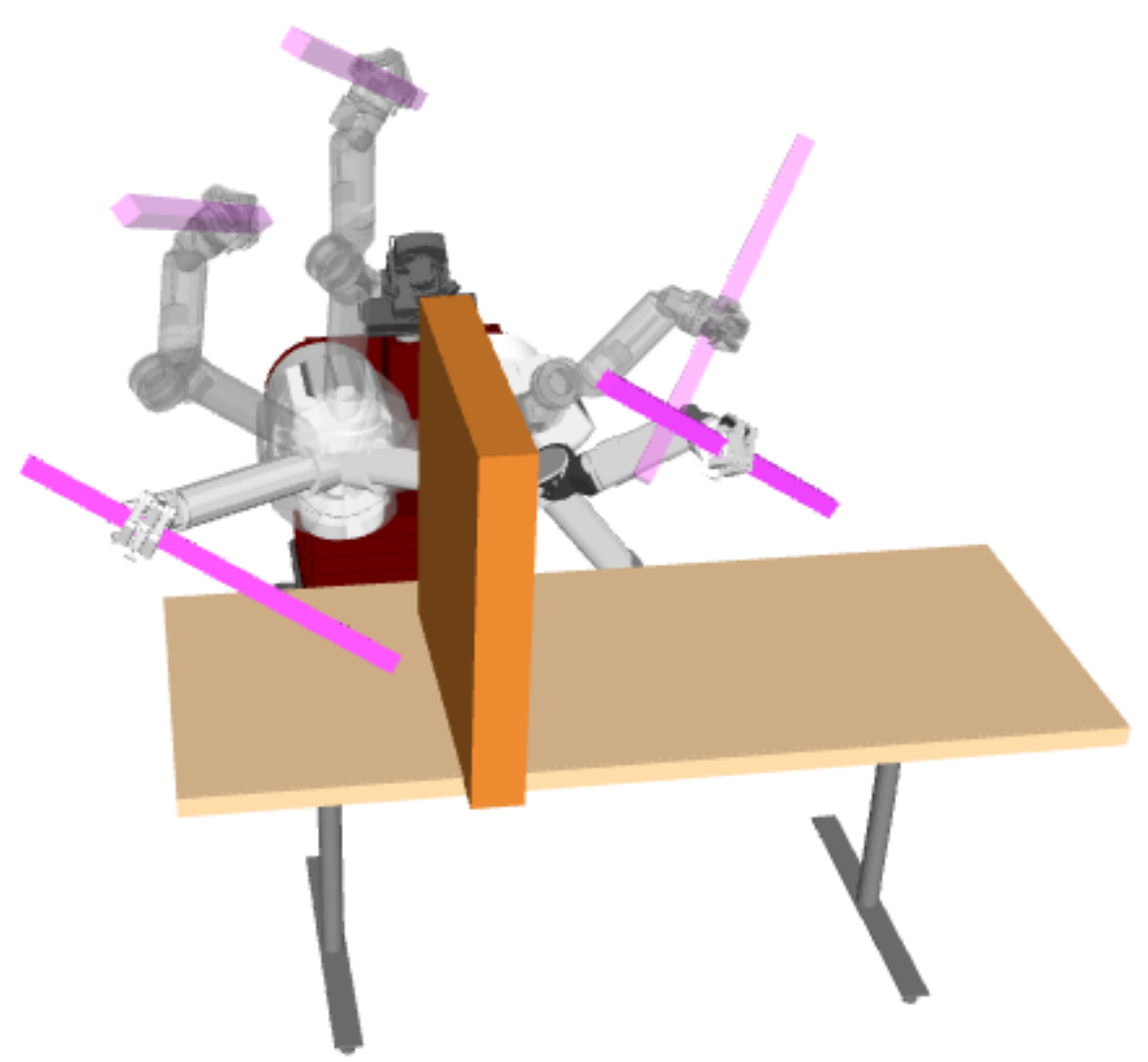}
    \caption{}
    \label{fig:herb_appendix_b}
  \end{subfigure} \hfill
  \begin{subfigure}[b]{0.32\linewidth}
    \centering
    \includegraphics[width=\linewidth]{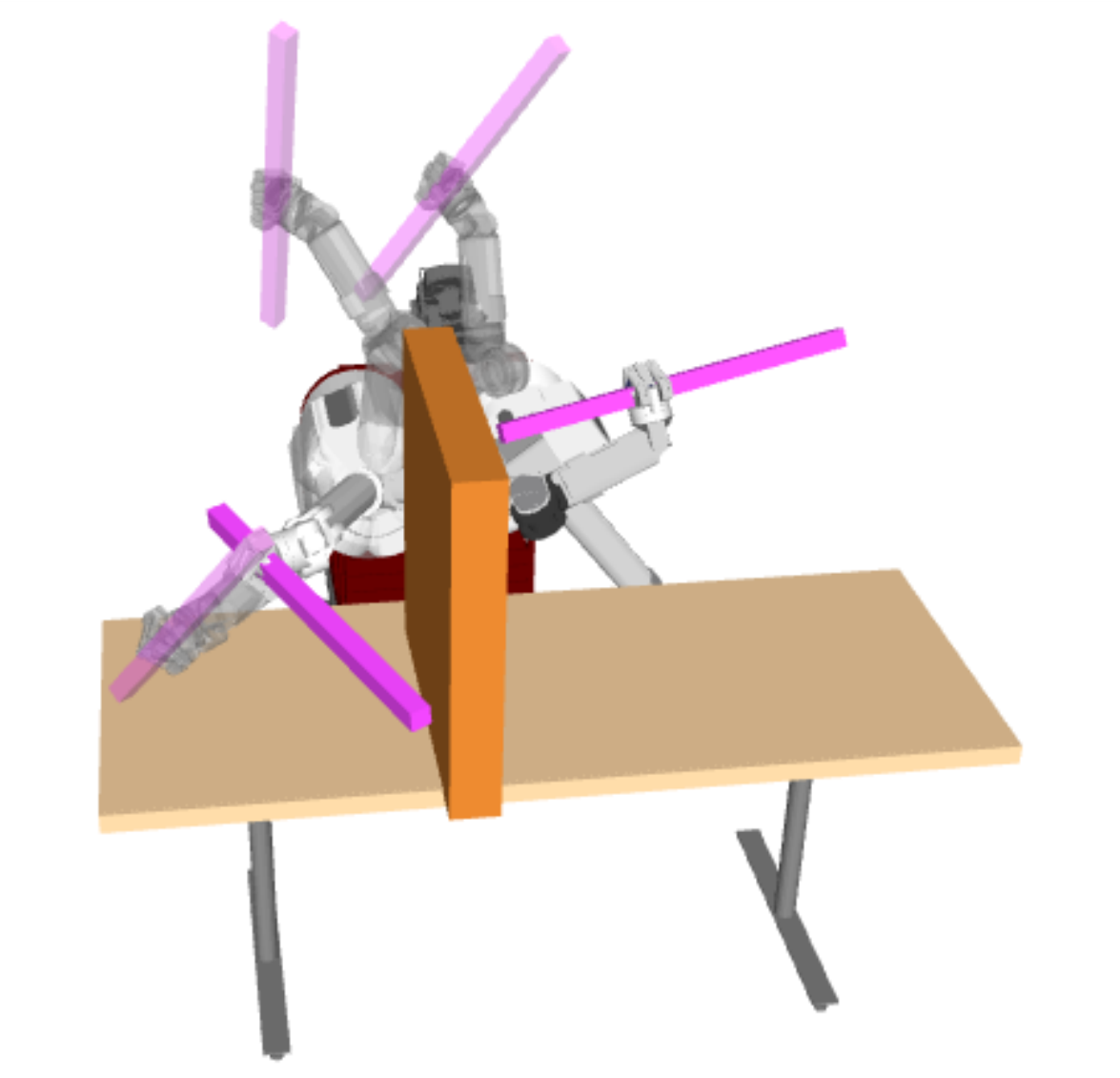}
    \caption{}
    \label{fig:herb_appendix_c}
  \end{subfigure}
  \caption{
  Manipulator arm environments sampled to train the CVAE. The solutions obtained using samples generated by \algLEGO are also visualized. \fullFigGap
  }
  \label{fig:herb_environments}
  \vspace{2mm}
\end{figure}

\subsection{Additional Experiment Results}

\paragraph{\algBottleneck and \algDiversity}

In addition to the qualitative observations presented in \sref{sec:results} (O1 and O2) and \figref{fig:bd_experiments}, we present here the analysis of the performance of the foundational algorithms of \algLEGO, namely \algBottleneck and \algDiversity when compared to \algSP. \figref{fig:bottleneck_success_rate} shows that on a $\real^2$ world, \algBottleneck has a significantly higher success rate that \algSP, almost converging to $1.0$ by $400$ samples. \figref{fig:bottleneck_path_length} shows that in terms of path length, \algSP is initially better but both are eventually comparable. This is expected because of the near-optimality objective of \algBottleneck~\eref{eq:bottleneck_nodes}. \figref{fig:diverse_success_rate} shows that \algDiversity has a better success rate. \figref{fig:diverse_path_length} shows that while both algorithms are comparable in terms of path length, \algDiversity has a smaller variance.

\begin{figure}[!h]
  \centering
  \begin{subfigure}[b]{0.49\linewidth}
    \centering
    \includegraphics[width=\linewidth]{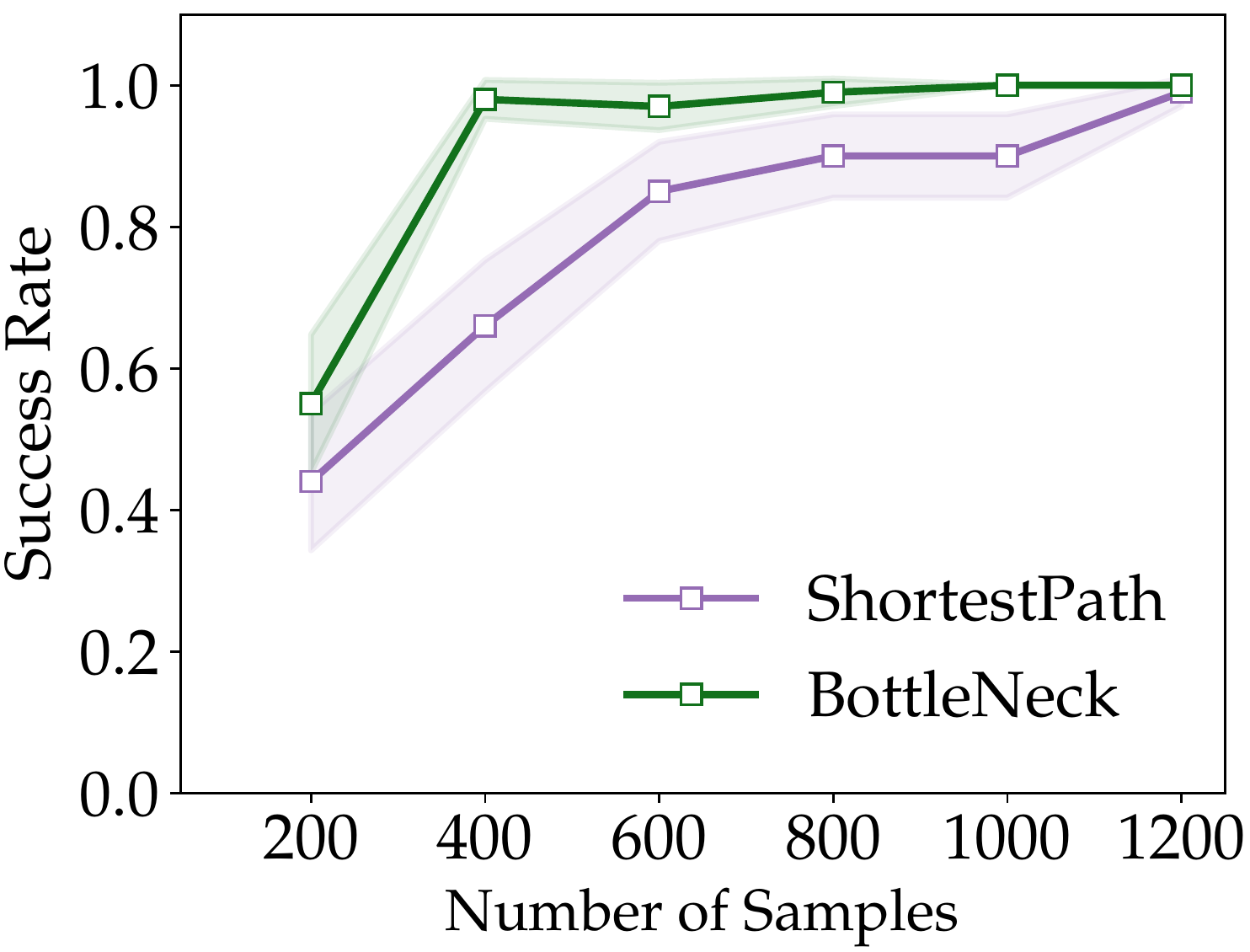}
    \caption{}
    \label{fig:bottleneck_success_rate}
  \end{subfigure}
  \begin{subfigure}[b]{0.49\linewidth}
    \centering
    \includegraphics[width=\linewidth]{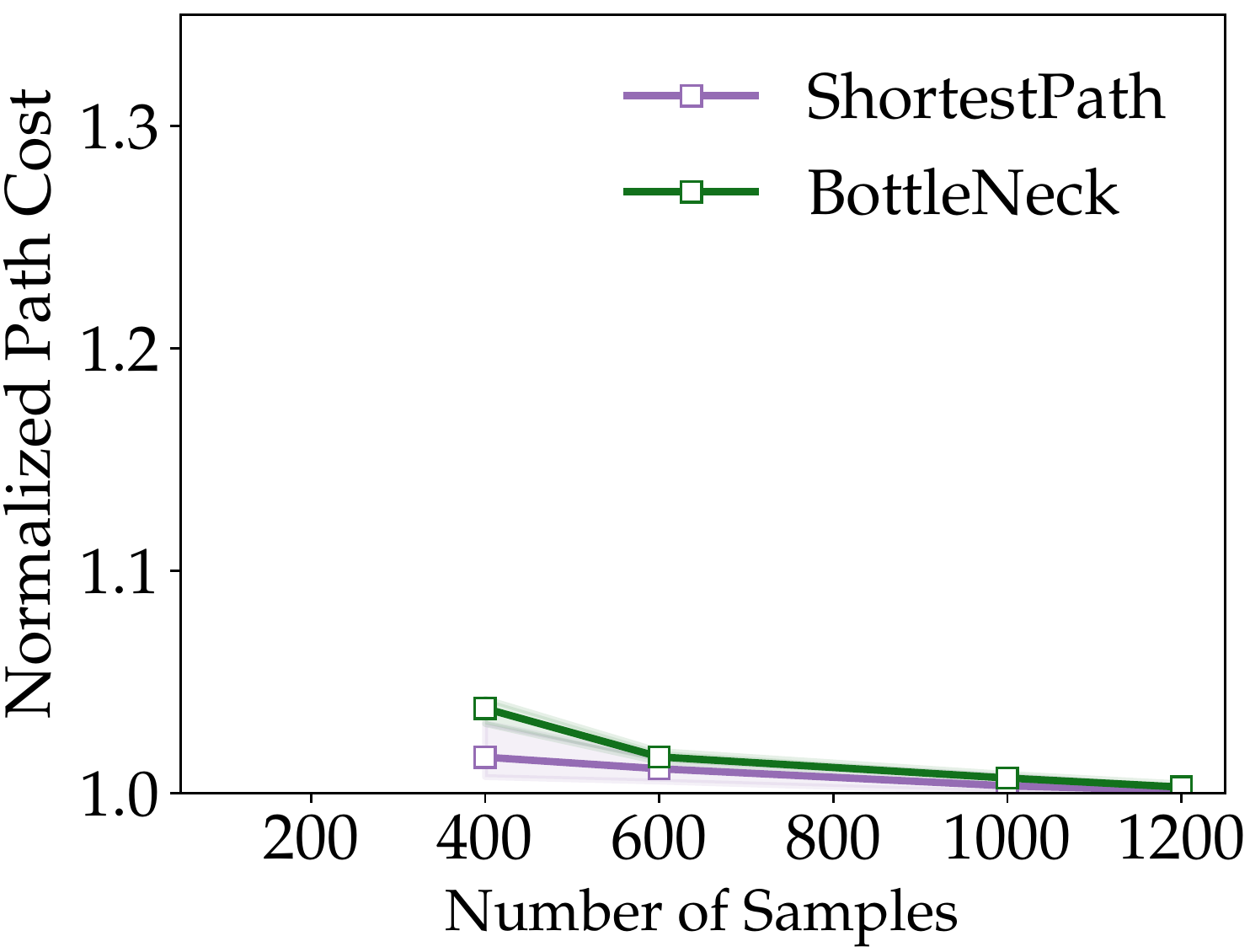}
    \caption{}
    \label{fig:bottleneck_path_length}
  \end{subfigure}
  \begin{subfigure}[b]{0.49\linewidth}
    \centering
    \includegraphics[width=\linewidth]{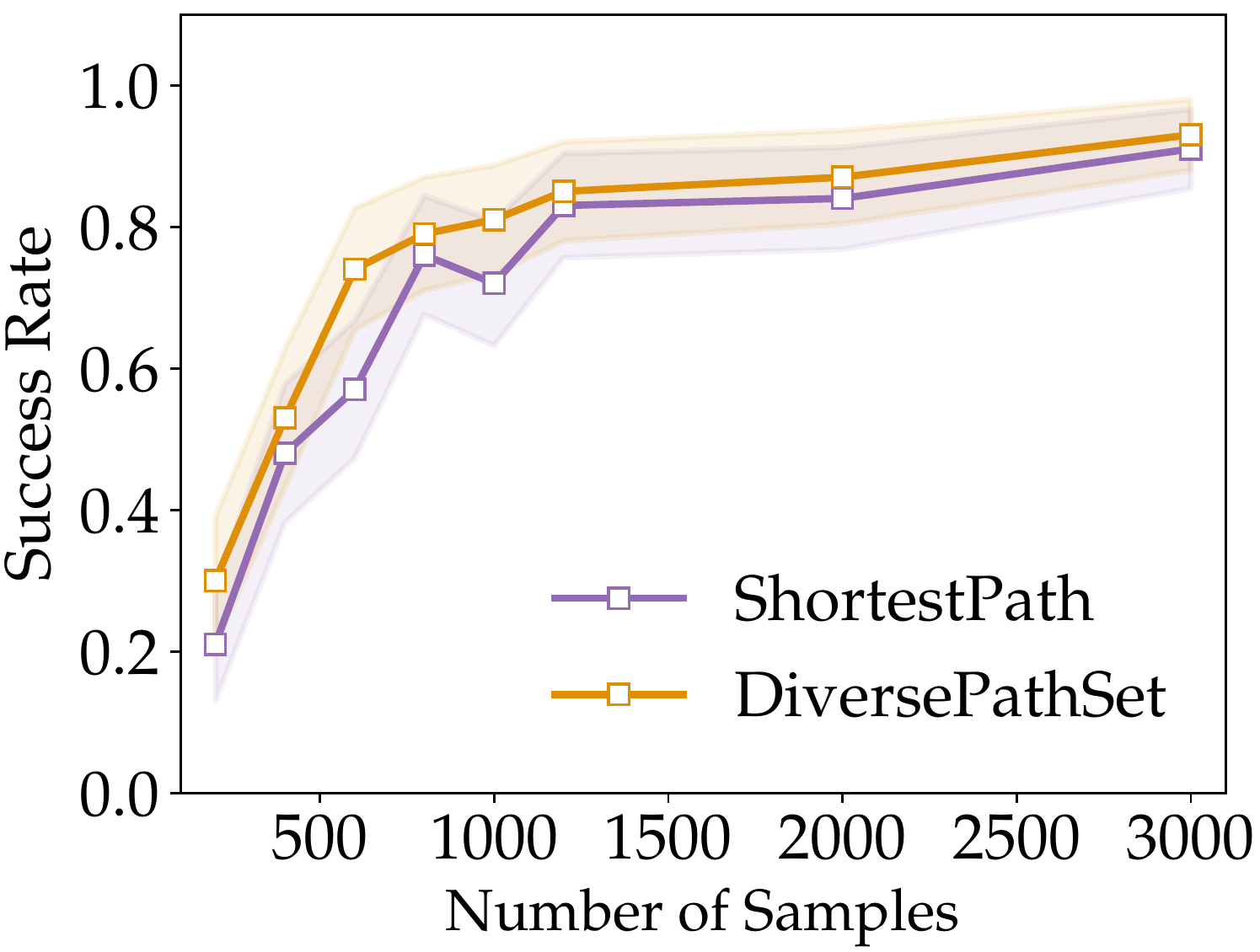}
    \caption{}
    \label{fig:diverse_success_rate}
  \end{subfigure}
  \begin{subfigure}[b]{0.49\linewidth}
    \centering
    \includegraphics[width=\linewidth]{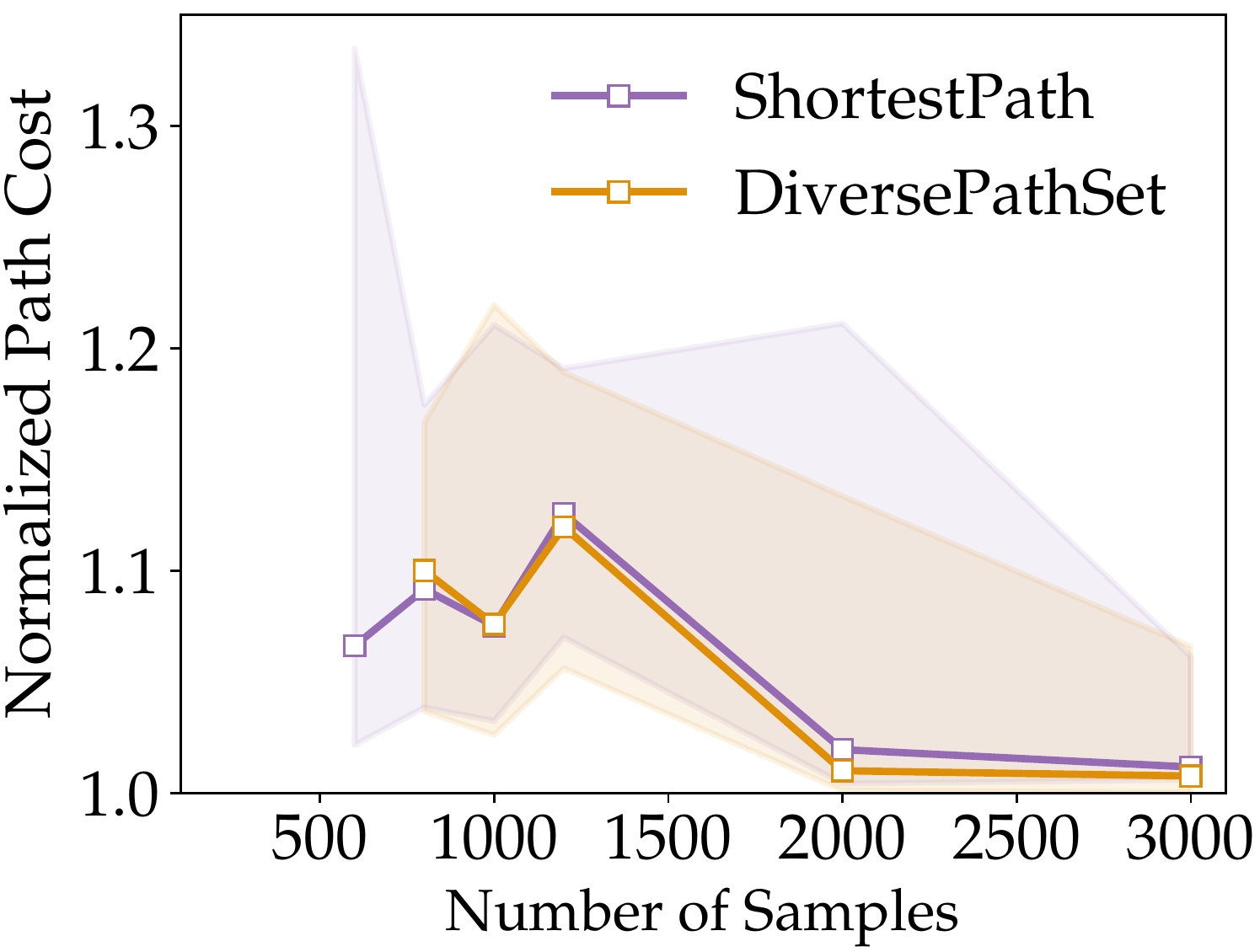}
    \caption{}
    \label{fig:diverse_path_length}
  \end{subfigure}
  \caption{
  Comparison of \algSP against \algBottleneck (top row) and \algDiversity (bottom row) on success rate (left column) and normalized path length (right column).\fullFigGap
  }
  \vspace{2mm}
  \label{fig:parameter_figs}
\end{figure}

\subsection{Roadmap Construction}

\begin{figure}[!h]
  \centering
  \begin{subfigure}[b]{0.49\linewidth}
    \centering
    \includegraphics[width=\linewidth]{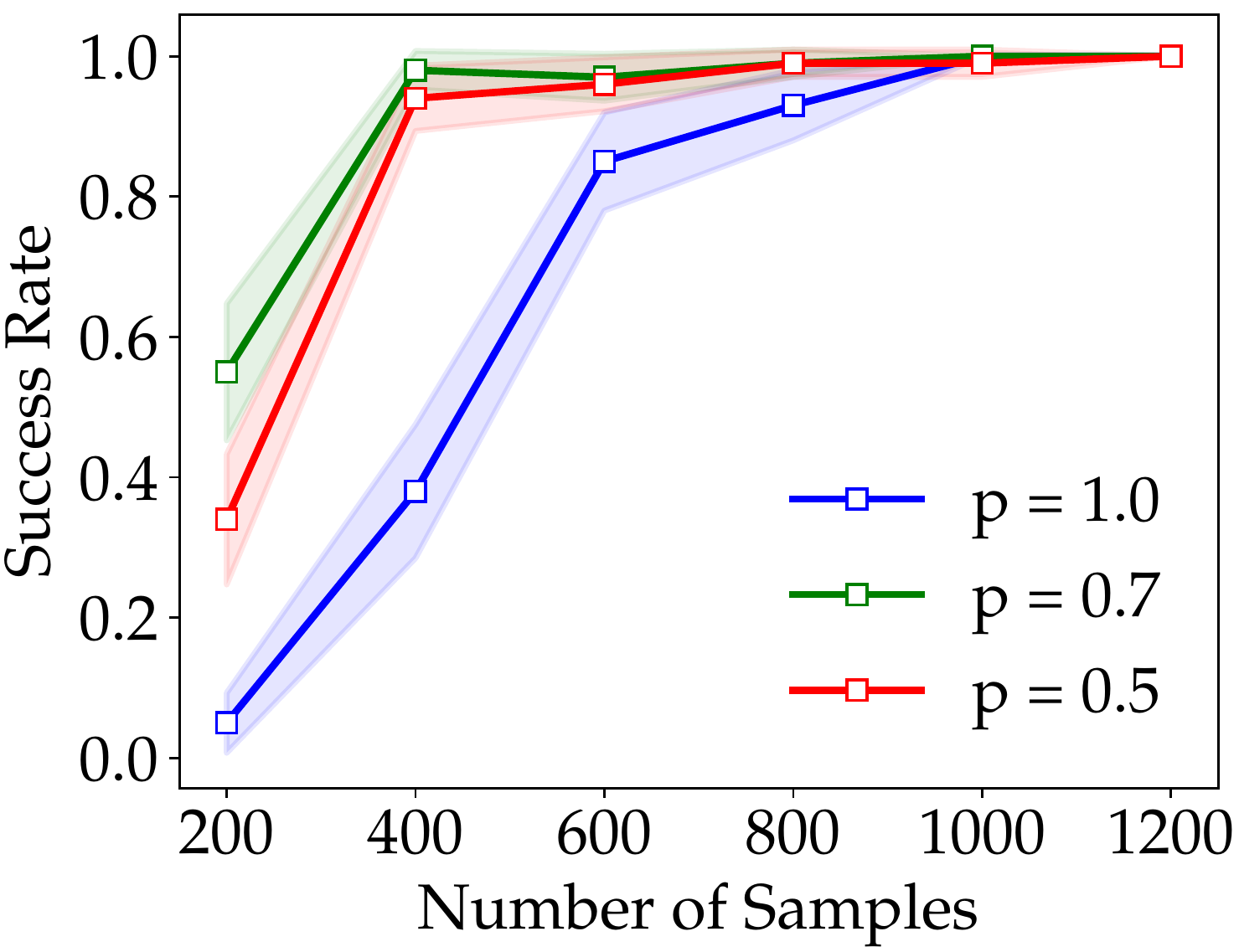}
    \caption{}
    \label{fig:parameter_sr}
  \end{subfigure}
  \begin{subfigure}[b]{0.49\linewidth}
    \centering
    \includegraphics[width=\linewidth]{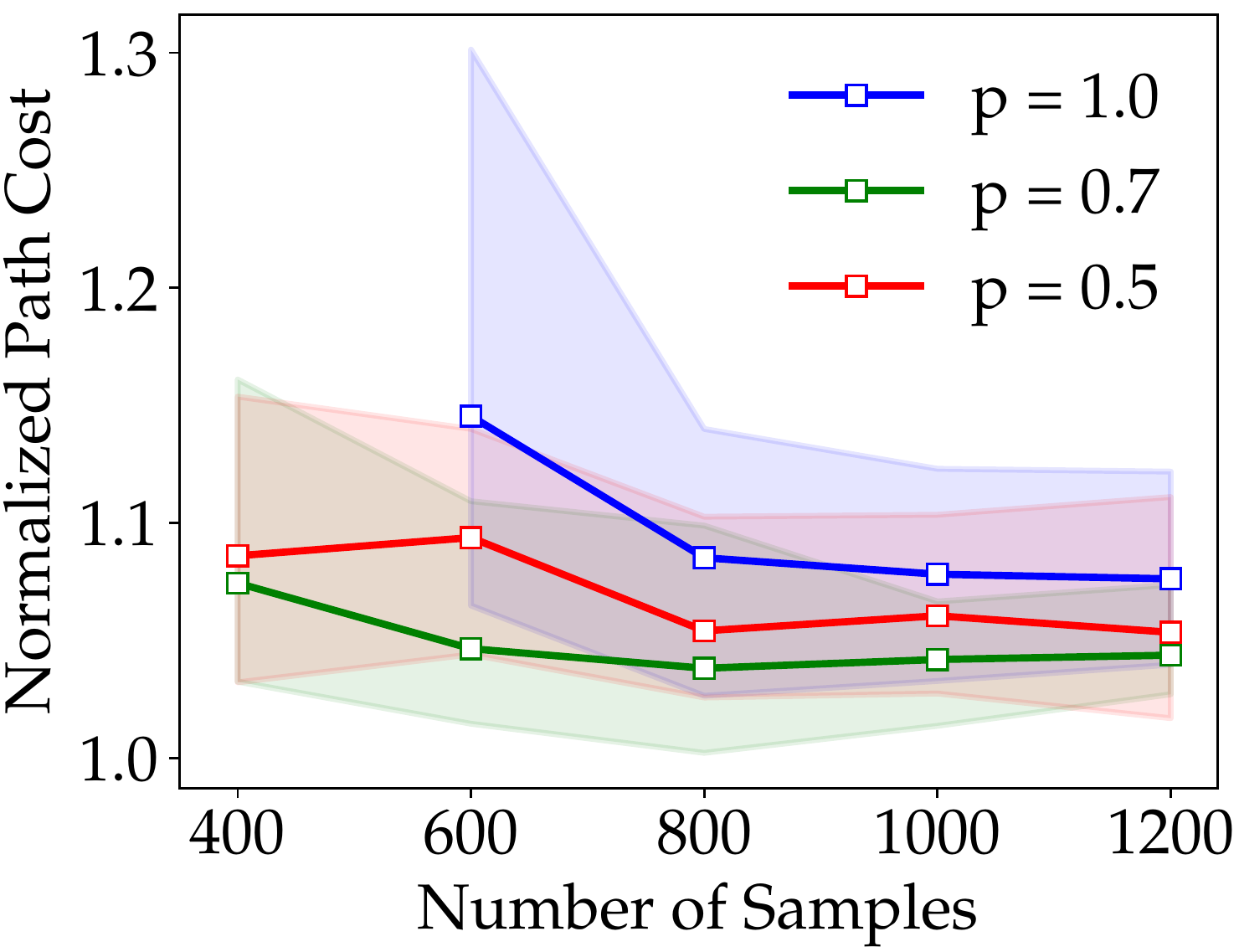}
    \caption{}
    \label{fig:parameter_pl}
  \end{subfigure}
  \caption{
  Performance of \algLEGO on different roadmaps. The parameter $p$ denotes the ratio of Halton samples to learned samples in the roadmap).\fullFigGap
  }
  \label{fig:parameter_figs}
  \vspace{2mm}
\end{figure}

To evaluate the performance of \algLEGO, we construct sparse roadmaps, $\sparseGraph$. The sparse graph consisted of 200 samples in $\real^2,~\real^5$ problems and 300 samples in case of $\real^7,~\real^8$ and $\real^9$ problems. Not however, that this sparse roadmap contains both the learned samples as well as samples generated from Halton sequence. While the learned samples are concentrated near the bottleneck regions and along diverse paths, Halton samples ensure the coverage over the free regions of the configuration space as well. We analyze different proportions of Halton samples and learned samples. \figref{fig:parameter_figs} shows the performance characteristics of \algLEGO on roadmaps constructed with different proportions of Halton and learned samples for the 2D point robot example. We observe that \algLEGO over a roadmap of 200 samples with just 30\% learned samples significantly outperforms \algLEGO over a Halton graph ($p = 1$). \figref{fig:herb_samples} visualizes the samples generated by \algLEGO represented by the end-effector positions (blue) in the workspace.

\begin{figure}[!ht]
  \centering
  \begin{subfigure}[b]{0.32\linewidth}
    \centering
    \includegraphics[width=\linewidth]{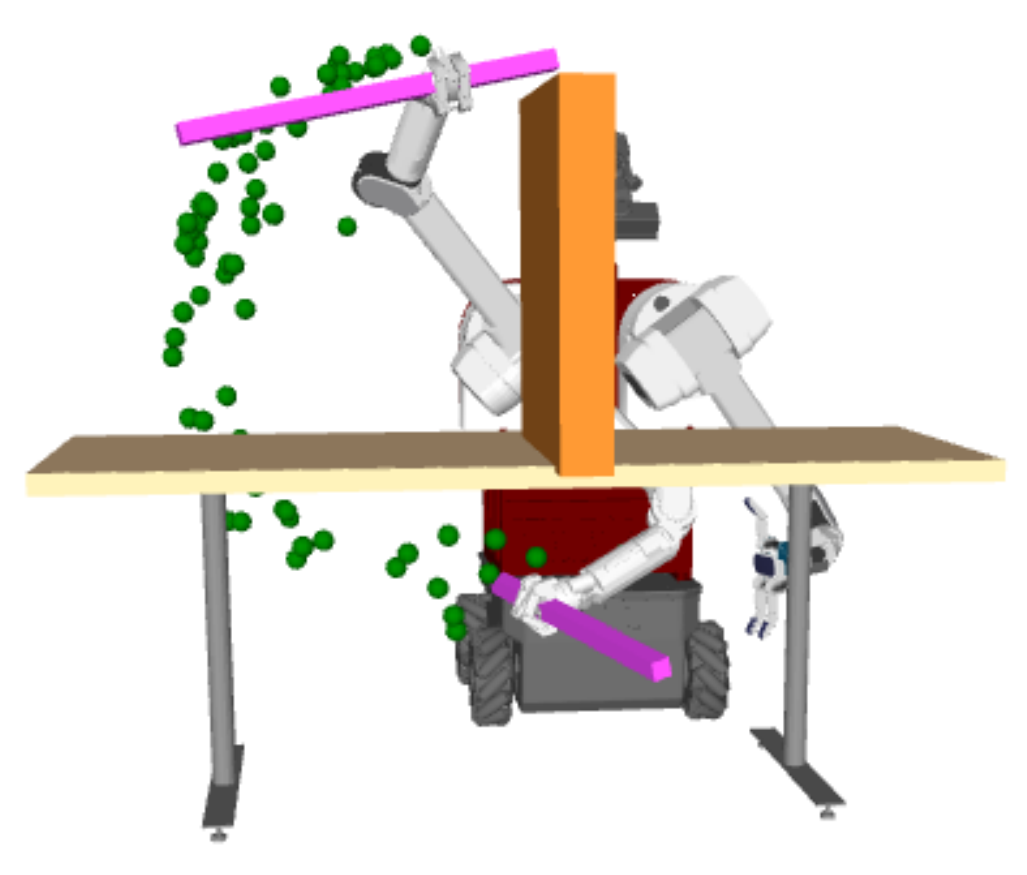}
    \caption{}
    \label{fig:herb_samples_a}
  \end{subfigure} \hfill
  \begin{subfigure}[b]{0.32\linewidth}
    \centering
    \includegraphics[width=\linewidth]{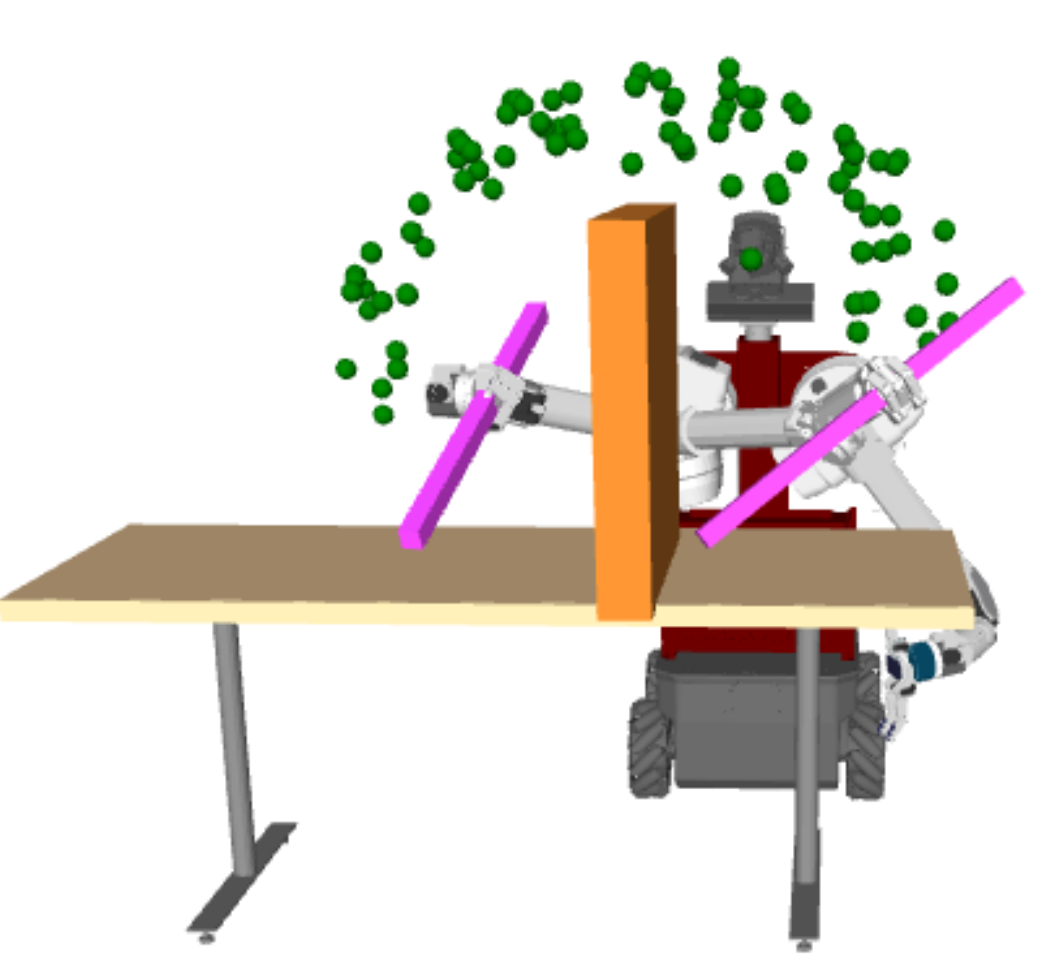}
    \caption{}
    \label{fig:herb_samples_b}
  \end{subfigure} \hfill
  \begin{subfigure}[b]{0.32\linewidth}
    \centering
    \includegraphics[width=\linewidth]{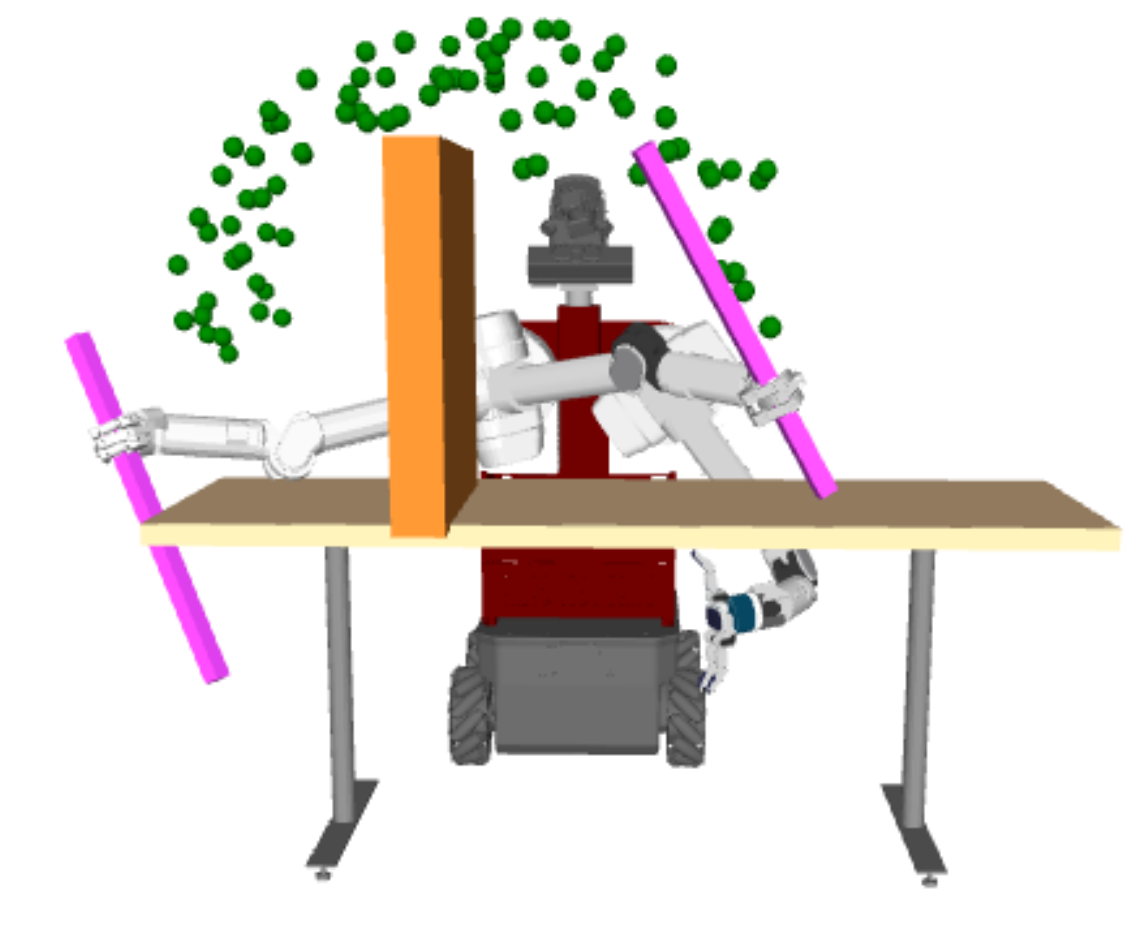}
    \caption{}
    \label{fig:herb_samples_c}
  \end{subfigure}
  \caption{
  Samples generated by \algLEGO for manipulator arm planning. The blue dots represent the end effector positions corresponding to the samples.\fullFigGap
  }
  \label{fig:herb_samples}
  \vspace{2mm}
\end{figure}

\end{appendices}

}

\end{document}